\documentclass[lettersize,journal]{IEEEtran}
\newcommand{\myparagraph}[1]{\vspace{0.3\baselineskip}\noindent{\textbf{#1.}}~}
\usepackage[linesnumbered,ruled,vlined,boxed,noend]{algorithm2e}
\usepackage{booktabs}
\usepackage{amsmath}
\usepackage{amsthm}
\usepackage{amssymb}
\usepackage{graphicx}
\usepackage{enumitem}
\usepackage{multirow}
\usepackage{xcolor}

\setlist[itemize]{leftmargin=*}

\usepackage[colorlinks=true,linkcolor=red,urlcolor=blue,citecolor=blue]{hyperref}

\DeclareMathOperator*{\argmin}{argmin}
\newcommand{\bw}{\boldsymbol{w}}
\newcommand{\bP}{\boldsymbol{P}}

\newcommand{\flame}{\texttt{FLAME}\xspace}
\newcommand{\lpproj}{\texttt{Lp-Proj}\xspace}
\newcommand{\lpprojtwo}{\texttt{Lp-Proj-2}\xspace}

\newcommand{\ditto}{\texttt{Ditto}\xspace}
\newcommand{\pfedme}{\texttt{pFedMe}\xspace}
\newcommand{\flamehm}{\texttt{FLAME-HM}\xspace}

\newcommand{\flamepm}{\texttt{FLAME-PM}\xspace}
\newcommand{\dittopm}{\texttt{Ditto-PM}\xspace}
\newcommand{\pfedmepm}{\texttt{pFedMe-PM}\xspace}
\newcommand{\flamegm}{\texttt{FLAME-GM}\xspace}
\newcommand{\dittogm}{\texttt{Ditto-GM}\xspace}
\newcommand{\pfedmegm}{\texttt{pFedMe-GM}\xspace}
\newcommand{\loss}{\text{Loss}\xspace}

\newcommand{\fedadmm}{\texttt{FedADMM}\xspace}
\newcommand{\fedavg}{\texttt{FedAvg}\xspace}
\newcommand{\mnist}{\textsf{MNIST}\xspace}
\newcommand{\fmnist}{\textsf{FMNIST}\xspace}
\newcommand{\mmnist}{\textsf{MMNIST}\xspace}
\newcommand{\femnist}{\textsf{FEMNIST}\xspace}
\newcommand{\cifar}{\textsf{CIFAR10}\xspace}
\newcommand{\kl}{KŁ\xspace}

\newcommand{\ml}{\mathcal{L}}

\newcommand{\bx}{\boldsymbol{x}}
\newcommand{\by}{\boldsymbol{y}}
\newcommand{\biden}{\boldsymbol{I}}
\newcommand{\bz}{\boldsymbol{z}}
\newcommand{\bu}{\boldsymbol{u}}
\newcommand{\bone}{\boldsymbol{1}}
\newcommand{\bzero}{\boldsymbol{0}}
\newcommand{\sva}{\text{SVA}}
\newcommand{\sfa}{\text{SFA}}
\newcommand{\bV}{\boldsymbol{V}}

\newcommand{\bpi}{\boldsymbol{\pi}}
\newcommand{\btheta}{\boldsymbol{\theta}}
\newtheorem{theorem}{Theorem}

\newtheorem{proposition}{Proposition}
\newtheorem{definition}{Definition}
\newtheorem{assumption}{Assumption}
\newtheorem{lemma}{Lemma}

\DeclareMathOperator{\trace}{tr}
\DeclareMathOperator{\var}{var}

\ifCLASSOPTIONcompsoc
  \usepackage[nocompress]{cite}
\else
  \usepackage{cite}
\fi
\ifCLASSINFOpdf
\else
\fi  
\begin{document}
\title{On ADMM in Heterogeneous Federated Learning: Personalization, Robustness, and Fairness}
\author{Shengkun Zhu,
        Jinshan Zeng,
        Sheng Wang,
        Yuan Sun,
        Xiaodong Li,~\IEEEmembership{Fellow,~IEEE,}
        Yuan Yao,
        Zhiyong Peng
\IEEEcompsocitemizethanks{\IEEEcompsocthanksitem Shengkun Zhu,  Sheng Wang, and Zhiyong Peng are with the School of Computer Science, Wuhan University.
E-mail: \{whuzsk66, swangcs, peng\}@whu.edu.cn
\IEEEcompsocthanksitem Jinshan Zeng is with the School of Computer and Information Engineering, Jiangxi Normal University.
E-mail: jinshanzeng@jxnu.edu.cn
\IEEEcompsocthanksitem Yuan Sun is with La Trobe Business School, La Trobe University.
Email: yuan.sun@latrobe.edu.au.
\IEEEcompsocthanksitem Xiaodong Li is with the School of Computing Technologies, RMIT University.
Email: xiaodong.li@rmit.edu.au.
\IEEEcompsocthanksitem Yuan Yao is with Hong Kong University of Science and Technology.
Email: yuany@ust.hk.
\IEEEcompsocthanksitem Shengkun Zhu and Jinshan Zeng are co-first authors.
}

}

\markboth{Journal of \LaTeX\ Class Files,~Vol.~14, No.~8, August~2024}%
{Shell \MakeLowercase{\textit{et al.}}: Bare Demo of IEEEtran.cls for Computer Society Journals}
\IEEEtitleabstractindextext{
\begin{abstract}
Statistical heterogeneity is a root cause of tension among accuracy, fairness, and robustness of federated learning (FL), and is key in paving a path forward.
Personalized federated learning (PFL) is an approach that aims to reduce the impact of statistical heterogeneity by developing personalized models for individual users, while also inherently providing benefits in terms of fairness and robustness.
However, existing PFL frameworks focus on improving the performance of personalized models while neglecting the global model. This results in PFL suffering from lower solution accuracy when clients have different kinds of heterogeneous data.
Moreover, these frameworks typically achieve sublinear convergence rates and rely on strong assumptions.
In this paper, we employ the Moreau envelope as a regularized loss function and propose \flame, an optimization framework by utilizing the alternating direction method of multipliers (ADMM) to train personalized and global models. Due to the gradient-free nature of ADMM, \flame alleviates the need for tuning the learning rate during training of the global model. 
We demonstrate that \flame can generalize to the existing PFL and FL frameworks. 
Moreover, we propose a model selection strategy to improve performance in situations where clients have different types of heterogeneous data.
Our theoretical analysis establishes the global convergence and two kinds of convergence rates for \flame under mild assumptions. 
Specifically, under the assumption of gradient Lipschitz continuity, we obtain a sublinear convergence rate. Further assuming the loss function is lower semicontinuous, coercive, and either real analytic or semialgebraic, we can obtain constant, linear, and sublinear convergence rates under different conditions.
We also theoretically demonstrate that \flame is more robust and fair than the state-of-the-art methods on a class of linear problems.
We thoroughly conduct experiments by utilizing six schemes to partition non-i.i.d. data, confirming the performance comparison among state-of-the-art methods.
Our experimental findings show that \flame outperforms state-of-the-art methods in convergence and accuracy, and it achieves higher test accuracy under various attacks and performs more uniformly across clients in terms of robustness and fairness.


\end{abstract}
\begin{IEEEkeywords}
Federated learning, ADMM, global convergence, heterogeneity, personalization, robustness, fairness. 
\end{IEEEkeywords}}

\maketitle

\IEEEdisplaynontitleabstractindextext
\IEEEpeerreviewmaketitle

\section{Introduction}
\IEEEPARstart{F}{ederated} learning (FL) plays a crucial role in the field of artificial intelligence \cite{mcmahan2017communication}, particularly in critical applications such as next-word prediction \cite{hard2018federated}, smart healthcare \cite{nguyen2022federated,antunes2022federated,Hu2023Source}, and its recent integration into emerging large language models (LLMs) \cite{wang2023can,kuang2023federatedscope,zhang2024towards,cho2022heterogeneous}.
In scenarios where privacy issues become very acute, and training becomes particularly challenging, such as in edge computing \cite{Wang2019Adaptive,Ma2021FedSA}, FL enables the collaborative training of models across devices while preserving users' privacy \cite{Sun2023Decentralized,kairouz2021advances,li2021survey,Kummari2024Impact,Li2023FedIPR}.


Despite the advantages of FL in preserving privacy, it still encounters challenges with respect to the statistical heterogeneity of data, affecting its accuracy and convergence \cite{li2020federated,li2020federatedoptimization, Luo2021No,huang2023generalizable}. 
The statistical heterogeneity of data primarily manifests in the non-independent and non-identically distributed (non-i.i.d.) data across different clients \cite{zhang2021survey}.
When training FL models on non-i.i.d. data, the generalization error significantly increases, and the models converge in different directions. 
Beyond accuracy and convergence, statistical heterogeneity also affects fairness in terms of providing a fair quality of service for all participants in the network \cite{Li2020qffl}.
Specifically, an FL system promotes uniform accuracy distribution among clients to ensure \textit{performance fairness}\footnote{All mentions of fairness in this paper refer to performance fairness.} \cite{li2021ditto}. This is closely related to resource allocation, as FL can be viewed as a joint optimization system over a heterogeneous network \cite{Lin2022Personalized, Huang2024Federated}.
Moreover, Li et al. \cite{li2021ditto} found that statistical heterogeneity is a root cause for tension among \textit{accuracy}, \textit{fairness}, and \textit{robustness} of FL, where the robustness of FL refers to the ability against training-time attacks (including data poisoning and model poisoning) \cite{li2021ditto}. 
Exploring statistical heterogeneity in FL is key in paving a path forward to allow for competing constraints of accuracy, robustness, and fairness.

Personalized federated learning (PFL) is a method that aims to mitigate the impact of heterogeneous data by developing personalized models for individual users based on their distinct preferences \cite{vettoruzzo2024advances}.
Numerous strategies have been proposed to achieve PFL. 
A widely recognized strategy is known as \textit{meta-learning}, also referred to as ``learning to learn'' \cite{tan2022towards}. 
Model-agnostic meta-learning (\texttt{MAML}) \cite{finn2017model} is regarded as the pioneering approach to meta-learning,  notable for its ability to generalize effectively and quickly adapt to new heterogeneous tasks.
However, \texttt{MAML} necessitates computing the Hessian term, which poses significant computational challenges.
Several studies, including \cite{nichol2018first,Fallah2020Convergence}, aimed to address this issue by approximating the Hessian matrix.
\texttt{Per-FedAvg} \cite{fallah2020personalized}, inspired by \texttt{MAML}, established a meta-model that can be effectively updated with just one gradient descent step.
Dinh et al. \cite{t2020personalized} expanded \texttt{Per-FedAvg} to introduce a federated meta-learning framework by employing \textit{Moreau envelope} (\pfedme). This framework integrates an $l_2$-norm regularization term, leveraging the global model to optimize personalized models with respect to local data.
Li et al. \cite{li2021ditto} innovatively proposed that PFL can be used to improve accuracy and balance the competing constraints of robustness and fairness.
Lin et al. \cite{Lin2022Personalized} proposed projecting local models into a shared and fixed low-dimensional random subspace and using infimal convolution to control deviations between personalized and projected models, ensuring robustness and fairness while also improving communication efficiency.

However, these PFL frameworks face a common issue: They focus on improving the performance of the personalized model while neglecting the global model. Heterogeneous data can be classified into four types \cite{kairouz2021advances,li2022federated,ye2023heterogeneous}: \textit{label skew}, \textit{feature skew}, \textit{quality skew}, and \textit{quantity skew}. The personalized model typically performs well in the presence of label skew but often struggles to achieve good results with the other types of heterogeneous data. In contrast, the global model tends to perform well across these various types of data.
When clients have different kinds of heterogeneous data, existing PFL frameworks struggle to perform well.
Moreover, current PFL frameworks rely on strong assumptions for convergence, including gradient Lipschitz continuity, bounded variance, and bounded diversity. Nonetheless, they are only capable of achieving a sublinear convergence rate.
Furthermore, these approaches rely on the gradient method to compute inexact solutions for personalized and global models, leading to decreased solution accuracy and model performance. Moreover, the gradient-based method typically requires manual adjustment of the learning rate, a highly sensitive hyperparameter \cite{goodfellow2016deep}.
An excessively large learning rate can trigger model instability or even divergence, whereas an overly small learning rate leads to sluggish convergence and an elevated risk of becoming trapped in local minima. 

The \textit{Alternating Direction Method of Multipliers} (ADMM) is an iterative algorithm that transforms optimization problems into an augmented Lagrangian function and updates primal and dual variables alternately to reach the optimal solution \cite{boyd2011distributed}.
ADMM has been shown to achieve higher solution accuracy in various disciplines, such as matrix completion and separation \cite{xu2012alternating, Shen2014Augmented}, compressive sensing \cite{Chartrand2013nonconvex, Yang2020ADMM-CSNet}, and machine learning \cite{Liu2021Accelerated, Zeng2021Deep,gong2022fedadmm,zhou2023federated, zhu2024SIGMOD}.
Moreover, as a primal-dual scheme, ADMM is more stable. Compared to primal or dual schemes, ADMM typically allows for larger step sizes in gradient updates, which can speed up convergence.
However, there is currently no research applying ADMM to PFL, and the theoretical convergence, fairness, robustness, and experimental performance on different types of non-i.i.d. data partitioning settings remain unknown.

We aim to solve the optimization problem of PFL by leveraging the superior performance offered by ADMM, resulting in improved convergence, accuracy, fairness, and robustness.
Building on this concept, we propose \texttt{FLAME}, a PFL framework with Moreau envelope based on ADMM for training models. 
Specifically, we consider employing the Moreau envelope as clients’ regularized loss function \cite{t2020personalized}. This helps decouple personalized model optimization from global model learning in a bi-level problem framework. Next, we consider transforming the bi-level optimization problem into a multi-block optimization problem and solving it using ADMM, which eliminates the need to adjust the learning rate during the training of the global model. 
When clients have various types of heterogeneous data, we propose a strategy to flexibly choose between personalized and global models.
Moreover, we show that \flame can generalize to existing PFL and FL frameworks by selecting appropriate hyperparameters.

We theoretically establish a convergence analysis for \flame utilizing the Kurdyka-Łojasiewicz (KŁ) inequality framework \cite{kurdyka1998gradients} formulated in \cite{Attouch2013Convergence,wang2019global}. 
Specifically, under the assumption of gradient Lipschitz continuity, we obtain a sublinear convergence rate. Further, assuming that the loss function is lower semicontinuous and either real analytic or semialgebraic (these assumptions are weaker than bounded variance and bounded diversity), we can derive constant, linear, and sublinear convergence rates under different conditions.
Note that our convergence analysis is different from the state-of-the-art method \cite{Attouch2013Convergence,wang2019global} in several aspects, thereby allowing us to attain the aforementioned general convergence results.
Based on \cite{Attouch2013Convergence,wang2019global}, the \textit{sufficient descent}, \textit{relative error}, \textit{continuity conditions}, and the \textit{KŁ property} ensure the global convergence of a nonconvex algorithm. 
We establish sufficient descent and relative error conditions under an inexact ADMM optimization framework while existing theoretical analysis frameworks \cite{Attouch2013Convergence,wang2019global} necessitate all subproblems in ADMM should be solved exactly.
The treatments of this paper are of their own value to the optimization community.
Furthermore, we theoretically analyze the fairness and robustness of \flame. 
By examining test losses and corresponding variances across the network on federated linear regression, we demonstrate that \flame is more robust and fair compared to two state-of-the-art methods \cite{li2021ditto,t2020personalized} under regular conditions.

We conduct comprehensive experiments on various types of heterogeneous data, overcoming the limitation of existing PFL methods that only focus on label skew. 
We compare the performance of \flame with several state-of-the-art PFL methods through experiments conducted on five real-world datasets and six data partition strategies. 
Notably, we experimentally validate the performance of PFL on various types of heterogeneous data for the first time and proposed a strategy for selecting models.
Our experimental findings indicate that under various non-i.i.d. data partitioning schemes, \flame exhibits superior convergence and accuracy compared to state-of-the-art methods.
In terms of robustness and fairness, our experimental results show that \flame achieves higher test accuracy under various attacks and performs more uniformly across clients.
Additionally, we validate the impact of hyperparameters in \flame on its performance.

Our contributions are summarized as follows:
\begin{itemize}
    \item We propose \flame, an ADMM-based PFL framework that boosts the accuracy of both personalized and global models.
    \item We demonstrate that \flame can generalize to other PFL and FL frameworks by configuring certain hyperparameters.
    \item We propose a model selection strategy to improve the performance of \flame in situations where clients have different types of heterogeneous data.
    \item We establish the global convergence and convergence rates for \flame under mild assumptions.
    \item Our theoretical results demonstrate that \flame offers better fairness and robustness compared to state-of-the-art methods under mild conditions.
    \item We generate comprehensive non-i.i.d data distribution cases to validate the accuracy, convergence, robustness, and fairness of several methods on several real-world datasets.
\end{itemize}

The remaining content is structured as follows: Section \ref{sec related work} presents the related work, Section \ref{sce preliminary} introduces the preliminaries used in our study, Section \ref{sec proposed flame} outlines our proposed method, \flame, 
Section \ref{sec:theory} establishes the global convergence, robustness, and fairness results,
Section \ref{sec experiment} validates \flame through experiments, and Section \ref{sec conclusion} concludes our paper.

\section{Related Work}\label{sec related work}
Considering the impact of data heterogeneity on the training model in FL, we first review several specific patterns of data heterogeneity. 
Subsequently, due to our consideration of using ADMM as the optimization method, which is a primal-dual framework, we therefore examine recent research on the integration of the primal-dual framework within FL.
Next, we summarize the PFL research closely related to our work.
Finally, since personalization can provide robustness and fairness for FL, we introduce these concepts.

\myparagraph{Data heterogeneity}
Kairouz et al. \cite{kairouz2021advances} provided a thorough overview of heterogeneous data scenarios from a distribution perspective.
Ye et al. \cite{ye2023heterogeneous} further categorized the statistical heterogeneity of data into four distinct skew patterns: label skew, feature skew, quality skew, and quantity skew. 
Fig. \ref{fig: heterogeneous data} shows examples of the four skew patterns.
Label skew refers to the dissimilarity in label distributions among the participating clients \cite{kairouz2021advances,zhang2022federated}. 
Feature skew denotes a situation in which the feature distributions among participating clients diverge \cite{li2022federated,luo2022disentangled}. 
Quality skew illustrates the inconsistency in data collection quality across different clients \cite{yang2022robust}.
Quantity skew denotes an imbalance in the amount of local data across clients \cite{shang2022federated}. 
These skew patterns lead to local models converging in different directions \cite{ye2023heterogeneous}, thereby resulting in the trained model not being optimal.
Li et al. \cite{li2022federated} conducted thorough experiments to evaluate the effectiveness of current FL algorithms. Their findings indicate that non-i.i.d data does indeed pose significant challenges to the accuracy of FL algorithms during the learning process. 
Chen et al. \cite{chen2021theorem} demonstrated that when data heterogeneity exceeds a certain threshold, purely local training is minimax optimal; otherwise, the global model is minimax optimal. 
In practice, we prefer PFL because it intervenes between the two extremes by interpolating between global and personalized models \cite{Lin2022Personalized}. However, its performance remains uncertain when dealing with different heterogeneous data types across clients.
\begin{figure}[t]
  \centering
  \includegraphics[width=\linewidth]{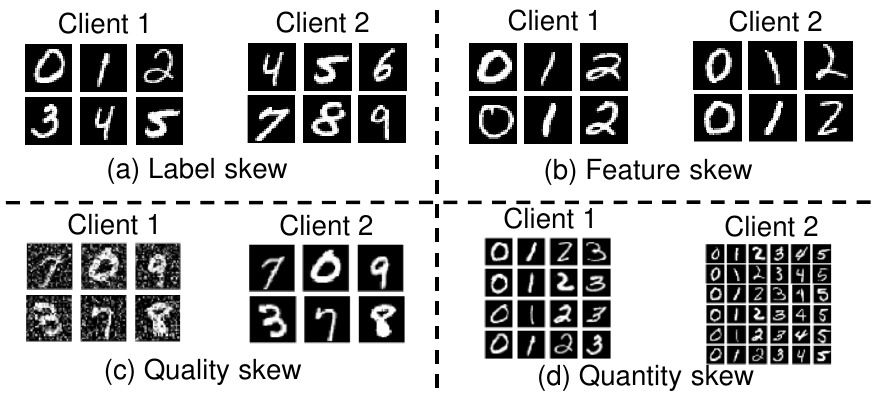}
  \vspace{-1em}
  \caption{Four skew patterns: (a) the labels vary among different clients; (b) the features of the data differ among different clients, manifested as variations in the stroke thickness and slant angle; (c) the data quality varies, notably due to the presence of noise; (d) the quantity of data differs among clients.}
  \label{fig: heterogeneous data}
\end{figure}

\myparagraph{Primal-dual scheme for FL}
From an optimization perspective, we categorize FL into three types: primal scheme \cite{mcmahan2017communication, li2020federatedoptimization, Sai2020SCAFFOLD, Wang2020Tackling, Li2019FedDANE}, dual scheme \cite{Yang2013Trading,ma2015adding,smith2017federated,smith2018cocoa}, and primal-dual scheme \cite{zhang2021fedpd, gong2022fedadmm,zhou2023federated,zhou2023fedgia, Kang2024FedAND}.
Most existing FL frameworks are based on the primal scheme, where each client trains a local model by solving the primal problem via gradient descent, and then the server aggregates these local models. 
In contrast, FL frameworks based on the dual scheme involve solving dual problems, which have been shown to converge faster \cite{smith2018cocoa}. However, the dual scheme is only suitable for convex problems.
In recent years, primal-dual schemes have gained widespread utilization in the context of FL. 
Zhang et al. \cite{zhang2021fedpd} introduced a primal-dual FL framework designed to handle non-convex objective functions. However, this method suffers restrictive assumptions for convergence. Zhou and Li \cite{zhou2023federated} proposed the \texttt{FedADMM}, establishing convergence under mild conditions. Gong et al. \cite{gong2022fedadmm} proposed that within \texttt{FedADMM}, the dual variables can effectively mitigate the impact of data heterogeneity on the training model.
Based on \texttt{FedADMM}, Zhou and Li \cite{zhou2023fedgia} proposed a method that differs from \texttt{FedADMM} by performing a single step of gradient descent on the unselected clients, thereby enhancing communication efficiency.
However, although these primal-dual schemes generally exhibit superior accuracy and convergence performance, they have not yet been applied in PFL.

\myparagraph{Personalized federated learning}
Tan et al. \cite{tan2022towards} categorized the methods of PFL into two classes: global model personalization and learning personalized models. 
Learning personalized models is not the focus of this paper, so interested readers can refer to \cite{Zhu2021Data, Lin2020Ensemble, He2020Group,bistritz2020distributed}.
Tan et al. \cite{tan2022towards} further classified the techniques for personalizing global models into two categories: data-based approaches \cite{zhao2018federated,duan2020self,wu2020fedhome,wang2020optimizing} and model-based approaches \cite{finn2017model,fallah2020personalized,t2020personalized}.
As our proposed method falls within the realm of model-based approaches, we provide a detailed overview of the relevant literature pertaining to model-based approaches.
Finn et al. \cite{finn2017model} considered a model-agnostic meta-learning (MAML) algorithm that is designed to be compatible with different learning problems, enabling the training of a model across various tasks.
However, MAML necessitates computing the Hessian term, which poses significant computational challenges.
Several studies, including \cite{nichol2018first, Fallah2020Convergence}, aimed to address this issue by approximating the Hessian matrix.
Inspired by the principles of MAML, Fallah et al. \cite{fallah2020personalized} introduced \texttt{Per-FedAvg}, which builds a meta-model that can be efficiently updated with a single additional gradient descent step.
Inspired by \texttt{Per-FedAvg}, Dinh et al. \cite{t2020personalized} proposed \texttt{pFedMe} which implements multiple gradient descent steps within the meta-model update process to enhance the solution precision. 
Subsequent studies proposed that employing PFL can enhance the fairness and robustness of the trained models \cite{li2021ditto, Liping2019AAAI, Lin2022Personalized}.
Li et al. \cite{li2021ditto} proposed a framework called \ditto, which was the first to improve the robustness and fairness of FL through personalization.
Lin et al. \cite{Lin2022Personalized} further considered using $l_p$ norm and low-dimensional random projection for regularization. However, the convergence analysis of this method introduces a stronger assumption, namely the low-dimensional condition.
Although their approach can improve the accuracy of personalized models, it renders the global model unusable due to the low-dimensional projection. Consequently, this method becomes ineffective in our scenarios where clients have various types of heterogeneous data.

\myparagraph{Robustness and fairness in FL}
Zhou et al. \cite{Zhou2021Towards} classified fairness in FL into three categories: performance fairness \cite{li2021ditto, Li2020qffl, Li2023term, Lin2022Personalized}, collaboration fairness \cite{Lyu2020Collaborative, xu2020reputation, Yu2020Fairness}, and model fairness \cite{Dwork2012Fairness, Hardt2016Equality}. This paper focuses on performance fairness. In federated networks, the heterogeneity of data across different clients can lead to significant variations in model performance. This issue, referred to as \textit{representation disparity} \cite{Tatsunori2018Fairness}, poses a significant challenge in FL because it can result in unfair outcomes for different clients. 
Next, we present the definition of performance fairness in FL \cite{li2021ditto}. 
\begin{definition}[Performance fairness]
    We define model $\btheta_1$ as fairer than $\btheta_2$ if the test performance distribution of $\btheta_1$ across the network exhibits greater uniformity compared to that of $\btheta_2$. This is quantified by the condition $\var \{f_i(\btheta_1)\} < \var {f_i(\btheta_2)},\,i\!\in\![m]$, where $m$ is the number of clients, $f_i(\cdot)$ denotes the test loss on client $i$, and $\var\{\cdot\}$ denotes variance.
\end{definition}
Li et al. \cite{Li2020qffl, Li2023term} first proposed sample reweighting approaches in FL to promote a more uniform quality of service. 
However, Li et al. \cite{li2021ditto} later found that by increasing the importance of rare devices or data, fairness methods might not be robust, as they can easily overfit to corrupted devices.

Robustness refers to the ability to defend against training-time attacks, which have been extensively studied in previous studies \cite{Biggio2012Poisoning,gu2017badnets,Shafahi2018Poison,Liu2018Trojaning,Huang2020MetaPoison,Xie2020DBA,chen2017targeted,Wang2020Attack,Lyu2024Privacy,Fang2020Local}.
This paper aims to investigate common attacks related to Byzantine robustness \cite{Lamport2019Byzantine}, as formally described below \cite{li2021ditto}.
\begin{definition}[Robustness]
    We are interested in Byzantine robustness \cite{Lamport2019Byzantine}, where malicious devices send arbitrary updates to the server to compromise training. To measure robustness, we compare the mean test performance on benign devices. Specifically, we consider model $\btheta_1$ to be more robust than model $\btheta_2$ if the mean test performance across benign devices is higher for $\btheta_1$ after training with the attack. 
\end{definition}

Robust aggregation is one of the most common strategies to mitigate the effects of malicious updates \cite{Blanchard2017Adversaries, Pillutla2022Robust, sun2019can, Chen2017Distributed, Yin2018Byzantine}. However, it may result in unfair models by filtering out informative updates, particularly in heterogeneous settings. 
Recently, Li et al. \cite{li2021ditto} suggested that personalization can reconcile competing constraints of robustness and fairness. Lin et al. \cite{Lin2022Personalized} further proposed that using low-dimensional random projection can balance communication efficiency, robustness, and fairness simultaneously. 
However, these methods rely on gradient-based optimization methods, and the impact of ADMM on fairness and robustness remains unclear.

\noindent\underline{\textit{\textbf{Remarks}}}.
1) Current PFL methods struggle to perform well when different clients have various types of heterogeneous data;
2) The convergence of existing PFL frameworks still relies on strong assumptions (gradient Lipschitz continuity, bounded variance, and bounded diversity);
3) Existing PFL frameworks rely on gradient-based methods for training models, leading to low model accuracy and difficulties in fine-tuning hyperparameters;
4) The impact of ADMM on the model's accuracy, convergence, robustness, and fairness in PFL is still unknown.


\section{Preliminaries}\label{sce preliminary}

This section describes the notations used throughout the paper, details the definition of FL, introduces the concept of ADMM, and provides an overview of the Moreau envelope.
\subsection{Notations}
We use different text-formatting styles to represent different mathematical concepts: plain letters for scalars, bold letters for vectors, and capitalized letters for matrices. For instance, $m$ represents a scalar, $\boldsymbol{w}$ represents a vector, and $W$ denotes a matrix. Without loss of generality, all training models in this paper are represented using vectors.
We use $[m]$ to represent the set $\{1, 2, ..., m\}$. The symbol $\mathbb{E}$ denotes the expectation of a random variable, and we use ``$:=$" to indicate a definition, while $\mathbb{R}^n$ represents the $n$-dimensional Euclidean space. We represent the inner product of vectors, such as $\langle\boldsymbol{a},\boldsymbol{b}\rangle$, as the sum of the products of their corresponding elements. We use $||\cdot||$ to denote the Euclidean norm of a vector and the spectral norm of a matrix.
We use $\biden$ to represent the identity matrix and $\boldsymbol{1}$ to represent the all-ones matrix.
Table \ref{tab:notation} enumerates the notations used in this paper along with the description.

\begin{table}
\centering
\setlength\tabcolsep{15pt}
\caption{Summary of notations}
\label{tab:notation}
\begin{tabular}{ll}
\toprule
\textbf{Notations} & \textbf{Description}\\
\midrule
$X_i$, $i\in[m]$ & The local dataset \\
$\btheta_i$, $i\in[m]$ & The personalized model\\
$\bw_i$, $i\in[m]$ & The local model\\
$\bpi_i$, $i\in[m]$ & The dual variable\\
$\alpha_i$, $i\in[m]$ & The weight parameter\\
$\bw$ & The global model\\
$\lambda$  & The regularization parameter\\
$\rho$ & The penalty parameter\\
$H$ & The number of local iterations\\
$\mathcal{S}^t$ & The selected clients set in the $t$-th iteration\\
\bottomrule
\end{tabular}
\end{table}

\subsection{Federated Learning}
Consider an FL scenario with $m$ clients, where each client $i$ possesses a local dataset $X_i$ comprising $n_i$ data samples with data distribution $\mathcal{D}_i$. 
These clients are interconnected through a central server and aim to collectively train a model $\bw$ that minimizes the empirical risk \cite{mcmahan2017communication}:
\begin{equation}\label{eq:problem of FL}
    \min_{\bw}\left\{\sum_{i=1}^{m}\alpha_i f_i(\bw):=\sum_{i=1}^{m}\alpha_i\mathbb{E}_{\bx\sim\mathcal{D}_i}[\ell_i(\bw;\bx)]\right\},
\end{equation}
where $\alpha_i$ is a weight parameter, $f_i(\bw):=\mathbb{E}_{\bx\sim\mathcal{D}_i}[\ell_i(\bw;\bx)]$ denotes the expected loss over the data distribution of client $i$, $\bx$ is a random data sample drawn from $\mathcal{D}_i$, and $\ell_i(\bw;\bx)$ denotes the loss function for sample $\bx$ with respect to model $\bw$. Typically, the value of $\alpha_i$ is set to $1/m$ or $n_i/n$, where $n = \sum_{i=1}^m n_i$ is the total number of data points.

\subsection{Alternating Direction Method of Multipliers}
ADMM is an optimization method that belongs to the class of augmented Lagrangian methods and is particularly well-suited for solving the following general problem \cite{boyd2011distributed}:
\begin{align*}
    \min_{\bw\in\mathbb{R}^{r},\boldsymbol{v}\in\mathbb{R}^q}f(\bw)+g(\boldsymbol{v}),\quad\text{s.t. } A\bw+B\boldsymbol{v}-\boldsymbol{b}=\boldsymbol{0},
\end{align*}
where $A\in\mathbb{R}^{p\times r}$, $B\in\mathbb{R}^{p\times q}$, and $\boldsymbol{b}\in\mathbb{R}^{p}$. We directly give the augmented Lagrangian function of the problem as follows,
\begin{equation*}
\begin{split}
    \ml(\bw,\boldsymbol{v},\bpi)&:=f(\bw)+g(\boldsymbol{v})\\
    &+\langle\bpi,A\bw+B\boldsymbol{v}-\boldsymbol{b}\rangle+\frac{\rho}{2}||A\bw+B\boldsymbol{v}-\boldsymbol{b}||^2,
\end{split}
\end{equation*}
where $\bpi\in\mathbb{R}^p$ is the dual variable, and $\rho>0$ is the penalty parameter. After initializing the variables with $(\bw^0,\boldsymbol{v}^0,\bpi^0)$, ADMM iteratively performs the following steps:
\begin{equation*}
\begin{split}
\bw^{t+1}&=\operatorname{argmin}_{\mathbf{w} \in \mathbb{R}^r} \mathcal{L}(\bw, \boldsymbol{v}^t, \boldsymbol{\pi}^t), \\
\boldsymbol{v}^{t+1}&=\operatorname{argmin}_{\boldsymbol{v} \in \mathbb{R}^q} \mathcal{L}(\bw^{t+1}, \boldsymbol{v}, \boldsymbol{\pi}^t), \\
\boldsymbol{\pi}^{t+1}&=\boldsymbol{\pi}^t+\rho(A \bw^{t+1}+B \boldsymbol{v}^{t+1}-\boldsymbol{b}).
\end{split}
\end{equation*}

ADMM exhibits distributed and parallel computing capabilities, effectively addresses equality-constrained problems, and provides global convergence guarantees \cite{wang2019global}, making it particularly well-suited for tackling large-scale optimization problems. It finds widespread applications in distributed computing, machine learning, and related fields.
\subsection{Moreau Envelope}
The Moreau envelope is an essential concept in the fields of mathematics and optimization \cite{moreau1965proximite}. It finds widespread application in convex analysis, non-smooth optimization, and numerical optimization. Here, we present the definition.
\begin{definition}[Moreau envelope \cite{rockafellar2009variational}]
Consider a function $f:\mathbb{R}^p\rightarrow\mathbb{R}$, its \textit{Moreau envelope} is defined as:
\begin{equation}
    F(\bw):=\min_{\btheta\in\mathbb{R}^p}f(\btheta)+\frac{\lambda}{2}||\bw-\btheta||^2,
\end{equation}
where $\lambda$ is a hyperparameter. Its associated proximal operator is defined as follows,
\begin{equation}
    \!\hat{\btheta}(\bw):=\text{prox}_{f/\lambda}(w)=\argmin_{\btheta\in\mathbb{R}^p}\left\{f(\btheta)+\frac{\lambda}{2}||\btheta-\bw||^2 \right\}\!.\!
\end{equation}
\end{definition}

The Moreau envelope provides a smooth approximation of the original function $f$. This approximation is helpful when dealing with optimization algorithms that require smooth functions. As $\lambda$ becomes smaller, the Moreau envelope approaches the original function, making it useful for approximating and optimizing non-smooth functions.
Next, we describe a useful property of Moreau envelope \cite{rockafellar2009variational}.
\begin{proposition}\label{proposition: smooth of moreau envelope}
If $f$ is a proper, lower semicontinuous, and weakly convex (or nonconvex with $L$-\textit{Lipschitz} $\nabla f$) function, then $F$ is $L_F$-smooth with $L_F=\lambda$ (with the condition that $\lambda>2L$ for nonconvex $L$-smooth $f$), and the gradient of $F$ is defined as
\begin{equation}
    \nabla F(\bw)=\lambda(\bw-\hat{\btheta}(\bw)).
\end{equation}
\end{proposition}

\section{Proposed \texttt{FLAME}}\label{sec proposed flame}
In this section, we first present the formulation of the optimization problem along with the stationary points for PFL based on ADMM. We then provide an algorithmic description of \flame along with a specific example. Following this, we propose a model selection strategy to adapt to FL scenarios with different types of heterogeneous data.
Finally, we demonstrate that \flame can generalize to the existing PFL and FL frameworks by configuring certain hyperparameters.

\subsection{Problem Formulation}
To construct the objective function for PFL, we employ the approach outlined in \cite{t2020personalized}, where $f_i$ is substituted by the Moreau envelope of $f_i$ in the optimization Problem (\ref{eq:problem of FL}).
The specific formulation of the problem is presented as follows: 
\begin{equation}\label{eq: pfedme}
\begin{split}
    &\min_{\bw}\!\sum_{i=1}^{m}\alpha_i F_i(\bw),\\
    \text{where } F_i(\bw):=&\min_{\btheta_i}f_i(\btheta_i)+\frac{\lambda}{2}||\btheta_i-\bw||^2,\quad i\in[m].
\end{split}
\end{equation} 

Note that $F_i$ is the Moreau envelope of $f_i$, and $\btheta_i$ is the personalized model of client $i$. The hyperparameter $\lambda$ controls the influence of the global model $\bw$ on the personalized model $\btheta_i$. 
A higher value of $\lambda$ provides an advantage to clients with unreliable data by harnessing extensive data aggregation, whereas a lower $\lambda$ places greater emphasis on personalization for clients with a substantial amount of useful data. Note that $\lambda\in(0,\infty)$ is used to prevent extreme cases where $\lambda=0$ (no FL) or $\lambda=\infty$ (no PFL). 
The overall concept is to enable clients to develop their personalized models in different directions while remaining close to the global model $\bw$ contributed by every client.
Note that Problem (\ref{eq: pfedme}) is a bi-level optimization problem. The conventional approach to solving bi-level problems typically involves initially using a first-order gradient method to solve $\btheta_i$ in the lower-level problem, obtaining an approximate solution. This approximate solution is then incorporated into the upper-level problem, followed by another round of the first-order gradient method to solve $\bw$ in the upper-level problem. Iterating through this process multiple times yields the final solutions. Even though the first-order gradient method is simple, it suffers from low solution accuracy and is cumbersome to fine-tune parameters like the learning rate. To address these issues, we propose a relaxed form of Problem (\ref{eq: pfedme}) as follows,
\begin{equation}\label{eq: bi-variable problem}
\min_{\bw,\Theta}\Bigl\{f(\Theta,\bw):=\sum_{i=1}^{m}\alpha_i \Bigl(f_i(\btheta_i)+\frac{\lambda}{2}||\btheta_i-\bw||^2\Bigr)\Bigr\},
\end{equation}
where $\Theta:=\{\btheta_i\}_{i=1}^m$ is the set of personalized models. Note that Problem (\ref{eq: bi-variable problem}) is a multi-block optimization problem with respect to $\bw$ and $\Theta$. It is obvious that Problem (\ref{eq: bi-variable problem}) serves as a lower bound for Problem (\ref{eq: pfedme}).
We aim to learn an optimal personalized model $\Theta^*$ and an optimal global model $\bw^*$ that minimizes $f(\Theta,\bw)$. That is
\begin{align}
    \Theta^*,\bw^*:=\argmin_{\Theta,\bw} f(\Theta,\bw),
\end{align}
and the corresponding optimal function value is given as
\begin{align}
    f^*:=f(\Theta^*,\bw^*).
\end{align}

Given that ADMM, as a primal-dual method, is often regarded as more iteration-stable and converges faster compared to gradient-based approaches, we consider employing ADMM to solve Problem \eqref{eq: bi-variable problem}.
Firstly, we consider introducing the auxiliary variable $W:=\{\bw_i\}_{i=1}^m$ to transform Problem (\ref{eq: bi-variable problem}) into a separable form (with respect to a partition or splitting of the variable into multi-block variables). That is 
\begin{equation}\label{eq:problem pfl}
\begin{split}
    &\min_{\btheta_i,\bw_i,\bw}\Bigr\{\tilde{f}(\Theta,W)\!:=\!\sum_{i=1}^{m}\alpha_i \Bigl(f_i(\btheta_i)+\frac{\lambda}{2}||\btheta_i-\bw_i||^2\Bigr)\Bigl\},\!\!\\
    &\,\,\,\text{ s.t. }\,\bw_i=\bw,\quad i\in[m],
\end{split}
\end{equation}     
where $\bw_i$ can be regarded as the \textit{local model} of client $i$. Note that Problem (\ref{eq:problem pfl}) is equivalent to Problem (\ref{eq: bi-variable problem}) in the sense that the optimal solutions coincide.
We consider using the exact penalty method, ADMM, to solve Problem \eqref{eq:problem pfl}, as it involves linear constraints and multiple block variables, making ADMM well-suited for this optimization problem. 
Moreover, in Section \ref{sec: Connections with Existing Work}, we will analyze in detail how other PFL frameworks can be seen as solving Problem \eqref{eq:problem pfl} by using an inexact penalty method.
To implement ADMM for Problem (\ref{eq:problem pfl}), we establish the corresponding augmented Lagrangian function as follows:
\begin{equation}\label{eq: lagrangian function}
\begin{split}
    \ml(\Theta,W,\Pi,\bw)&:=\sum_{i=1}^m\ml_i(\btheta_i,\bw_i,\bw,\bpi_i), \\
    \ml_i(\btheta_i,\bw_i,\bpi_i,\bw)&:=\alpha_i(f_i(\btheta_i)+\frac{\lambda}{2}||\btheta_i\!-\bw_i||^2)\\
    &+\langle\bpi_i,\bw_i-\bw\rangle+\frac{\rho}{2}||\bw_i-\bw||^2,
\end{split}
\end{equation} 
where $\Pi:=\left\{\bpi_i\right\}_{i=1}^{m}$ is the set of dual variables, and $\rho>0$ is the penalty parameter. The ADMM framework for solving Problem (\ref{eq:problem pfl}) can be summarized as follows: after initializing the variables with $(\Theta^0,W^0,\Pi^0,\bw^0)$, the following update steps are executed iteratively for each $t\geq0, $
\begin{align}
\btheta_i^{t+1}&=\operatorname{argmin}_{\btheta_i}\ml_i(\btheta_i,  \bw_i^t, \boldsymbol{\pi}^t_i,\bw^{t}),\label{eq12}\\ 
\bw_i^{t+1} & =\operatorname{argmin}_{\bw_i} \ml_i(\btheta_i^{t+1}, \bw_i, \boldsymbol{\pi}^t_i, \bw^{t})  \notag\\ 
&=\frac{1}{\lambda\alpha_i+\rho}(\lambda\alpha_i\btheta_i^{t+1}+\rho\bw^t-\bpi_i^t),\label{eq13}\\
\boldsymbol{\pi}_i^{t+1} & =\boldsymbol{\pi}_i^t+\rho(\bw_i^{t+1}-\bw^{t}),\label{eq14}\\
\bw^{t+1} & =\operatorname{argmin}_{\bw} \mathcal{L}(\Theta^{t+1}, W^{t+1}, \bw, \Pi^{t+1})\notag\\
&=\frac{1}{m} \sum_{i=1}^m(\bw_i^{t+1}+\frac{1}{\rho}\boldsymbol{\bpi}_i^{t+1})\label{eq15}. 
\end{align}

Note that we do not provide a closed-form solution for $\btheta_i$ due to the possibly non-convex nature of $f_i$. We will discuss this issue in detail in Section \ref{sec:Algorithmic Design}.

\subsection{Stationary Points}
We present the optimal conditions of Problems \eqref{eq: bi-variable problem} and \eqref{eq:problem pfl}.
\begin{definition}[Stationary point]\label{definition: optimal condition}
    A point $(\Theta^*, \bw^*)$ is a stationary point of Problem \eqref{eq: bi-variable problem} if it satisfies
\begin{align}\label{eq: optimal condition of (6)}
\left\{\begin{aligned}
\nabla f_i(\btheta_i^*)+\lambda(\btheta_i^*-\bw^*)& =0, & i \in[m], \\
\bw^* - \sum_{i=1}^m\alpha_i\btheta_i^*& =0. &  \\
\end{aligned}\right.    
\end{align}
A point $(\Theta^*,W^*,\bw^*,\Pi^*)$ is a stationary point of Problem \eqref{eq:problem pfl} if it satisfies
\begin{align}\label{eq: optimal condition of (7)}
\left\{\begin{aligned}
\nabla f_i(\btheta_i^*)+\lambda(\btheta_i^*-\bw_i^*)& =0, & i \in[m], \\
\alpha_i\lambda(\bw_i^* - \btheta_i^*)+\bpi_i^*& =0, &  i \in[m],\\
\bw_i^* - \bw^*& =0, & i \in[m], \\
\sum_{i=1}^m \bpi_i^*& =0. &  \\
\end{aligned}\right.    
\end{align}

\end{definition}

Definition \ref{definition: optimal condition} indicates that a locally optimal solution of Problem \eqref{eq: bi-variable problem} (resp. \eqref{eq:problem pfl}) must satisfy \eqref{eq: optimal condition of (6)} (resp. \eqref{eq: optimal condition of (7)}). When $f_i$ is convex for any $i\in[m]$, then a point is the globally optimal solution of Problem \eqref{eq: bi-variable problem} (resp. \eqref{eq:problem pfl}) if and only if it satisfies \eqref{eq: optimal condition of (6)} (resp. \eqref{eq: optimal condition of (7)}). Moreover, a stationary point $(\Theta^*,W^*,\bw^*,\Pi^*)$ of Problem \eqref{eq:problem pfl} can imply \eqref{eq: optimal condition of (6)}, which indicates that $(\Theta^*,\bw^*)$ is also a stationary point of Problem \eqref{eq: bi-variable problem}.

\subsection{Algorithmic Design}\label{sec:Algorithmic Design}
In Algorithm \ref{alg: flame on server side}, we introduce \flame, a PFL framework that employs ADMM to solve Problem (\ref{eq:problem pfl}). 
We employ ADMM for training models due to the following advantages: the subproblems for solving $\bw_i$ and $\bw$ are convex, allowing direct derivation of closed-form solutions. In contrast, other PFL frameworks \cite{fallah2020personalized,t2020personalized,li2021ditto} employ gradient-based methods for solving, resulting in lower solution accuracy and necessitating adjustments to learning rates.
Fig. \ref{fig example of flame} shows an example of our algorithm. In each communication round $t$, the server selects $s$ clients from the entire client pool to form $\mathcal{S}^t$ (Line \ref{line 3 alg 1}). The selected clients update their local parameters using Equations (\ref{eq12})-(\ref{eq14}) (Lines \ref{line 8 alg 1}-\ref{line 14 alg 1}). However, when $f_i$ is non-convex, obtaining a closed-form solution of (\ref{eq12}) may be challenging. Therefore, for Problem (\ref{eq12}), we employ gradient descent to iteratively update the personalized model until we get an $\epsilon_i^{t+1}$-approximate solution. That is
\begin{equation}\label{eq16}
\begin{split}
&||\nabla_{\btheta_i}\ml(\btheta_i^{t+1},\bw_i^t,\bpi_i^t,\bw^t)||^2\\
=&||\alpha_i(\nabla f_i(\btheta_i^{t+1})+\lambda(\btheta_i^{t+1}-\bw_i^t)||^2\leq\epsilon_i^{t+1}.    
\end{split}
\end{equation}      

Note that Equation (\ref{eq16}) can be satisfied after $H:=\mathcal{O}(\zeta\log(\frac{r}{ \epsilon^{t+1}_i}))$ iterations, where $\zeta$ is a condition number measuring the difficulty of optimizing Problem (\ref{eq12}), $r$ is the diameter of the search space \cite{bubeck2015convex}. Given a user-defined constant $\upsilon_i\in(0,1)$, we additionally establish the condition 
\begin{align}\label{eq: iteration condition}
\left\{\begin{aligned}
\epsilon^{t+1}_i &\leq\upsilon_i \epsilon_i^{t},  &\,\,\,i \in \mathcal{S}^t, \\
\epsilon^{t+1}_i&=  \epsilon^{t}_i,  \,\,\, &i  \notin \mathcal{S}^t.
\end{aligned}\right.    
\end{align}
which is useful in the convergence analysis.
For clients not included in $\mathcal{S}^t$, their local parameters remain unchanged (Line \ref{line 16 alg 1}).
After the completion of local parameter updates for each client, the update messages $\{\bu_i^{t+1}\}_{i=1}^m$ are sent to the server for updating the global model (Lines \ref{line 5 alg 1} and \ref{line 6 alg 1}).

\begingroup
\begin{algorithm}[t]
  \caption{\flame}
  \label{alg: flame on server side}
  \SetAlgoLined
  \KwIn{$T:$ the total communication rounds, $\rho$: the penalty parameter, $\lambda$: the regularization parameter, $m$: the number of clients, $X_i,i\in[m]$: the local dataset, $\eta$: the learning rate, $H$: the number of local iterations.}
  \textbf{Initialize:} $\btheta_i^0,\bw_i^0,\bpi_i^0,\bu_i^{0}=\bw_i^{0}+\frac{1}{\rho}\bpi_i^{0}$, $i\in[m]$.
  
  \For{$t=0,1,\dots, T-1$}{
  \tcc{On the server side.}
  Randomly select $s$ clients $\mathcal{S}^t\subset[m]$\label{line 3 alg 1}\;

  Call each client to upload $\{\boldsymbol{u}^{t}_i\}_{i=1}^m$ to the server\;
  

  Update $\bw^{t}=\frac{1}{m}\sum_{i=1}^{m}\bu_i^{t}$\label{line 5 alg 1}\;
  

  Broadcast $\bw^{t}$ to the selected clients\label{line 6 alg 1}\;
  
  
  \tcc{On the client side.}

  \For{each client $i\in\mathcal{S}^t$}{
  Create Batches $\mathcal{B}$\label{line 8 alg 1}\;
  
  
  \For{$h=0,1,\dots,H-1$}{
    \For{batch $\xi\in\mathcal{B}$}{
    Compute the gradient $f_i(\btheta_i^{t,h},\xi)$\;

    $\btheta_i^{t,h+1}=\btheta_i^{t,h}-\eta(\nabla f_i(\btheta_i^{t,h},\xi)+\lambda(\btheta_i^{t,h}-\bw_i^t))$\;

    }
  }
    $\btheta_i^{t+1}=\btheta_i^{t,H-1}$\;
    
    $\bw_i^{t+1}=\frac{1}{\lambda\alpha_i+\rho}(\lambda\alpha_i\btheta_i^{t+1}+\rho\bw^t-\bpi_i^t)$\;

    $\bpi_i^{t+1}=\bpi_i^t+\rho(\bw_i^{t+1}-\bw^t)$\;
    

    $\bu_i^{t+1}=\bw_i^{t+1}+\frac{1}{\rho}\bpi_i^{t+1}$\label{line 14 alg 1}\;
    
  }
  \For{each client $i\notin\mathcal{S}^t$}{
  $(\btheta_i^{t+1},\bw_i^{t+1},\bpi_i^{t+1},\bu_i^{t+1})=(\btheta_i^{t},\bw_i^{t},\bpi_i^{t},\bu_i^{t})$\label{line 16 alg 1}\;
  
  }
  }
\KwRet $\bw\,(global),\, \{\btheta_i\}_{i=1}^m\,(personalized)$.




 

\end{algorithm}    
\setlength{\textfloatsep}{2pt}
\endgroup

\subsection{Model Selection}
Since different clients often have various types of heterogeneous data, we refer to this situation as a \textit{hybrid skew}. As shown in Fig. \ref{fig example of flame}, client 1 has label skew data, while client 2 has quantity skew data.
Personalized models often perform well under label skew but generally lag behind global models in other types of heterogeneous data, as validated in our numerical experiments. 
Therefore, in the context of PFL, we propose a model selection strategy: after training models for each client, we choose the model with better performance between personalized and global models for deployment, which we refer to as a \textit{hybrid model}. 
However, existing PFL frameworks usually do not focus on the performance of global models \cite{t2020personalized,li2021ditto}. They even make the global model unusable by mapping it to a low-dimensional space through low-dimension projection \cite{Lin2022Personalized}. Due to the poor performance of the global model, these methods typically do not perform well on heterogeneous data with hybrid skew. In contrast, \flame uses ADMM as an optimization method, which can improve the accuracy of both personalized and global models, thereby adapting to various types of heterogeneous data.


\begin{figure}
  \centering
  \includegraphics[width=\linewidth]{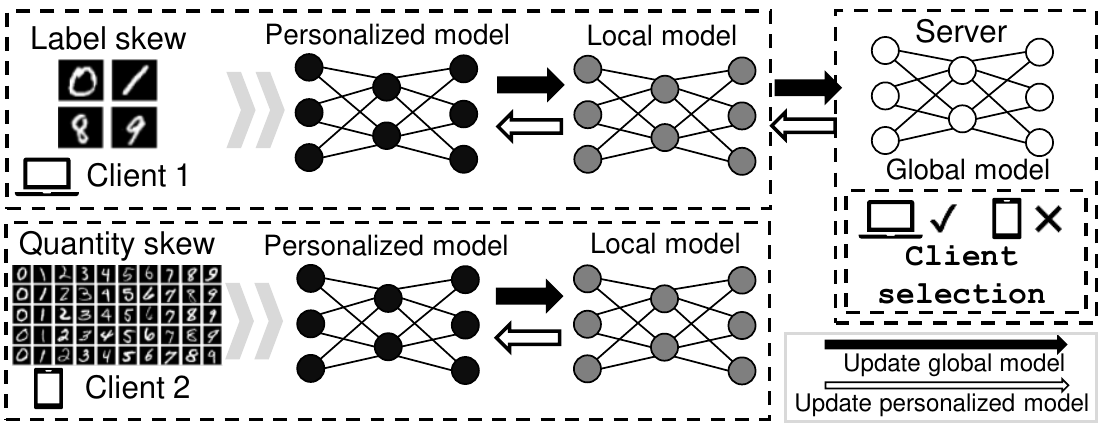}
  \caption{An example of \texttt{FLAME}. Various clients may have different types of non-i.i.d. data (label and quantity skew). Selected clients download the global model from the server, and update their personalized models and local models, while unselected clients keep their model parameters unchanged. Client selection is performed on the server, and update messages $\{\bu_i\}_{i=1}^m$ uploaded by clients are used to update the global model.}
  \label{fig example of flame}
\end{figure}

\subsection{Connections with Existing Work}\label{sec: Connections with Existing Work}
In this section, we demonstrate the connections and comparison between \flame and other FL frameworks: \pfedme \cite{t2020personalized}, \fedadmm \cite{gong2022fedadmm,zhou2023federated}, and \fedavg \cite{mcmahan2017communication}. Note that \pfedme is a PFL framework, while \fedadmm and \fedavg are different FL frameworks. 
We found that \pfedme can be considered as an inexact penalty method. Specifically, \pfedme can be regarded as solving the following alternative minimization problem using gradient-based methods:
\begin{equation*}
\begin{split}
    &\min_{\Theta,W,\bw} \overline{\ml}(\Theta,W,\bw):=\sum_{i=1}^m\overline{\ml}_i(\btheta_i,\bw_i,\bw), \\
    \overline{\ml}_i(\btheta_i,\bw_i,&\bw):=\alpha_i(f_i(\btheta_i)+\frac{\lambda}{2}||\btheta_i\!-\bw_i||^2)+\frac{\rho}{2}||\bw_i-\bw||^2.
\end{split}
\end{equation*}

Note that the difference between $\overline{\ml}(\Theta,W,\bw)$ and $\ml(\Theta,W,\Pi,\bw)$ lies in the fact that $\overline{\ml}(\Theta,W,\bw)$ does not include the dual variable.
When we set $\bpi_i=\textbf{0}$ in $\ml(\Theta, W,\Pi,\bw)$, and iteratively solve $\btheta_i$, $\bw_i$, and $\bw$ (solving $\btheta_i$ and $\bw_i$ via first-order gradient method), we can recover the training framework of \pfedme. 
Next, we consider the connections among \flame, \fedadmm, and \fedavg.
The optimization problem of \fedadmm is established as \cite{zhou2023federated}
\begin{align}
    \min_{\Theta,\bw} \sum_{i=1}^m \alpha_i f_i(\btheta_i),\,\, \text{s.t. }\btheta_i = \bw,\, i\in[m].
\end{align}
Recalling Problem \eqref{eq:problem pfl}, if we set $\lambda\to \infty$, we can infer $\btheta_i=\bw_i$.
Consequently, Problem \eqref{eq:problem pfl} generalizes to the optimization problem of \fedadmm.
Moreover, if we further set the dual variable to zero in the optimization framework of \fedadmm \cite{gong2022fedadmm}, we can recover \fedavg.

\section{Theoretical Analysis}\label{sec:theory}
\subsection{Main Assumptions}
We begin by providing the definitions of \textit{graph}, \textit{semicontinuous}, \textit{real analytic}, and \textit{semialgebraic} functions that are used in our assumptions.
\begin{definition}[Graph]
    Let $f:\mathbb{R}^p\to\mathbb{R}\cup\{+\infty\}$ be an extended real-valued function, its graph is defined by
\begin{align*}
    \text{Graph}(f):=\{(\boldsymbol{x},y)\in\mathbb{R}^p\times\mathbb{R}:y=f(\boldsymbol{x})\},
\end{align*}
and its domain is defined by $\text{dom}(f):=\{\boldsymbol{x}\in\mathbb{R}^p:f(\boldsymbol{x})<+\infty\}$. If $f$ is a proper function, i.e., $\text{dom}(f)\neq\emptyset$, then the set of its global minimizers is defined by
\begin{align*}
    \argmin f:=\{\boldsymbol{x}\in\mathbb{R}^p:f(\boldsymbol{x})=\inf f\}
\end{align*}
\end{definition}

\begin{definition}[Semicontinuous]
    A function $f:\mathcal{X}\to \mathbb{R}$ is called lower semicontinuous if for any $x_0\in\mathcal{X}$, 
\begin{align*}
    \lim_{x\to x_0}\inf f(x)\geq f(x_0).
\end{align*}
\end{definition}

\begin{definition}[Real analytic function\cite{krantz2002primer}]\label{definition: real analytic}
    A function $f$ is real analytic on an open set $\mathcal{X}$ in the real line if for any $x_0\in \mathcal{X}$, $f(x)$ can be represented as
\begin{align*}
   f(x)=\sum_{i=1}^{+\infty}a_i(x-x_0)^i,
\end{align*}
where the coefficients $\{a_i\}_{i=1}^{+\infty}$ are real numbers and the series is convergent to $f(x)$ for $x$ in a neighborhood of $x_0$.
\end{definition}

\begin{definition}[Semialgebraic set and function\cite{bochnak2013real}]\label{definition: Semialgebraic}
       \quad
\begin{itemize}
    \item[a)] A set $\mathcal{X}$ is called semialgebraic if it can be represented by  
    \begin{align*}
        \mathcal{X}=\cup_{i=1}^r\cap_{j=1}^s\{\bx\in\mathbb{R}^p: P_{ij}(\bx)=0,Q_{ij}(\bx)>0\},
    \end{align*}
    where $P_{ij}$ and $Q_{ij}$ are real polynomial functions for $1\leq i\leq r$ and $1\leq j\leq s$.
    \item[b)] A function $f$ is called semialgebraic if its graph $Graph(f)$ is semialgebraic.
\end{itemize}
\end{definition}

According to \cite{law1965ensembles,bochnak2013real,shiota1997geometry}, the class of semialgebraic sets is stable under operations such as finite union, finite intersection, Cartesian product, and complementation. Common examples include polynomial functions, the indicator function of a semialgebraic set, and the Euclidean norm.
Next, we present the assumptions used in our convergence analysis.
\begin{assumption}\label{assumption: either real analytic or semialgebraic}
Suppose that
\begin{itemize}
    \item[a)] the expected loss function $f_i,\,i\in[m]$ is a proper lower semicontinuous and nonnegative function,
    \item[b)] the expected loss function $f_i,\,i\in[m]$ is either real analytic or semialgebraic.
\end{itemize}
\end{assumption}

\begin{assumption}[Gradient Lipschitz continuity]\label{assmption lipschitz}
    The expected loss function $f_i$ is $L$-smooth (gradient Lipschitz continue), i.e., for $\forall \btheta_1, \btheta_2$, the following inequality holds
\begin{align}\label{eq: gradient lipschitz}
\|\nabla f_i(\btheta_1)-\nabla& f_i(\btheta_2)\| \leq L\|\btheta_1-\btheta_2\|.
\end{align}
\end{assumption}

\begin{assumption}\label{assumption: coercive}
    The expected loss function $f_i$ is coercive. That is, $f_i(\btheta_i)\to+\infty$ when $\theta_i\to+\infty$.
\end{assumption}

According to \cite{krantz2002primer,law1965ensembles,bochnak2013real,shiota1997geometry}, most of the commonly used loss functions, such as squared, logistic, hinge, or cross-entropy losses, can be verified to satisfy Assumption \ref{assumption: either real analytic or semialgebraic}.
Assumption \ref{assmption lipschitz} is a standard assumption in the convergence analysis of FL, as evidenced by \cite{t2020personalized,zhou2023federated,zhou2023fedgia,li2019convergence}.
Assumption \ref{assumption: coercive} is widely used in establishing the convergence properties of optimization algorithms to bound the generated sequence, as evidenced by \cite{Zeng2019Global,zhou2023fedgia,Zeng2021Deep}.
Our assumptions for convergence are weaker than the assumptions of other PFL frameworks \cite{t2020personalized,li2021ditto}, which include gradient Lipschitz continuity, bounded variance, and bounded diversity.

\subsection{Global Convergence}
Our global convergence analysis is based on the analytical framework delineated in \cite{Attouch2013Convergence}. 
Let $\mathcal{P}^t:=(\Theta^t, W^t, \Pi^t, \bw^t)$, then by defining the Lyapunov function $\tilde{\ml}(\mathcal{P}^t)$ as 
\begin{equation}
\begin{split}
    \tilde{\ml}(\mathcal{P}^t):=\ml(\mathcal{P}^t)+\sum_{i=1}^m \iota_i \epsilon_i^t,
\end{split}  
\end{equation}
where $\iota_i=\frac{\alpha_i^2}{\bigl((\frac{1}{\rho^2}-1)\lambda^2\alpha_i^2-\frac{L\alpha_i+\rho}{2}\bigr)(1-\upsilon_i)}$,
the analytical framework identifies four crucial components: ensuring \textit{sufficient descent} and \textit{relative error} conditions, verifying the \textit{continuity} condition, and confirming the \textit{Kurdyka-Łojasiewicz (KŁ) property} of sequences $\{\tilde{\ml}(\mathcal{P}^t)\}$. 
Next, we
present two key lemmas of sufficient decent and relative error, while leaving other details in Appendix \ref{Appendix: convergence}.
\begin{lemma}[Sufficient descent] \label{lemma1}
Suppose that Assumption \ref{assmption lipschitz} holds. Let $\{\mathcal{P}^t\}$ denote the sequence generated by Algorithm \ref{alg: flame on server side}, let each client $i$ choose $\frac{\lambda^2\alpha_i^2(1+\rho)}{\rho^2}-\frac{\lambda\alpha_i+\rho}{2}<0$, $\rho\geq\lambda\alpha_i$, and $\frac{1-\rho^2}{\rho^2}\lambda^2\alpha_i^2-\frac{L\alpha_i+\rho}{2}>0$, then for all $t\geq 0$, it holds that
\begin{align*}
    &\tilde{\ml}(\mathcal{P}^{t})-\tilde{\ml}(\mathcal{P}^{t+1})\geq\mathcal{D}_1\Delta\Gamma^{t+1},\,\text{where}\\
    \Delta\Gamma^{t+1} \!:=\!\sum_{i=1}^m(&\|\bw^{t+1}\!-\!\bw^{t}\|^2\!+\!\|\bw^{t+1}_i\!-\!\bw_i^{t}\|^2\!+\!\|\btheta_i^{t+1}\!-\!\btheta_i^t\|^2),
\end{align*}
and $\mathcal{D}_1:=\min_i\{\frac{\rho}{2},\frac{\lambda\alpha_i+\rho}{2}-\frac{\lambda^2\alpha_i^2(1+\rho)}{\rho^2}\}$.
\end{lemma}

\begin{lemma}[Relative error]\label{lemma: relative error}
    Suppose that Assumption \ref{assmption lipschitz} holds. Let $\{\mathcal{P}^t\}$ denote the sequence generated by Algorithm \ref{alg: flame on server side}, and let $\varepsilon^{t+1}=\sum_{i=1}^m\epsilon_{i}^{t+1}$, then for all $t\geq 0$, it holds that
\begin{align*}
    \|\partial \tilde{\ml}(\mathcal{P}^t)\|^2\leq D_2(\Delta\Gamma^{t+1}+\varepsilon^{t+1}), \,\text{where}\\
    \partial\tilde{\ml}(\mathcal{P}^t):=(\{\partial_{\Theta}\tilde{\ml}\},\{\partial_{W}\tilde{\ml}\},\{\partial_{\Pi}\tilde{\ml}\},\{\partial_{\bw}\tilde{\ml}\})(\mathcal{P}^t),
\end{align*}
and $D_2:=\max_i\{4\lambda^2\alpha_{i}^2(1+\frac{1}{\rho^2})+2\alpha_{i}^2(L^2+2\lambda^2)+2\rho^2+2\}$.
\end{lemma}

Note that in the relative error, there exists an error term $\varepsilon^{t+1}$ on the right-hand side of the inequality. The current theoretical framework cannot accommodate the presence of this error term \cite{wang2019global}. Next, we provide an analysis of global convergence in the presence of this error term.

\begin{theorem}\label{theorem: objective function values convergence}
    Suppose that Assumption \ref{assmption lipschitz} holds, let each client $i$ choose $\frac{\lambda^2\alpha_i^2(1+\rho)}{\rho^2}-\frac{\lambda\alpha_i+\rho}{2}<0$, $\rho\geq\lambda\alpha_i$, and $\frac{1-\rho^2}{\rho^2}\lambda^2\alpha_i^2-\frac{L\alpha_i+\rho}{2}>0$, then the following results hold.
\begin{itemize}
    \item[a)] Sequence $\{\mathcal{P}^t\}$ is bounded.
    \item[b)] Sequences $\{\ml(\mathcal{P}^t)\}$, $\{f(\Theta^t,\bw^t)\}$, and $\{\tilde{f}(\Theta^t,W^t)\}$ converge to the same value, i.e.,
    \begin{align}
        \lim_{t\to \infty}\ml(\mathcal{P}^t)=\lim_{t\to \infty}f(\Theta^t,\bw^t)=\lim_{t\to \infty}\tilde{f}(\Theta^t,W^t).
    \end{align}
    \item[c)] $\nabla_{\Theta}f(\Theta^t,\bw^{t})$, $\nabla_{\bw}f(\Theta^t,\bw^{t})$, $\nabla_{\Theta}\tilde{f}(\Theta^t,W^{t})$, and $\nabla_{W}\tilde{f}(\Theta^t,W^{t})$ eventually vanish, i.e.,
    \begin{align}
         \lim_{t\to \infty}\nabla_{\Theta}f(\Theta^t,\bw^{t})&=\lim_{t\to \infty}\nabla_{\Theta}\tilde{f}(\Theta^t,W^{t})=0,\\
         \lim_{t\to \infty}\nabla_{\bw}f(\Theta^t,\bw^{t})&=\lim_{t\to \infty}\nabla_{W}\tilde{f}(\Theta^t,W^{t})=0.
    \end{align}
\end{itemize}
\end{theorem}

Theorem \ref{theorem: objective function values convergence} establishes the convergence property of the objective function values. We next consider the convergence property of the sequences $\{\mathcal{P}^t\}$.
\begin{theorem}\label{theorem: sequences convergence}
    Suppose that Assumptions \ref{assmption lipschitz} and \ref{assumption: coercive} hold, let each client $i$ choose $\frac{\lambda^2\alpha_i^2(1+\rho)}{\rho^2}-\frac{\lambda\alpha_i+\rho}{2}<0$, $\rho\geq\lambda\alpha_i$, and $\frac{1-\rho^2}{\rho^2}\lambda^2\alpha_i^2-\frac{L\alpha_i+\rho}{2}>0$, then the following results hold.
    \begin{itemize}
        \item[a)] The accumulating point $\mathcal{P}^{\infty}$ of sequences $\{\mathcal{P}^t\}$ is a stationary point of Problem \eqref{eq:problem pfl}, and $(\Theta^{\infty},W^{\infty})$ is a stationary point of Problem \eqref{eq: bi-variable problem}.
        \item[b)] Under Assumption \ref{assumption: either real analytic or semialgebraic}, the sequence $\{\mathcal{P}^t\}$ converges to $\mathcal{P}^{\infty}$.
    \end{itemize}
\end{theorem}

Note that the establishment of Theorem \ref{theorem: sequences convergence} does not rely on the convexity of the loss functions $f_i$. Consequently, the sequence is ensured to reach a stationary point for Problems \eqref{eq: bi-variable problem} and \eqref{eq:problem pfl}. Furthermore, if we make an additional assumption of convexity for $f_i$, then the sequence will converge to the optimal solution of Problems \eqref{eq: bi-variable problem} and \eqref{eq:problem pfl}.
\subsection{Convergence Rate}
We have demonstrated the convergence of Algorithm \ref{alg: flame on server side}. Next, we aim to examine the rate of this convergence. Specifically, we would like to explore two types of convergence rates under different assumptions.
\begin{theorem} \label{theorem1}
    Suppose Assumption \ref{assmption lipschitz} holds. Let each client $i$ choose $\frac{\lambda^2\alpha_i^2(1+\rho)}{\rho^2}-\frac{\lambda\alpha_i+\rho}{2}<0$, $\rho\geq\lambda\alpha_i$, and $\frac{1-\rho^2}{\rho^2}\lambda^2\alpha_i^2-\frac{L\alpha_i+\rho}{2}>0$, then for all $t\geq 0$, it holds that
\begin{align}
    \frac{1}{T}\sum_{t=0}^{T-1}||\nabla \tilde{\ml}(\mathcal{P}^t)||^2\leq \frac{D_2}{D_1 T}\Bigl(\tilde{\ml}(\mathcal{P}^{0})-f^*\Bigr)+D_2\varepsilon^{1},
\end{align}

\end{theorem}

According to Theorem \ref{theorem1}, we obtain the summation of $||\nabla \ml(\mathcal{P}^t)||^2$ vanishes with a convergence rate of $\mathcal{O}(\frac{1}{T})$, which is considered sublinear. It is worth noting that this convergence rate is established solely under Assumption \ref{assmption lipschitz}, which pertains to gradient Lipschitz continuity, a condition that is not hard to satisfy. In contrast, the results in \cite{t2020personalized,li2021ditto,Lin2022Personalized} are obtained under more assumptions, including gradient Lipschitz continuity, bounded variance, and bounded diversity. 
If we further suppose that Assumption \ref{assumption: either real analytic or semialgebraic} holds, then we can achieve an improved convergence rate as follows.

\begin{theorem}\label{theorem: convergence rate based on kl property}
    Let $\{\mathcal{P}^t\}$ be the sequence generated by Algorithm \ref{alg: flame on server side}, and $\mathcal{P}^{\infty}$ be its limit, let $\psi(x)=\frac{\sqrt{c}}{1-\tau}x^{1-\tau}$ be a desingularizing function (see Definition \ref{definition: Desingularizing Function}), where $c>0$ and $\tau\in[0,1)$, then under Assumptions \ref{assumption: either real analytic or semialgebraic} and \ref{assmption lipschitz}, let each client $i$ choose $\frac{\lambda^2\alpha_i^2(1+\rho)}{\rho^2}-\frac{\lambda\alpha_i+\rho}{2}<0$, $\rho\geq\lambda\alpha_i$, and $\frac{1-\rho^2}{\rho^2}\lambda^2\alpha_i^2-\frac{L\alpha_i+\rho}{2}>0$, the following results hold.
\begin{itemize}
    \item[a)] If $\tau=0$, then there exists a $t_1$ such that the sequence $\{\tilde{\ml}(\mathcal{P}^t)\}$, $t\geq t_1$ converges in a finite number of iterations.
    \item[b)] If $\tau\in(0,1/2]$, then there exists a $t_2$ such that for any $t\geq t_2$, it holds that
\begin{align*}
    \tilde{\ml}(\mathcal{P}^{t})\!-\!\tilde{\ml}(\mathcal{P}^{\infty})\!\leq\!(\frac{cD_2}{D_1\!+\!cD_2})^{t-t_2}(\tilde{\ml}(\mathcal{P}^{t_2})\!-\!f^*\!+\!D_1 \varepsilon^{t+1}).
\end{align*}
   \item[c)] If $\tau \in(1/2,1)$, then there exists a $t_3$ such that for any $t
    \geq t_3$, it holds that
\begin{align*}
    \tilde{\ml}(\mathcal{P}^t)-\tilde{\ml}(\mathcal{P}^{\infty})\leq\Bigl(\frac{cD_2}{(2\tau-1)\mu D_1(t-t_3)}\Bigr)^{\frac{1}{2\tau-1}},
\end{align*}
where $\mu>0$ is a constant.
\end{itemize}
\end{theorem}

Theorem \ref{theorem: convergence rate based on kl property} shows that when $\tau=0$, the convergence rate reaches a constant. For $\tau\in (0,1/2]$, the convergence rate is linear, and for $\tau\in (1/2,1)$, the convergence rate is sublinear.
Note that although Theorem \ref{theorem: convergence rate based on kl property} requires Assumptions \ref{assumption: either real analytic or semialgebraic} and \ref{assmption lipschitz} to hold, they are generally satisfied for most loss functions and remains weaker than those in \cite{t2020personalized,li2021ditto,Lin2022Personalized}.

\subsection{Robustness and Fairness}\label{sec: robustness and fairness theory}
Inspired by \cite{Lin2022Personalized}, we explore the robustness and fairness benefits of \flame on a class of linear problems and compare \flame with \pfedme \cite{t2020personalized} and \ditto \cite{li2021ditto}. We focus on a simplified setting as in \cite{Lin2022Personalized}: infinite local update steps, a single communication round, and all clients participating. 
Suppose the truly personalized model on client $i$ is $\btheta_i$, each client possesses $N$ samples, and the covariate on client $i$ is $\{\bx_{i,j}\}_{j=1}^N$ with $\bx_{i, j}\in\mathbb{R}^d$ is fixed. The observations are generated by $y_{i,j}=\bx_{i,j}^{\top}\btheta_i+z_{i,j}$, where $z_{i,j}$ denotes an i.i.d. Gaussian noise with distribution $\mathcal{N}(0,\sigma^2)$. 
Then the loss on client $i$ is $f_i(\btheta_i)=\frac{1}{2N}\sum_{j=1}^N(y_{i,j}-\bx_{i,j}^{\top}\btheta_i)$. For simplicity, we assume $\sum_{j=1}^N \bx_{i,j}\bx_{i,j}^\top = N b \biden_d$, and there are \(m_a\) malicious clients and \(m_b\) benign clients, with \(m_a + m_b = m\).
We examine the robustness of \flame under three types of Byzantine attacks:
\begin{itemize}
    \item \textbf{Same-value attacks} \cite{Lin2022Personalized}: The message sent by a Byzantine client $i$ is set as $\bu_i^{(ma)}=p\boldsymbol{1}_d$, where $p\sim\mathcal{N}(0,\gamma^2)$;
    \item \textbf{Sign-flipping attacks} \cite{Eugene2020Backdoor}: The message sent by a Byzantine client $i$ is set as $\bu_i^{(ma)}=-|p|\bu_i$, where $p\sim\mathcal{N}(0,\gamma^2)$;
    \item \textbf{Gaussian attacks} \cite{xu2020towards}: The message sent by a Byzantine client $i$ is set as $\bu_i^{(ma)}\sim\mathcal{N}(\boldsymbol{0}_d,\gamma^2\biden_d)$.
\end{itemize}
\begin{proposition}\label{proposition: robustness}
     Let the average testing losses on benign clients for the personalized and global models of \flame, \pfedme, and \ditto be $\loss^{\flamepm}$, $\loss^{\pfedmepm}$, $\loss^{\dittopm}$, $\loss^{\flamegm}$, $\loss^{\pfedmegm}$, $\loss^{\dittogm}$, let $q:=\frac{2\lambda\alpha_i}{\lambda\alpha_i+\rho}\frac{b}{b+\lambda}<1$, then under the three types of Byzantine attacks, the following results hold.
     \begin{align*}
         \loss^{\flamegm} \leq \loss^{\pfedmegm} = \loss^{\dittogm},
     \end{align*}
     when $q\geq \frac{m N b\overline{\btheta}^\top_{b}\overline{\btheta}_{b}}{d \sigma^2+m_b Nb}$ under same-value and Gaussian attacks, and $q\geq\frac{\frac{1}{m_b} \sum_{i\in\mathcal{S}_a}\btheta_i^\top \overline{\btheta}_m }{\frac{d}{m^2}\Bigl(\frac{m_b\sigma^2}{bm}+\sum_{i\in\mathcal{S}_a} \frac{\trace(\bV_i)}{d} \Bigr)+\frac{m_a}{m_b}\overline{\btheta}_m^\top\overline{\btheta}_m}$ under sign-flipping attacks, and
     \begin{align*}
        \loss^{\flamepm} \leq \loss^{\pfedmepm} = \loss^{\dittopm},
     \end{align*}
     when $q\geq \frac{mb (m_bN\lambda\overline{\btheta}^\top_{b}\overline{\btheta}_{b}-d\sigma^2)}{d\sigma^2m_b\lambda+m_b^2 Nb\lambda\overline{\btheta}^\top_{b}\overline{\btheta}_{b}}$ under same-value and Gaussian attacks, \!and $q\geq \frac{\frac{b\lambda}{m_b}\sum_{i'\in\mathcal{S}_a}\btheta_i^\top \overline{\btheta}_m-\frac{bd\sigma^2}{mN} }{\frac{d\sigma^2m_b\lambda}{m^2N}+ \frac{b\lambda d}{m^2}\!\sum_{i\in\mathcal{S}_a} \!\!\!\!\frac{\trace(\bV_i)}{d} + \frac{b\lambda m_a}{m_b} \overline{\btheta}_m^\top \overline{\btheta}_m}$ under sign-flipping attacks, where \begin{small}
         $\overline{\btheta}_{b}\! =\! \frac{1}{m_b}\sum_{i'\in\mathcal{S}_b}\btheta_{i'}$, $\overline{\btheta}_m \!=\! \frac{1}{m}\Bigl(\sum_{i\in\mathcal{S}_b}\!\!\btheta_{i} \!-\!\sum_{i'\in\mathcal{S}_a}\!\!\sqrt{\frac{2}{\pi}} \gamma\btheta_{i'} \Bigr)$, and $\bV_i \!=\! \frac{\pi-2}{\pi}\gamma^2  \btheta_i\btheta_i^{\top} + \gamma^2\frac{ \sigma^2}{bm}\biden_d$.
     \end{small}
\end{proposition}

The proof of Proposition \ref{proposition: robustness} can be obtained in Appendix \ref{appendix: robustness}. Proposition \ref{proposition: robustness} implies \flame outperforms both \pfedme and \ditto in terms of robustness under regular conditions. From the constraints on $q$, we can see that when the dataset has many features, and each client has a small number of data points, $m_bN\lambda\overline{\btheta}^\top_{b}\overline{\btheta}_{b}-d\sigma^2$ and $\frac{b\lambda^2}{m_b}\sum_{i'\in\mathcal{S}_a}\btheta_i^\top \overline{\btheta}_m-\frac{bd\sigma^2\lambda}{mN}$ may be smaller than 0. In this case, for any $q\in(0,1)$, the robustness of \flame is better than that of \pfedme and \ditto. Therefore, \flame is a potentially better PFL framework for scenarios with many features and limited dataset sizes in each client. 
Next, we turn to the fairness analysis of \flame.
\begin{proposition}\label{proposition: fairness}
    Let the variance of test losses on different clients for the personalized and global models of \flame, \pfedme, and \ditto be 
    \begin{small}$\var\{f_i(\btheta_i^\flamepm)\}$, $\var\{f_i(\btheta_i^\pfedmepm)\}$, $\var\{f_i(\btheta_i^\dittopm)\}$, $\var\{f_i(\btheta_i^\flamegm)\}$, $\var\{f_i(\btheta_i^\pfedmegm)\}$, $\var\{f_i(\btheta_i^\dittogm)\}$\end{small}, then we have
\begin{small}
\begin{align*}
    \var\{f_i(\btheta_i^\flamepm)\}<\var\{f_i(\btheta_i^\pfedmepm)\}=\var\{f_i(\btheta_i^\dittopm)\},\\
    \var\{f_i(\btheta_i^\flamegm)\}<\var\{f_i(\btheta_i^\pfedmegm)\}=\var\{f_i(\btheta_i^\dittogm)\}.
\end{align*}    
\end{small}
\end{proposition}
The proof of Proposition \ref{proposition: fairness} is given in Appendix \ref{appendix: fairness}.
Proposition \ref{proposition: fairness} demonstrates that \flame consistently results in more uniform test losses, which implies \flame is more fair than \pfedme and \ditto.

\section{Experiments}\label{sec experiment}

\begin{table}[t]
\centering
\renewcommand{\arraystretch}{1.2}
\setlength{\tabcolsep}{5.5pt} 
\caption{An overview of the datasets and models.}
\label{table: datasets and models}
\begin{tabular}{cccccc}
\toprule
\textbf{Datasets}  & \textbf{\# samples} & \textbf{\# classes}   & \textbf{Ref.} & \textbf{Models} & \textbf{\# parameters} \\ \midrule
\mnist & 70,000 &  10  & \cite{lecun1998gradient}     & MLP    & 7,850       \\
\fmnist     & 70,000  &   10    & \cite{xiao2017fashion}    & MLP    & 7,850     \\
\mmnist    & 58,954  & 6 & \cite{mmnist}    & CNN    & 206,678    \\
\cifar    &  60,000    & 10  & \cite{krizhevsky2009learning}    & CNN    & 268,650   \\ 
\femnist    &  382,705    & 10  & \cite{femnist}    & CNN    & 214,590   \\ \bottomrule
\end{tabular}
\end{table}

\begin{table*}[t]

\caption{The top-1 accuracy of various approaches. We conduct five trials and present the mean accuracy along with the standard derivation. Bold values indicate the highest accuracy achieved by different models.}
\label{table: hightest acc}
\resizebox{\textwidth}{!}{
\begin{tabular}{c c c c c c c c c c}

\toprule
\textbf{Categories}                                                               & \textbf{Datasets}        & \textbf{Partitioning}     & \textbf{\flamehm}           & \textbf{\flamepm}           & \textbf{\flamegm}  & \textbf{\pfedmepm} & \textbf{\pfedmegm} & \textbf{\dittopm}  & \textbf{\dittogm}  \\ \midrule
\multirow{4}{*}{\textbf{\begin{tabular}[c]{@{}c@{}}Hybrid skew\end{tabular}}}   & \mnist                    & \# label=2, Dir(0.5)      & \textbf{0.9494$\pm$0.0129} & 0.9456$\pm$0.0176          & 0.9129$\pm$0.0174 & 0.9264$\pm$0.0193 & 0.8762$\pm$0.0170 & 0.9269$\pm$0.0179 & 0.8709$\pm$0.0100 \\
                                                                                  & \fmnist                   & \# label=2, Dir(0.5)      & \textbf{0.9046$\pm$0.0079} & 0.8989$\pm$0.0109          & 0.8315$\pm$0.0159 & 0.8753$\pm$0.0087 & 0.7840$\pm$0.0180 & 0.8738$\pm$0.0078 & 0.7524$\pm$0.0334 \\ 
                                                                                  & \mmnist                   & \# label=2, Dir(0.5)      & \textbf{0.9973$\pm$0.0010} & 0.9966$\pm$0.0008          & 0.9952$\pm$0.0018 & 0.9932$\pm$0.0006 & 0.9888$\pm$0.0030 & 0.9732$\pm$0.0139 & 0.5723$\pm$0.1631 \\ 
                                                                                  & \cifar                  & \# label=2, Dir(0.5)      & \textbf{0.6633$\pm$0.0414} & 0.6425$\pm$0.0455          & 0.4685$\pm$0.0508 & 0.5737$\pm$0.0242 & 0.3629$\pm$0.0176 & 0.5683$\pm$0.0178 & 0.3273$\pm$0.0114 \\ \midrule
\multirow{6}{*}{\textbf{\begin{tabular}[c]{@{}c@{}}Label skew\end{tabular}}}   & \multirow{6}{*}{\mnist}   & Dir(0.5)                  & \textbf{0.9523$\pm$0.0029} & \textbf{0.9523$\pm$0.0029} & 0.9333$\pm$0.0034 & 0.9245$\pm$0.0022 & 0.8875$\pm$0.0063 & 0.9122$\pm$0.0032 & 0.7997$\pm$0.0388 \\ 
                                                                                  &                          & \# label=2                & \textbf{0.9835$\pm$0.0034} & \textbf{0.9835$\pm$0.0034} & 0.7984$\pm$0.0202 & 0.9791$\pm$0.0034 & 0.7423$\pm$0.0248 & 0.9783$\pm$0.0035 & 0.4782$\pm$0.0585 \\ 
                                                                                  &                          & \# label=3                & \textbf{0.9753$\pm$0.0019} & \textbf{0.9753$\pm$0.0019} & 0.8705$\pm$0.0272 & 0.9661$\pm$0.0027 & 0.7936$\pm$0.0536 & 0.9650$\pm$0.0024 & 0.5649$\pm$0.0995 \\ 
                                                                                  &                          & \# label=4                & \textbf{0.9689$\pm$0.0022} & \textbf{0.9689$\pm$0.0022} & 0.8970$\pm$0.0071 & 0.9566$\pm$0.0038 & 0.8169$\pm$0.0188 & 0.9548$\pm$0.0043 & 0.5881$\pm$0.0581 \\ 
                                                                                  &                          & \# label=5                & \textbf{0.9645$\pm$0.0029} & \textbf{0.9645$\pm$0.0029} & 0.9131$\pm$0.0071 & 0.9489$\pm$0.0042 & 0.8497$\pm$0.0289 & 0.9452$\pm$0.0047 & 0.6905$\pm$0.0279 \\ 
                                                                                  &                          & \# label=6                & \textbf{0.9604$\pm$0.0033} & \textbf{0.9604$\pm$0.0033} & 0.9115$\pm$0.0059 & 0.9450$\pm$0.0058 & 0.8504$\pm$0.0187 & 0.9414$\pm$0.0058 & 0.6338$\pm$0.0301 \\ \midrule
\multirow{6}{*}{\textbf{\begin{tabular}[c]{@{}c@{}}Label skew\end{tabular}}}    & \multirow{6}{*}{\fmnist}  & Dir(0.5)                  & \textbf{0.8963$\pm$0.0065} & \textbf{0.8963$\pm$0.0065} & 0.8463$\pm$0.0082 & 0.8672$\pm$0.0079 & 0.7933$\pm$0.0115 & 0.8525$\pm$0.0068 & 0.7168$\pm$0.0472 \\ 
                                                                                  &                          & \# label=2                & \textbf{0.9642$\pm$0.0198} & \textbf{0.9642$\pm$0.0198} & 0.7501$\pm$0.0567 & 0.9578$\pm$0.0228 & 0.6460$\pm$0.0687 & 0.9543$\pm$0.0245 & 0.4196$\pm$0.0470 \\ 
                                                                                  &                          & \# label=3                & \textbf{0.9549$\pm$0.0100} & \textbf{0.9549$\pm$0.0100} & 0.8012$\pm$0.0084 & 0.9455$\pm$0.0099 & 0.6532$\pm$0.0881 & 0.9412$\pm$0.0104 & 0.5389$\pm$0.0856 \\ 
                                                                                  &                          & \# label=4                & \textbf{0.9333$\pm$0.0076} & \textbf{0.9333$\pm$0.0076} & 0.8121$\pm$0.0154 & 0.9189$\pm$0.0100 & 0.7239$\pm$0.0185 & 0.9132$\pm$0.0111 & 0.5741$\pm$0.0488 \\ 
                                                                                  &                          & \# label=5                & \textbf{0.9278$\pm$0.0054} & \textbf{0.9278$\pm$0.0054} & 0.8259$\pm$0.0065 & 0.9106$\pm$0.0083 & 0.7477$\pm$0.0147 & 0.9048$\pm$0.0110 & 0.6309$\pm$0.0419 \\ 
                                                                                  &                          & \# label=6                & \textbf{0.9153$\pm$0.0083} & \textbf{0.9153$\pm$0.0083} & 0.8264$\pm$0.0122 & 0.8933$\pm$0.0123 & 0.7514$\pm$0.0404 & 0.8845$\pm$0.0122 & 0.6291$\pm$0.0564 \\ \midrule
\multirow{6}{*}{\textbf{\begin{tabular}[c]{@{}c@{}}Label skew\end{tabular}}}    & \multirow{6}{*}{\mmnist}  & Dir(0.5)                   & \textbf{0.9974$\pm$0.0006} & 0.9967$\pm$0.0006          & 0.9965$\pm$0.0009 & 0.9917$\pm$0.0015 & 0.9905$\pm$0.0017 & 0.9733$\pm$0.0061 & 0.8303$\pm$0.0571 \\ 
                                                                                  &                          & \# label=2                & \textbf{0.9987$\pm$0.0006} & \textbf{0.9987$\pm$0.0006} & 0.9649$\pm$0.0244 & 0.9979$\pm$0.0009 & 0.9025$\pm$0.0388 & 0.9948$\pm$0.0024 & 0.4312$\pm$0.0537 \\ 
                                                                                  &                          & \# label=3                & \textbf{0.9984$\pm$0.0008} & \textbf{0.9984$\pm$0.0008} & 0.9616$\pm$0.0675 & 0.9965$\pm$0.0015 & 0.9451$\pm$0.0760 & 0.9904$\pm$0.0045 & 0.6334$\pm$0.0404 \\ 
                                                                                  &                          & \# label=4                & \textbf{0.9984$\pm$0.0002} & \textbf{0.9984$\pm$0.0002} & 0.9944$\pm$0.0013 & 0.9958$\pm$0.0008 & 0.9777$\pm$0.0077 & 0.9860$\pm$0.0019 & 0.6778$\pm$0.0997 \\ 
                                                                                  &                          & \# label=5                & \textbf{0.9980$\pm$0.0003} & 0.9978$\pm$0.0004          & 0.9954$\pm$0.0011 & 0.9948$\pm$0.0012 & 0.9855$\pm$0.0028 & 0.9841$\pm$0.0017 & 0.7188$\pm$0.1502 \\ 
                                                                                  &                          & \# label=6                & \textbf{0.9979$\pm$0.0003} & 0.9976$\pm$0.0003          & 0.9962$\pm$0.0004 & 0.9940$\pm$0.0005 & 0.9896$\pm$0.0016 & 0.9795$\pm$0.0037 & 0.7217$\pm$0.1137 \\ \midrule
\multirow{6}{*}{\textbf{\begin{tabular}[c]{@{}c@{}}Label skew\end{tabular}}}    & \multirow{6}{*}{\cifar} & Dir(0.5)                     & \textbf{0.6629$\pm$0.0172}    & 0.6617$\pm$0.0185       & 0.5228$\pm$0.0149 & 0.5592$\pm$0.0235 & 0.3676$\pm$0.0154 & 0.5427$\pm$0.0204 & 0.2949$\pm$0.0115 \\ 
                                                                                  &                          & \# label=2                & \textbf{0.8525$\pm$0.0065} & \textbf{0.8525$\pm$0.0065} & 0.4370$\pm$0.0131 & 0.7950$\pm$0.0098 & 0.3351$\pm$0.0105 & 0.7796$\pm$0.0129 & 0.2151$\pm$0.0202 \\ 
                                                                                  &                          & \# label=3                & \textbf{0.7731$\pm$0.0169} & \textbf{0.7731$\pm$0.0169} & 0.4697$\pm$0.0144 & 0.6819$\pm$0.0187 & 0.3388$\pm$0.0291 & 0.6718$\pm$0.0244 & 0.2628$\pm$0.0289 \\ 
                                                                                  &                          & \# label=4                & \textbf{0.7539$\pm$0.0148} & \textbf{0.7539$\pm$0.0148} & 0.4921$\pm$0.0026 & 0.6475$\pm$0.0210 & 0.3489$\pm$0.0075 & 0.6276$\pm$0.0235 & 0.2570$\pm$0.0045 \\ 
                                                                                  &                          & \# label=5                & \textbf{0.7103$\pm$0.0193} & \textbf{0.7103$\pm$0.0193} & 0.5062$\pm$0.0130 & 0.5945$\pm$0.0357 & 0.3494$\pm$0.0170 & 0.5791$\pm$0.0360 & 0.2805$\pm$0.0066 \\ 
                                                                                  &                          & \# label=6                & \textbf{0.6932$\pm$0.0144} & \textbf{0.6932$\pm$0.0144} & 0.4989$\pm$0.0170 & 0.5769$\pm$0.0207 & 0.3391$\pm$0.0271 & 0.5552$\pm$0.0163 & 0.2815$\pm$0.0252 \\ \midrule
\multirow{4}{*}{\textbf{\begin{tabular}[c]{@{}c@{}}Quality skew\end{tabular}}}  & \mnist                    & \multirow{4}{*}{Gau(0.1)} & \textbf{0.9457$\pm$0.0016} & 0.9448$\pm$0.0012          & 0.9444$\pm$0.0015 & 0.9019$\pm$0.0021 & 0.8999$\pm$0.0018 & 0.9008$\pm$0.0020 & 0.8966$\pm$0.0015 \\ 
                                                                                  & \fmnist                   &                           & \textbf{0.8579$\pm$0.0016} & 0.8546$\pm$0.0013          & 0.8574$\pm$0.0015 & 0.8087$\pm$0.0028 & 0.8099$\pm$0.0030 & 0.7929$\pm$0.0037 & 0.7843$\pm$0.0063 \\ 
                                                                                  & \mmnist                   &                           & \textbf{0.9991$\pm$0.0003} & \textbf{0.9991$\pm$0.0003} & 0.9376$\pm$0.0158 & 0.9984$\pm$0.0002 & 0.7940$\pm$0.0493 & 0.9974$\pm$0.0010 & 0.5282$\pm$0.1446 \\ 
                                                                                  & \cifar                  &                           & \textbf{0.5581$\pm$0.0069} & 0.5448$\pm$0.0093          & 0.5578$\pm$0.0067 & 0.3810$\pm$0.0146 & 0.3805$\pm$0.0143 & 0.3543$\pm$0.0126 & 0.3424$\pm$0.0120 \\ \midrule
\multirow{4}{*}{\textbf{\begin{tabular}[c]{@{}c@{}}Quantity skew\end{tabular}}} & \mnist                    & \multirow{4}{*}{Dir(0.5)} & \textbf{0.9139$\pm$0.0115} & 0.9036$\pm$0.0180          & 0.9119$\pm$0.0119 & 0.8759$\pm$0.0208 & 0.8890$\pm$0.0100 & 0.8941$\pm$0.0069 & 0.8968$\pm$0.0069 \\ 
                                                                                  & \fmnist                   &                           & \textbf{0.8239$\pm$0.0138} & 0.8093$\pm$0.0175          & 0.8221$\pm$0.0145 & 0.7743$\pm$0.0177 & 0.7824$\pm$0.0145 & 0.7825$\pm$0.0058 & 0.7767$\pm$0.0054 \\ 
                                                                                  & \mmnist                   &                           & \textbf{0.9947$\pm$0.0022} & 0.9884$\pm$0.0084          & 0.9944$\pm$0.0024 & 0.9837$\pm$0.0072 & 0.9862$\pm$0.0088 & 0.9684$\pm$0.0072 & 0.8156$\pm$0.0496 \\ 
                                                                                  & \cifar                   &                           & \textbf{0.4511$\pm$0.0359} & 0.4114$\pm$0.0392          & 0.4502$\pm$0.0365 & 0.3260$\pm$0.0256 & 0.3474$\pm$0.0310 & 0.3478$\pm$0.0140 & 0.3511$\pm$0.0217 \\ \midrule
\textbf{\begin{tabular}[c]{@{}c@{}}Feature skew\end{tabular}}                   & \femnist                  & \# writers=338            & \textbf{0.9993$\pm$0.0001} & 0.9988$\pm$0.0003          & 0.9992$\pm$0.0001 & 0.9850$\pm$0.0029 & 0.9920$\pm$0.0021 & 0.9847$\pm$0.0016 & 0.9805$\pm$0.0044 \\ \bottomrule
\end{tabular}}
\end{table*}

%

In the experiments, we aim to evaluate the performance of \flame in comparison to other state-of-the-art methods, specifically comparing their accuracy, convergence, robustness, and fairness.
Moreover, we investigate how different hyperparameters influence the convergence of \texttt{FLAME}.
\subsection{Settings}

\myparagraph{Datasets}
We employ \mnist \cite{lecun1998gradient}, Fashion MNIST (\fmnist) \cite{xiao2017fashion}, Medical MNIST (\mmnist) \cite{mmnist}, \cifar \cite{krizhevsky2009learning}, and \femnist \cite{femnist}, which are the most widely employed datasets in FL research community. We randomly select 20\% of each dataset to create a testing set, leaving the remaining 80\% as the training set.

\myparagraph{Models}
We evaluate the performance of multilayer perceptron (MLP) models with two hidden layers on both \mnist and \fmnist. For \mmnist, \cifar, and \femnist, we employ convolutional neural networks (CNN) consisting of two convolutional layers, each followed by a max-pooling layer. 
Subsequently, a flattening layer is applied to convert the extracted feature maps into a one-dimensional vector, which is then processed through a fully connected layer with the ReLU activation function. Before the output layer, a dropout layer is incorporated to mitigate overfitting.
Table \ref{table: datasets and models} presents a comprehensive overview of the datasets and models.

\myparagraph{Partitions}
To accommodate data heterogeneity, we adopt the same data partitioning strategies as described in \cite{li2022federated} and \cite{ye2023heterogeneous}, which respectively are a good experimental study on FL with non-i.i.d. data silos and a good review of heterogeneous FL.
These partitions encompass label skew, feature skew, quality skew, and quantity skew. Regarding label skew, quality skew, and quantity skew, we conduct experiments on \mnist, \fmnist, \mnist, and \cifar datasets. Additionally, we address feature skew specifically on the \femnist dataset due to the inclusion of writer information in each image.

For label skew, we consider two scenarios. 
The first involves \textit{quantity-based label imbalance}, where we organize the training data by their labels and distribute them into shards, with each client being assigned 2-6 shards randomly.
The second scenario, referred to as \textit{distribution-based label imbalance}, involves allocating a portion of samples from each label to every client based on the Dirichlet distribution. Specifically, we sample from $p_k\sim Dir(\beta)$ to assign a $p_{ki}$ proportion of samples of class $k$ to client $i$. Here, $Dir(\cdot)$ denotes the Dirichlet distribution with a concentration parameter $\beta>0$. 
For feature skew, we propose partitioning the \femnist dataset into different clients based on individual writers. Given that character features such as stroke width and slant often vary among writers, a discernible feature skew naturally arises across these clients.
For quality skew, we consider utilizing \textit{noise-based quality imbalance}, which involves introducing varying levels of Gaussian noise to each client's local dataset to achieve different quality distributions. 
Specifically, given a noise level $\sigma$, we add $\boldsymbol{n}_i\sim Gau(\sigma\cdot i/m)$ for client $i$, where $Gau(\sigma\cdot i/m)$ is a Gaussian distribution with mean 0 and variance $\sigma\cdot i/m$. 
For quantity skew, we employ the Dirichlet distribution to distribute varying amounts of data samples among each party. 
Specifically, we sample from $p\sim Dir(\beta)$ and assign a $p_i$ proportion of samples to client $i$.
Moreover, to illustrate that different clients may possess various types of heterogeneous data, we designate half of the clients with quantity-based label imbalance data and the other half with quantity skew data, which we refer to as \textit{hybrid skew}.

\myparagraph{Baselines} 
We evaluate the performance of \flame against two state-of-the-art PFL methods, namely, \texttt{pFedMe} \cite{t2020personalized} and \ditto \cite{li2021ditto}. 
We ensure that $\{\btheta_i\}_{i=1}^m$, $\{\bw_i\}_{i=1}^m$, and $\bw$ are updated an equal number of times across all methods.
We use \texttt{FLAME-GM} and \texttt{FLAME-PM} as abbreviations for the global and personalized models of \texttt{FLAME}. Similarly, \pfedmepm, \pfedmegm, \dittopm, and \dittogm follow the same convention. 
Moreover, we use \flamehm to represent the hybrid model, indicating the model with higher accuracy among personalized and global models.
Accuracy is calculated as the average accuracy across all clients, and in the presence of attacks, it is the average accuracy of the benign clients.

\myparagraph{Attacks}
In addition to the three Byzantine attacks discussed in Section \ref{sec: robustness and fairness theory}, we consider a stronger data poisoning attack in the following experiments.

\textit{Label poisoning attacks} \cite{Arjun2019Analyzing, Biggio2011Support}: Corrupted devices do not have access to the training APIs, and the training samples are poisoned with flipped labels (for binary classification) or uniformly random noisy labels.

The corruption levels, i.e., the fractions of malicious clients, are set as $\{0,0.2,0.5,0.8\}$. For the three Byzantine attacks in Section \ref{sec: robustness and fairness theory}, the noise variance is set to 0.1.
To better measure robustness, we introduce \texttt{multi-Krum} as a comparison method, which is a global training approach augmented with a robust aggregation technique.

\myparagraph{Implementations}
Our algorithms were executed on a computational platform comprising two Intel Xeon Gold 5320 CPUs with 52 cores, 512 GB of RAM, four NVIDIA A800 with 320 GB VRAM, and operating on the Ubuntu 22.04 environment. The software implementation was realized in Python 3.8 and Pytorch 2.1 and open-sourced (\textcolor{blue}{\href{https://github.com/zsk66/FLAME}{https://github.com/zsk66/FLAME}}).

\begin{figure*}[p]
    \centering
    \includegraphics[width=1\textwidth]{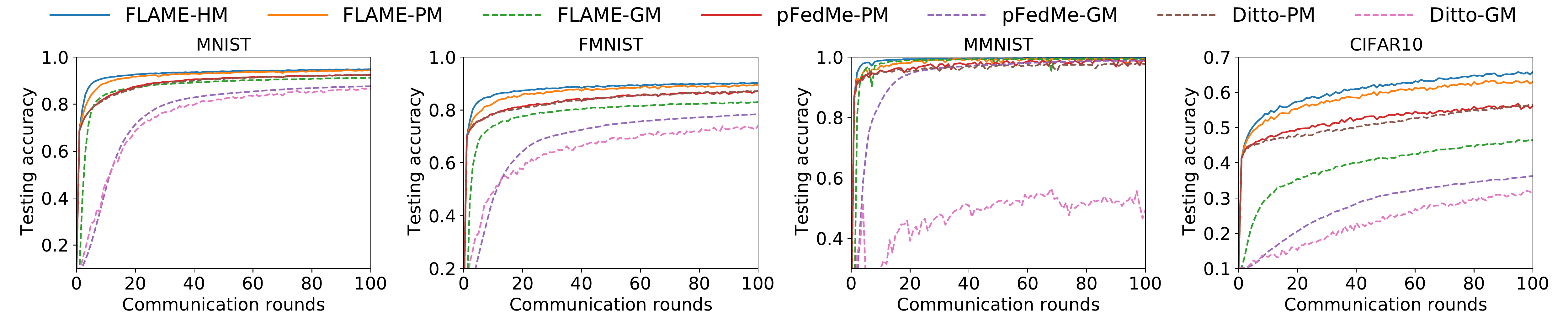} 
    \vspace{-1em}
    \caption{A comparison of the test accuracy across different methods. The dashed lines represent the global models, while the solid lines represent the personalized models. The personalized and global models of \flame outperform other methods on four datasets.}
    \label{fig:acc communication rounds six methods}
    \vspace{1em}
    \includegraphics[width=1\textwidth]{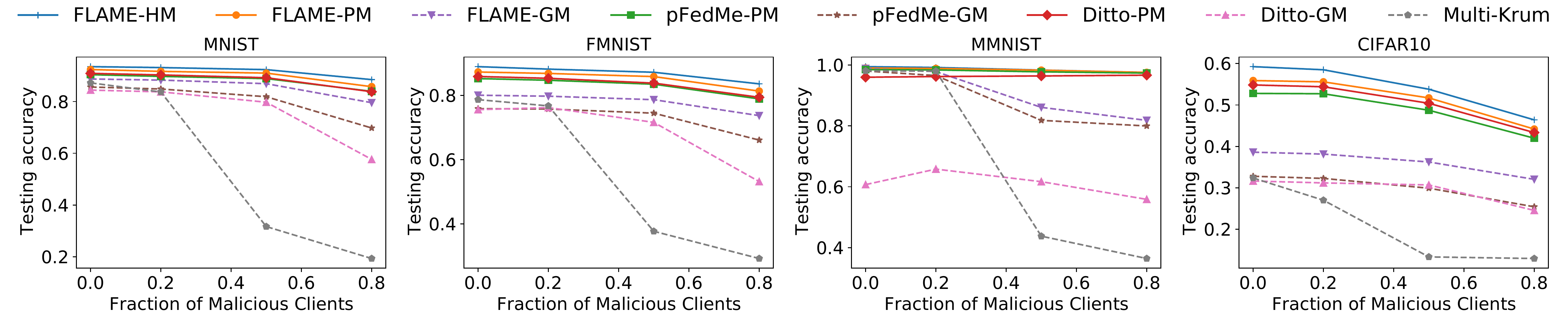} 
    \vspace{-1em}
    \caption{Robustness comparison of different methods under label poisoning attacks with hybrid skew.}
    \label{fig: acc_malicious_hybrid-skew}
    \vspace{1em}
    \includegraphics[width=1\linewidth]{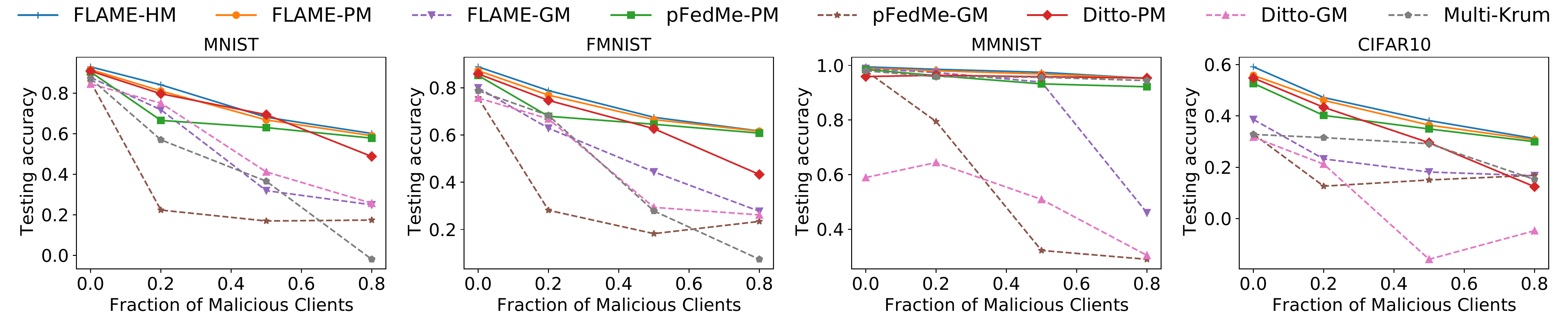}
    \vspace{-1em}
    \caption{Robustness comparison of different methods under same-value attacks with hybrid skew.}
    \label{fig:acc_malicious_attack_2_hybrid-skew_2}
    \vspace{1em}
    \includegraphics[width=1\linewidth]{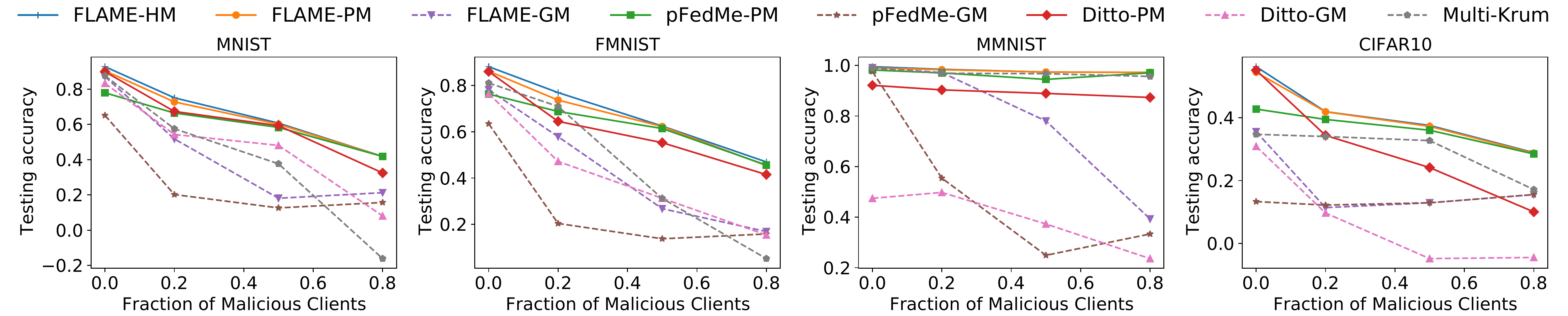}
    \vspace{-1em}
    \caption{Robustness comparison of different methods under sign-flipping attacks with hybrid skew.}
    \label{fig:acc_malicious_attack_3_hybrid-skew_2}
    \vspace{1em}
    \includegraphics[width=1\linewidth]{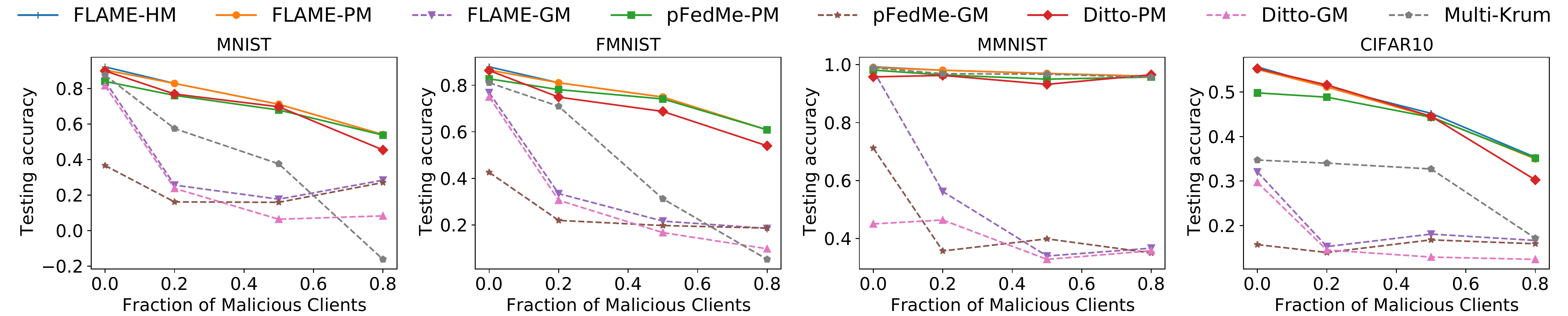}
    \vspace{-1em}
    \caption{Robustness comparison of different methods under Gaussian attacks with hybrid skew.}
    \label{fig:acc_malicious_attack_4_hybrid-skew_2}
\end{figure*}

\begin{figure*}[t]
  \centering
    \includegraphics[width=1\textwidth]{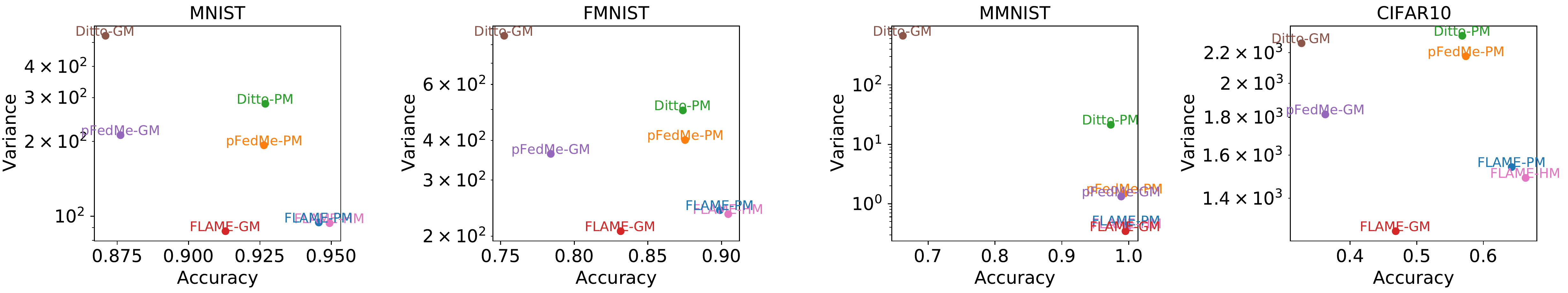} 
    \vspace{-2em}
    \caption{Accuracy-fairness trade-off of competing methods (The point closer to the bottom right corner is better).}
    \label{fig: Accuracy-fairness trade-off}
    \vspace{-0.5em}
\end{figure*}
\subsection{Comparison of Multiple Methods}
\subsubsection{Overall accuracy comparison}\label{sec: Overall accuracy comparison}
Table \ref{table: hightest acc} presents the top-1 accuracy of various methods under different data partitioning strategies. Each entry in the table represents the average accuracy obtained over five different random seeds, with the standard deviation calculated accordingly.
We configured each method with a local learning rate $\eta=0.01$, a batch size of 100, local iterations $H=1$, and the number of clients $m$ set to 10, and hyperparameter $\lambda=1$. Additionally, for all subsequent experiments, we set $\alpha_i=1/m$.
For \flame, we set $\rho=0.1$. For \pfedme and \ditto, we select the best-performing learning rate for their global models from the candidate set $\{0.01, 0.05, 0.1, 0.2, 0.5\}$.

Firstly, we can clearly see that \flamehm achieved the best accuracy across all datasets and data partitioning schemes. 
This is because the hybrid model selects the best performance from both the personalized and global models.
Under the hybrid skew, compared to using only the personalized or global model, we found that the hybrid model can improve accuracy by an average of 4.2\%.
Secondly, we examine the accuracy of personalized and global models of different methods and various data partitioning schemes.
It is evident that \flame's personalized model achieves the highest accuracy in the case of quantity-based label imbalance. 
This can be attributed to the design intent of PFL, which aims to train personalized models tailored to the individual needs of users. 
Therefore, personalized models exhibit enhanced capability in expressing the imbalanced label distribution inherent in the data. 
Thirdly, in the case of quantity-based label imbalance, we observe that as the number of labels owned by each client increases, the accuracy of personalized models in both PFL frameworks decreases, while the accuracy of global models increases. 
This is because, with an increase in the number of labels, data heterogeneity diminishes, thereby deteriorating the generalization capability of personalized models, while enhancing the generalization capability of global models.
Fourthly, in terms of quality skew, quantity skew, and feature skew, global models usually achieve better performance than personalized models.
Therefore, we should not only focus on improving the performance of personalized models but also consider enhancing the performance of global models to handle different types of heterogeneous data.
Finally, in all data partitioning cases, both the personalized and global models of \flame exhibit higher accuracy than \pfedme and \ditto.
On average, the personalized models are 3.9\% more accurate, and the global models are 14.2\% more accurate.
These improvements are attributable to the use of ADMM in \flame, which is a primal-dual method resulting in superior solving precision compared to \pfedme and \ditto.

\subsubsection{Convergence comparison}
Fig. \ref{fig:acc communication rounds six methods} illustrates how the testing accuracy varies with the number of communication rounds for different methods. 
The data partitioning approach employed here is characterized by hybrid skew, with each client possessing two labels. All the algorithms' hyperparameter settings are identical to those presented in Section \ref{sec: Overall accuracy comparison}. Due to space limitations, we only show the comparisons on test accuracy on hybrid skew.  Results on the other partition schemes are left in Appendix \ref{appendix: additional experiments}.
Firstly, we can clearly see that \flamehm consistently achieves higher accuracy across four datasets. 
This confirms that when different clients have various types of heterogeneous data, it is not certain whether personalized or global models will perform better. 
Choosing the better model between the personalized and global models can significantly improve accuracy.
Secondly, it can be observed that both \flame's personalized and global models achieve higher accuracy compared to \pfedme and \ditto, and we can see that both \flame's personalized and global models converge faster than \pfedme and \ditto, demonstrating \flame's superior performance in terms of convergence.

\subsubsection{Robustness comparison}
We compare the robustness of different methods, measured by the average test accuracy on benign devices, under four different attacks. 
We set the number of clients to 50. For \flame, \pfedme, and \ditto, we set $\lambda$ to 1. For \flame, we set $\rho$ to 0.02, while other parameters are set the same as in Section \ref{sec: Overall accuracy comparison}. Due to space limitations, we show in Fig. \ref{fig: acc_malicious_hybrid-skew} - Fig. \ref{fig:acc_malicious_attack_4_hybrid-skew_2} how the test accuracy varies with the number of malicious clients for different methods under hybrid skew and four attacks. The results for the other data partitioning schemes under four attacks can be found in Appendix \ref{appendix: additional experiments}.
Firstly, we find that under four attacks, the accuracy of all methods decreases as the number of malicious clients increases. However, the decline for \flame is significantly smaller than for other methods. Note that under label poisoning attacks on \mnist, \fmnist, and \mmnist, the testing accuracy of \flamehm rarely decays as the fraction of malicious clients increases, while we observe significant drops in the testing accuracy for other algorithms once malicious clients exist. 
Secondly, \flamehm consistently has the highest accuracy compared to other methods under different attacks, \flamepm and \flamegm have higher accuracy than other methods for personalized and global models, respectively.

\begin{figure*}[t]
    \centering
    \includegraphics[width=1\linewidth]{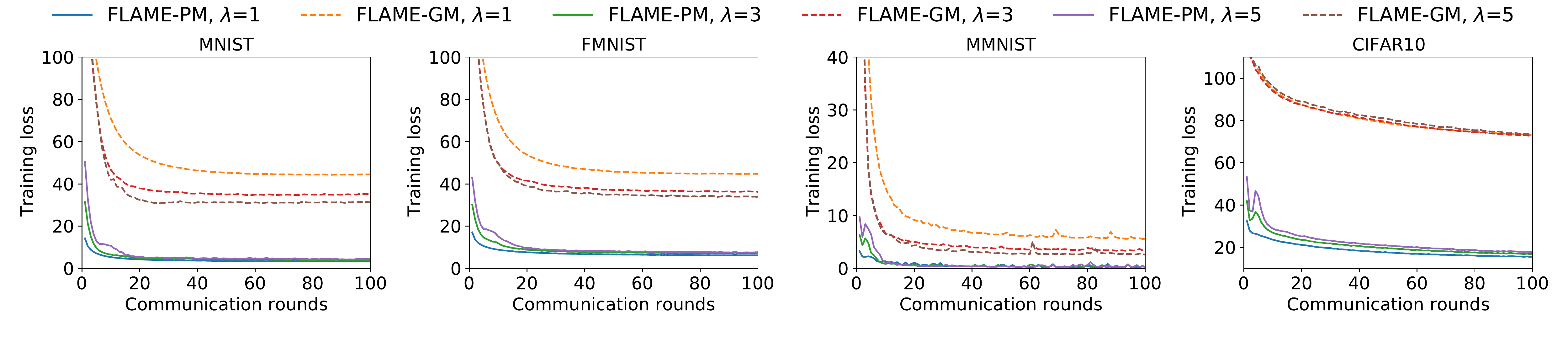}
    \vspace{-1.5em}
    \caption{Effect of the regularization parameter $\lambda$ on the convergence of \flame. As $\lambda$ increases, the performance of the personalized models becomes closer to that of the global models.}
    \label{fig:loss_communication_lambda}
\end{figure*}
\begin{figure*}[t]
    \centering
    \includegraphics[width=1\linewidth]{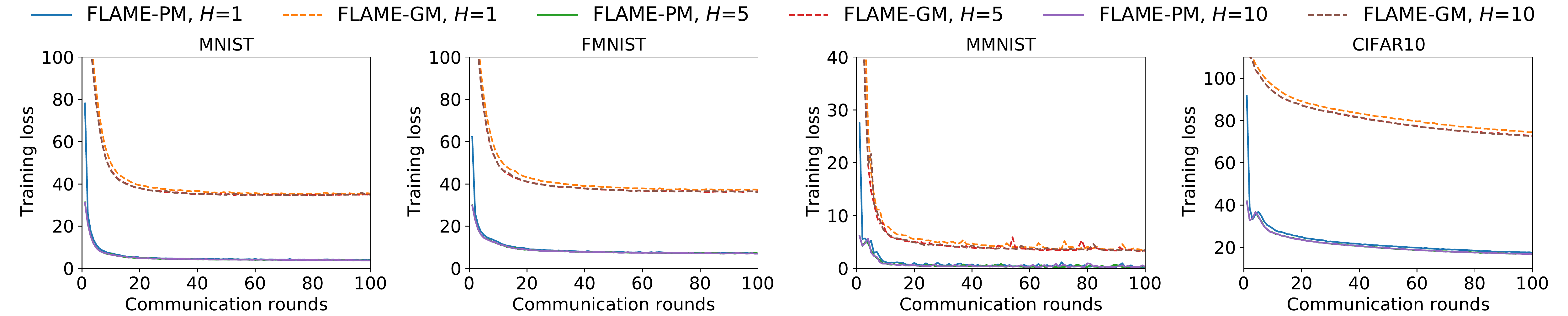}
    \vspace{-1.5em}
    \caption{Effect of the local iterations $H$ on the convergence of \flame. As $H$ increases, the performance of personalized and global models improves.}
    \label{fig:loss_communication_epoch}
\end{figure*}
\begin{figure*}[t]
    \centering
    \includegraphics[width=1\linewidth]{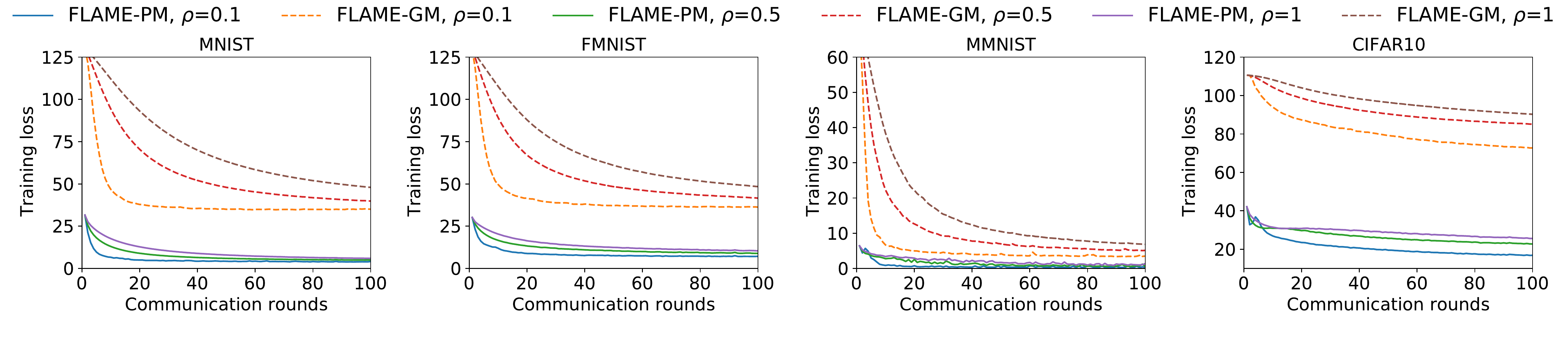}
    \vspace{-1.5em}
    \caption{Effect of the penalty parameter $\rho$ on the convergence of \flame. As $\rho$ decreases, the performance of the personalized and global models improves.}
    \label{fig:loss_communication_rho}
\end{figure*}
\begin{figure*}[t]
    \centering
    \includegraphics[width=1\linewidth]{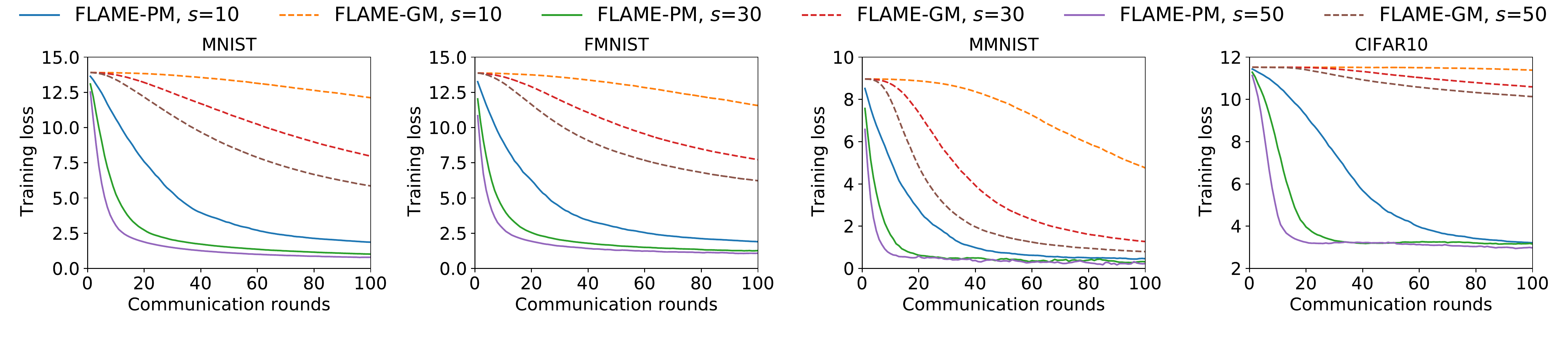}
    \vspace{-1.5em}
    \caption{Effect of the number of clients on the convergence of \flame. As $|\mathcal{S}|^t$ increases, the performance of the personalized and global models improves.}
    \label{fig:loss_communication_frac}
\end{figure*}

\subsubsection{Fairness comparison}
To illustrate the trade-off between accuracy and fairness, we plot the variances of test losses against the corresponding test accuracies for \flame, \pfedme, and \ditto. 
To isolate fairness, the numerical experiments are conducted without adversarial attacks in this section. Due to space limitations, we present the results for hybrid skew here, with the results for other data partitioning schemes provided in Appendix \ref{appendix: additional experiments}.
The results regarding fairness show that \flamehm provides the most accurate and fair solutions. 
We can see that \flamehm improves accuracy by an average of 4.3\% compared to \flamepm and \flamegm, and reduces variance by an average of 4.4\% compared to \flamepm.
Comparing different methods for personalized and global models, we can see that \flamepm and \flamegm consistently achieve more fair and accurate solutions.
On average, the personalized and global models improved by 2.9\% and 4.9\% in accuracy, while the variance decreased by an average of 51.5\% and 47.7\% respectively.

\subsection{Comparison of Multiple Parameters}
\subsubsection{Effect of regularization $\lambda$}
Fig. \ref{fig:loss_communication_lambda} illustrates the impact of different regularization parameters $\lambda$ on convergence. We configured the local learning rate $\eta$ to be 0.01, the number of clients $m$ set to 10, the penalty parameter $\rho$ to be 0.1, and the number of local iterations $H$ to be 3.
We assigned $\lambda$ as 1, 3, and 5, respectively, and conducted experiments with five different random seeds. We averaged the results to observe the variations in the losses of personalized and global models in \flame across communication rounds.
We observe that as the value of $\lambda$ increases, the personalized model converges more slowly, while the global model converges faster. This is attributed to the fact that an increase in the regularization parameter $\lambda$, results in a stronger penalty for minimizing the disparity between the personalized and global models. Consequently, this causes the personalized model to approach the global model more closely.

\subsubsection{Effect of local iterations $H$}
Fig. \ref{fig:loss_communication_epoch} illustrates the impact of different local iterations $H$ on the convergence of \flame. 
We configured the local learning rate $\eta$ to be 0.01, the penalty parameter $\rho$ to be 0.1, and the regularization parameter $\lambda$ to be 5. We set $H$ to 1, 5, and 10, respectively, and conducted experiments with five different random seeds, averaging the results to observe the changes in the loss of \flame's personalized and global models as communication rounds vary. 
It is evident that increasing local iterations accelerates the convergence of \flame. This is due to the improved precision in solving the personalized model $\btheta_i$ with the growth of local iterations, a characteristic inherent to the first-order gradient method. The enhanced accuracy of the personalized model, in turn, improves the precision of solving the dual variables $\bpi_i$ and the global model $\bw$, resulting in accelerated convergence of both the global and personalized models. However, as the number of local iterations increases beyond a certain threshold, the performance of both the personalized and global models stabilizes. For example, when $H=5$ and $H=10$, the training loss curves of the personalized model nearly overlap across the four datasets, and likewise, the training loss curves of the global model also nearly overlap. This is because the accuracy of solving for each $\btheta_i$ may have reached a local optimum, and the accuracy of solving will not continue to increase.


\subsubsection{Effect of penalty $\rho$}
Fig. \ref{fig:loss_communication_rho} illustrates the impact of different values of the penalty parameter $\rho$ on the convergence of \texttt{FLAME}. We set the local learning rate $\eta$ to 0.01, regularization parameter $\lambda$ to 3, and local iterations $H$ to 3. 
We varied the values of $\rho$ as 0.1, 0.5, and 1, and conducted experiments with five different random seeds, averaging the results to observe changes in the loss of \flame's personalized and global models with respect to communication rounds.
It is evident that reducing $\rho$ accelerates the convergence of \flame. This is because with increasing $\rho$, the penalty between the local and global models becomes more stringent, which may push the local model closer to the global model and further away from the better-performing personalized model. Consequently, this slows down the convergence of both the personalized models and the global model.

\subsubsection{Effect of the number of selected clients $s$}
Fig. \ref{fig:loss_communication_frac} demonstrates how varying the number of selected clients $\mathcal{S}^t$ influences the convergence of \texttt{FLAME}. We fix the total number of clients at $m=100$. Additionally, we set the local learning rate $\eta$ to 0.01, regularization parameter $\lambda$ to 3, local iterations $H$ to 3, and penalty parameter $\rho$ to 0.03. Experimentation involved altering $s$ to 10, 30, and 50, respectively, across five different random seeds and averaging the results to observe fluctuations in the loss of both personalized and global models of \texttt{FLAME} with communication rounds.
It is evident that as the number of selected clients increases, \flame converges faster. This phenomenon can be attributed to the larger number of clients selected, resulting in a higher number of variables solved in each iteration of the ADMM process. Consequently, this enhances the precision of the solution at each iteration, thereby expediting algorithm convergence.



\subsection{Summary of Lessons Learned}
We have compared the performance of \flame with state-of-the-art methods and validated the impact of algorithmic hyperparameters on the performance of \flame. 
Our experimental findings have led to the following definitive conclusions:
\begin{itemize}
    \item Under label skew, personalized models typically achieve better accuracy. In contrast, global models generally attain higher accuracy under feature skew, quantity skew, and quality skew. When dealing with hybrid skew, our model selection strategy can effectively improve accuracy.
    \item \flame demonstrates superior accuracy and convergence compared to state-of-the-art methods. This is particularly evident in its personalized and global models, which achieve higher accuracy and faster convergence.
    \item The robustness and fairness of \flame outperform those of state-of-the-art methods. This is primarily evidenced by \flame exhibiting smaller variances in testing losses and achieving lower testing losses under Byzantine attacks.
    \item Due to the fact that \flame does not require the adjustment of learning rate when training the global model, it significantly reduces the burden of hyperparameter tuning in comparison to \pfedme and \ditto.
    \item Choosing appropriate parameters can significantly enhance the performance of \flame. The regularization parameter $\lambda$ can adjust the gap between personalized and global models, with a larger $\lambda$ narrowing the gap between them. Increasing local iterations can improve model accuracy, thus expediting convergence. Reducing the penalty parameter can loosen the constraints, thereby accelerating convergence. When communication capacity permits, it is advisable to engage with a greater number of clients whenever possible, as this can expedite convergence.
\end{itemize}

\section{Conclusion and Future Work}\label{sec conclusion}
In this paper, we proposed a PFL framework, \flame, to address the impact of various types of heterogeneous data across different clients. 
We formulated the optimization problem for PFL based on the Moreau envelope and solved it using the ADMM. We proposed a model selection strategy that chooses the model with higher accuracy from either the personalized or global models. 
We established global convergence for \flame and proposed two kinds of convergence rates under mild conditions.
We theoretically demonstrated that \flame has improved robustness and fairness compared to \pfedme and \ditto on a class of linear problems.
Our experimental results demonstrated the superior performance of \flame in terms of accuracy, convergence, robustness, and fairness on various kinds of heterogeneous data compared to state-of-the-art methods. Furthermore, \flame, when applied to solving the global model, eliminates the need for learning rate adjustments, thereby alleviating the burden of hyperparameter tuning in contrast to \pfedme and \ditto.

In the future, we will focus on addressing privacy issues in FL, with an emphasis on using techniques such as encryption and differential privacy to mitigate privacy leakages. 
\ifCLASSOPTIONcaptionsoff
  \newpage
\fi

\bibliographystyle{abbrv}
\bibliography{sample.bib}

\begin{thebibliography}{100}

\bibitem{mmnist}
Medical {MNIST}.
\newblock \url{https://www.kaggle.com/datasets/andrewmvd/medical-mnist}, 2020.

\bibitem{antunes2022federated}
R.~S. Antunes, C.~Andr{\'e}~da Costa, A.~K{\"u}derle, I.~A. Yari, and B.~Eskofier.
\newblock Federated learning for healthcare: Systematic review and architecture proposal.
\newblock {\em ACM Trans. Intell. Syst. Technol.}, 13(4):1--23, 2022.

\bibitem{Attouch2009convergence}
H.~Attouch and J.~Bolte.
\newblock On the convergence of the proximal algorithm for nonsmooth functions involving analytic features.
\newblock {\em Math. Program.}, 116(1-2):5--16, 2009.

\bibitem{Attouch2013Convergence}
H.~Attouch, J.~Bolte, and B.~F. Svaiter.
\newblock Convergence of descent methods for semi-algebraic and tame problems: proximal algorithms, forward-backward splitting, and regularized gauss-seidel methods.
\newblock {\em Math. Program.}, 137(1-2):91--129, 2013.

\bibitem{Eugene2020Backdoor}
E.~Bagdasaryan, A.~Veit, Y.~Hua, D.~Estrin, and V.~Shmatikov.
\newblock How to backdoor federated learning.
\newblock In {\em AISTATS}, volume 108, pages 2938--2948, 2020.

\bibitem{Arjun2019Analyzing}
A.~N. Bhagoji, S.~Chakraborty, P.~Mittal, and S.~B. Calo.
\newblock Analyzing federated learning through an adversarial lens.
\newblock In {\em ICML}, volume~97, pages 634--643, 2019.

\bibitem{Biggio2011Support}
B.~Biggio, B.~Nelson, and P.~Laskov.
\newblock Support vector machines under adversarial label noise.
\newblock In {\em ACML}, volume~20, pages 97--112, 2011.

\bibitem{Biggio2012Poisoning}
B.~Biggio, B.~Nelson, and P.~Laskov.
\newblock Poisoning attacks against support vector machines.
\newblock In {\em ICML}, 2012.

\bibitem{bistritz2020distributed}
I.~Bistritz, A.~Mann, and N.~Bambos.
\newblock Distributed distillation for on-device learning.
\newblock In {\em NeurIPS}, pages 22593--22604, 2020.

\bibitem{Blanchard2017Adversaries}
P.~Blanchard, E.~M.~E. Mhamdi, R.~Guerraoui, and J.~Stainer.
\newblock Machine learning with adversaries: Byzantine tolerant gradient descent.
\newblock In {\em NeurIPS}, pages 119--129, 2017.

\bibitem{bochnak2013real}
J.~Bochnak, M.~Coste, and M.-F. Roy.
\newblock {\em Real algebraic geometry}.
\newblock 1998.

\bibitem{bolte2007lojasiewicz}
J.~Bolte, A.~Daniilidis, and A.~Lewis.
\newblock The {\l}ojasiewicz inequality for nonsmooth subanalytic functions with applications to subgradient dynamical systems.
\newblock {\em SIAM J. Optim.}, 17(4):1205--1223, 2007.

\bibitem{boyd2011distributed}
S.~Boyd, N.~Parikh, E.~Chu, B.~Peleato, J.~Eckstein, et~al.
\newblock Distributed optimization and statistical learning via the alternating direction method of multipliers.
\newblock {\em Found. Trends Mach. Learn.}, 3(1):1--122, 2011.

\bibitem{bubeck2015convex}
S.~Bubeck et~al.
\newblock Convex optimization: Algorithms and complexity.
\newblock {\em Found. Trends Mach. Learn.}, 8(3-4):231--357, 2015.

\bibitem{femnist}
S.~Caldas, S.~M.~K. Duddu, P.~Wu, T.~Li, J.~Kone{\v{c}}n{\`y}, H.~B. McMahan, V.~Smith, and A.~Talwalkar.
\newblock Leaf: A benchmark for federated settings.
\newblock {\em arXiv preprint arXiv:1812.01097}, 2018.

\bibitem{Chartrand2013nonconvex}
R.~Chartrand and B.~Wohlberg.
\newblock A nonconvex {ADMM} algorithm for group sparsity with sparse groups.
\newblock In {\em ICASSP}, pages 6009--6013, 2013.

\bibitem{chen2021theorem}
S.~Chen, Q.~Zheng, Q.~Long, and W.~J. Su.
\newblock A theorem of the alternative for personalized federated learning.
\newblock {\em arXiv preprint arXiv:2103.01901}, 2021.

\bibitem{chen2017targeted}
X.~Chen, C.~Liu, B.~Li, K.~Lu, and D.~Song.
\newblock Targeted backdoor attacks on deep learning systems using data poisoning.
\newblock {\em arXiv preprint arXiv:1712.05526}, 2017.

\bibitem{Chen2017Distributed}
Y.~Chen, L.~Su, and J.~Xu.
\newblock Distributed statistical machine learning in adversarial settings: Byzantine gradient descent.
\newblock {\em Proc. {ACM} Meas. Anal. Comput. Syst.}, 1(2):44:1--44:25, 2017.

\bibitem{cho2022heterogeneous}
Y.~J. Cho, A.~Manoel, G.~Joshi, R.~Sim, and D.~Dimitriadis.
\newblock Heterogeneous ensemble knowledge transfer for training large models in federated learning.
\newblock In {\em IJCAI}, pages 2881--2887, 2022.

\bibitem{duan2020self}
M.~Duan, D.~Liu, X.~Chen, R.~Liu, Y.~Tan, and L.~Liang.
\newblock Self-balancing federated learning with global imbalanced data in mobile systems.
\newblock {\em IEEE Trans. Parallel Distributed Syst.}, 32(1):59--71, 2020.

\bibitem{Dwork2012Fairness}
C.~Dwork, M.~Hardt, T.~Pitassi, O.~Reingold, and R.~S. Zemel.
\newblock Fairness through awareness.
\newblock In {\em ITCS}, pages 214--226, 2012.

\bibitem{fallah2020personalized}
A.~Fallah, A.~Mokhtari, and A.~Ozdaglar.
\newblock Personalized federated learning with theoretical guarantees: A model-agnostic meta-learning approach.
\newblock In {\em NeurIPS}, pages 3557--3568, 2020.

\bibitem{Fallah2020Convergence}
A.~Fallah, A.~Mokhtari, and A.~E. Ozdaglar.
\newblock On the convergence theory of gradient-based model-agnostic meta-learning algorithms.
\newblock In {\em AISTATS}, volume 108, pages 1082--1092, 2020.

\bibitem{Fang2020Local}
M.~Fang, X.~Cao, J.~Jia, and N.~Z. Gong.
\newblock Local model poisoning attacks to byzantine-robust federated learning.
\newblock In {\em {USENIX} Security Symposium}, pages 1605--1622, 2020.

\bibitem{finn2017model}
C.~Finn, P.~Abbeel, and S.~Levine.
\newblock Model-agnostic meta-learning for fast adaptation of deep networks.
\newblock In {\em ICML}, pages 1126--1135, 2017.

\bibitem{gong2022fedadmm}
Y.~Gong, Y.~Li, and N.~M. Freris.
\newblock Fedadmm: A robust federated deep learning framework with adaptivity to system heterogeneity.
\newblock In {\em ICDE}, pages 2575--2587, 2022.

\bibitem{goodfellow2016deep}
I.~Goodfellow, Y.~Bengio, and A.~Courville.
\newblock {\em Deep learning}.
\newblock MIT press, 2016.

\bibitem{gu2017badnets}
T.~Gu, B.~Dolan-Gavitt, and S.~Garg.
\newblock Badnets: Identifying vulnerabilities in the machine learning model supply chain.
\newblock {\em arXiv preprint arXiv:1708.06733}, 2017.

\bibitem{hard2018federated}
A.~Hard, K.~Rao, R.~Mathews, S.~Ramaswamy, F.~Beaufays, S.~Augenstein, H.~Eichner, C.~Kiddon, and D.~Ramage.
\newblock Federated learning for mobile keyboard prediction.
\newblock {\em arXiv preprint arXiv:1811.03604}, 2018.

\bibitem{Hardt2016Equality}
M.~Hardt, E.~Price, and N.~Srebro.
\newblock Equality of opportunity in supervised learning.
\newblock In {\em NeurIPS}, pages 3315--3323, 2016.

\bibitem{Tatsunori2018Fairness}
T.~B. Hashimoto, M.~Srivastava, H.~Namkoong, and P.~Liang.
\newblock Fairness without demographics in repeated loss minimization.
\newblock In {\em ICML}, volume~80, pages 1934--1943, 2018.

\bibitem{He2020Group}
C.~He, M.~Annavaram, and S.~Avestimehr.
\newblock Group knowledge transfer: Federated learning of large cnns at the edge.
\newblock In {\em NeurIPS}, pages 14068--14080, 2020.

\bibitem{Hu2023Source}
Y.~Hu, Z.~Huang, R.~Liu, X.~Xue, X.~Sun, L.~Song, and K.~C. Tan.
\newblock Source free semi-supervised transfer learning for diagnosis of mental disorders on fmri scans.
\newblock {\em {IEEE} Trans. Pattern Anal. Mach. Intell.}, 45(11):13778--13795, 2023.

\bibitem{huang2023generalizable}
W.~Huang, M.~Ye, Z.~Shi, and B.~Du.
\newblock Generalizable heterogeneous federated cross-correlation and instance similarity learning.
\newblock {\em IEEE Trans. Pattern Anal. Mach. Intell.}, (99):1--15, 2023.

\bibitem{Huang2024Federated}
W.~Huang, M.~Ye, Z.~Shi, G.~Wan, H.~Li, B.~Du, and Q.~Yang.
\newblock Federated learning for generalization, robustness, fairness: A survey and benchmark.
\newblock {\em IEEE Trans. Pattern Anal. Mach. Intell.}, pages 1--20, 2024.

\bibitem{Huang2020MetaPoison}
W.~R. Huang, J.~Geiping, L.~Fowl, G.~Taylor, and T.~Goldstein.
\newblock Metapoison: Practical general-purpose clean-label data poisoning.
\newblock In {\em NeurIPS}, 2020.

\bibitem{kairouz2021advances}
P.~Kairouz, H.~B. McMahan, B.~Avent, A.~Bellet, M.~Bennis, A.~N. Bhagoji, K.~Bonawitz, Z.~Charles, G.~Cormode, R.~Cummings, et~al.
\newblock Advances and open problems in federated learning.
\newblock {\em Found. Trends Mach. Learn.}, 14(1--2):1--210, 2021.

\bibitem{Kang2024FedAND}
H.~Kang, M.~Kim, B.~Lee, and H.~Kim.
\newblock Fedand: Federated learning exploiting consensus admm by nulling drift.
\newblock {\em IEEE Trans. Ind. Inform.}, 20(7):9837--9849, 2024.

\bibitem{Sai2020SCAFFOLD}
S.~P. Karimireddy, S.~Kale, M.~Mohri, S.~J. Reddi, S.~U. Stich, and A.~T. Suresh.
\newblock {SCAFFOLD:} stochastic controlled averaging for federated learning.
\newblock In {\em ICML}, volume 119, pages 5132--5143, 2020.

\bibitem{krantz2002primer}
S.~G. Krantz and H.~R. Parks.
\newblock {\em A primer of real analytic functions}.
\newblock 2002.

\bibitem{krizhevsky2009learning}
A.~Krizhevsky, G.~Hinton, et~al.
\newblock Learning multiple layers of features from tiny images.
\newblock 2009.

\bibitem{kuang2023federatedscope}
W.~Kuang, B.~Qian, Z.~Li, D.~Chen, D.~Gao, X.~Pan, Y.~Xie, Y.~Li, B.~Ding, and J.~Zhou.
\newblock Federatedscope-llm: A comprehensive package for fine-tuning large language models in federated learning.
\newblock {\em arXiv preprint arXiv:2309.00363}, 2023.

\bibitem{Kummari2024Impact}
K.~N. Kumar, C.~K. Mohan, and L.~R. Cenkeramaddi.
\newblock The impact of adversarial attacks on federated learning: {A} survey.
\newblock {\em {IEEE} Trans. Pattern Anal. Mach. Intell.}, 46(5):2672--2691, 2024.

\bibitem{kurdyka1998gradients}
K.~Kurdyka.
\newblock On gradients of functions definable in o-minimal structures.
\newblock In {\em Annales de l'institut Fourier}, volume~48, pages 769--783, 1998.

\bibitem{Lamport2019Byzantine}
L.~Lamport, R.~E. Shostak, and M.~C. Pease.
\newblock The byzantine generals problem.
\newblock In {\em Concurrency: the Works of Leslie Lamport}, pages 203--226. 2019.

\bibitem{lecun1998gradient}
Y.~LeCun, L.~Bottou, Y.~Bengio, and P.~Haffner.
\newblock Gradient-based learning applied to document recognition.
\newblock {\em Proc. IEEE}, 86(11):2278--2324, 1998.

\bibitem{Li2023FedIPR}
B.~Li, L.~Fan, H.~Gu, J.~Li, and Q.~Yang.
\newblock Fedipr: Ownership verification for federated deep neural network models.
\newblock {\em {IEEE} Trans. Pattern Anal. Mach. Intell.}, 45(4):4521--4536, 2023.

\bibitem{Liping2019AAAI}
L.~Li, W.~Xu, T.~Chen, G.~B. Giannakis, and Q.~Ling.
\newblock {RSA:} byzantine-robust stochastic aggregation methods for distributed learning from heterogeneous datasets.
\newblock In {\em AAAI}, pages 1544--1551, 2019.

\bibitem{li2022federated}
Q.~Li, Y.~Diao, Q.~Chen, and B.~He.
\newblock Federated learning on non-iid data silos: An experimental study.
\newblock In {\em ICDE}, pages 965--978, 2022.

\bibitem{li2021survey}
Q.~Li, Z.~Wen, Z.~Wu, S.~Hu, N.~Wang, Y.~Li, X.~Liu, and B.~He.
\newblock A survey on federated learning systems: Vision, hype and reality for data privacy and protection.
\newblock {\em IEEE Trans. Knowl. Data Eng.}, 35(4):3347--3366, 2021.

\bibitem{Li2023term}
T.~Li, A.~Beirami, M.~Sanjabi, and V.~Smith.
\newblock On tilted losses in machine learning: Theory and applications.
\newblock {\em J. Mach. Learn. Res.}, 24:142:1--142:79, 2023.

\bibitem{li2021ditto}
T.~Li, S.~Hu, A.~Beirami, and V.~Smith.
\newblock Ditto: Fair and robust federated learning through personalization.
\newblock In {\em ICML}, pages 6357--6368, 2021.

\bibitem{li2020federated}
T.~Li, A.~K. Sahu, A.~Talwalkar, and V.~Smith.
\newblock Federated learning: Challenges, methods, and future directions.
\newblock {\em IEEE Signal Process. Mag.}, 37(3):50--60, 2020.

\bibitem{Li2019FedDANE}
T.~Li, A.~K. Sahu, M.~Zaheer, M.~Sanjabi, A.~Talwalkar, and V.~Smith.
\newblock Feddane: {A} federated newton-type method.
\newblock In {\em ACSCC}, pages 1227--1231, 2019.

\bibitem{li2020federatedoptimization}
T.~Li, A.~K. Sahu, M.~Zaheer, M.~Sanjabi, A.~Talwalkar, and V.~Smith.
\newblock Federated optimization in heterogeneous networks.
\newblock In {\em MLSys}, volume~2, pages 429--450, 2020.

\bibitem{Li2020qffl}
T.~Li, M.~Sanjabi, A.~Beirami, and V.~Smith.
\newblock Fair resource allocation in federated learning.
\newblock In {\em ICLR}, 2020.

\bibitem{li2019convergence}
X.~Li, K.~Huang, W.~Yang, S.~Wang, and Z.~Zhang.
\newblock On the convergence of fedavg on non-iid data.
\newblock In {\em ICLR}, 2019.

\bibitem{Lin2022Personalized}
S.~Lin, Y.~Han, X.~Li, and Z.~Zhang.
\newblock Personalized federated learning towards communication efficiency, robustness and fairness.
\newblock In {\em NeurIPS}, pages 30471--30485, 2022.

\bibitem{Lin2020Ensemble}
T.~Lin, L.~Kong, S.~U. Stich, and M.~Jaggi.
\newblock Ensemble distillation for robust model fusion in federated learning.
\newblock In {\em NeurIPS}, pages 2351--2363, 2020.

\bibitem{Liu2018Trojaning}
Y.~Liu, S.~Ma, Y.~Aafer, W.~Lee, J.~Zhai, W.~Wang, and X.~Zhang.
\newblock Trojaning attack on neural networks.
\newblock In {\em NDSS}, 2018.

\bibitem{Liu2021Accelerated}
Y.~Liu, F.~Shang, H.~Liu, L.~Kong, L.~Jiao, and Z.~Lin.
\newblock Accelerated variance reduction stochastic {ADMM} for large-scale machine learning.
\newblock {\em {IEEE} Trans. Pattern Anal. Mach. Intell.}, 43(12):4242--4255, 2021.

\bibitem{law1965ensembles}
S.~Lojasiewicz.
\newblock Ensembles semi-analytiques.
\newblock {\em Institut des Hautes Etudes Scientifiques}, 1965.

\bibitem{Luo2021No}
M.~Luo, F.~Chen, D.~Hu, Y.~Zhang, J.~Liang, and J.~Feng.
\newblock No fear of heterogeneity: Classifier calibration for federated learning with non-iid data.
\newblock In {\em NeurIPS}, pages 5972--5984, 2021.

\bibitem{luo2022disentangled}
Z.~Luo, Y.~Wang, Z.~Wang, Z.~Sun, and T.~Tan.
\newblock Disentangled federated learning for tackling attributes skew via invariant aggregation and diversity transferring.
\newblock In {\em ICML}, pages 14527--14541, 2022.

\bibitem{Lyu2020Collaborative}
L.~Lyu, X.~Xu, Q.~Wang, and H.~Yu.
\newblock Collaborative fairness in federated learning.
\newblock In Q.~Yang, L.~Fan, and H.~Yu, editors, {\em Federated Learning - Privacy and Incentive}, volume 12500, pages 189--204. 2020.

\bibitem{Lyu2024Privacy}
L.~Lyu, H.~Yu, X.~Ma, C.~Chen, L.~Sun, J.~Zhao, Q.~Yang, and P.~S. Yu.
\newblock Privacy and robustness in federated learning: Attacks and defenses.
\newblock {\em IEEE Trans. Neural Networks Learn. Syst.}, 35(7):8726--8746, 2024.

\bibitem{ma2015adding}
C.~Ma, V.~Smith, M.~Jaggi, M.~Jordan, P.~Richt{\'a}rik, and M.~Tak{\'a}c.
\newblock Adding vs. averaging in distributed primal-dual optimization.
\newblock In {\em ICML}, pages 1973--1982, 2015.

\bibitem{Ma2021FedSA}
Q.~Ma, Y.~Xu, H.~Xu, Z.~Jiang, L.~Huang, and H.~Huang.
\newblock Fedsa: {A} semi-asynchronous federated learning mechanism in heterogeneous edge computing.
\newblock {\em {IEEE} J. Sel. Areas Commun.}, 39(12):3654--3672, 2021.

\bibitem{mcmahan2017communication}
B.~McMahan, E.~Moore, D.~Ramage, S.~Hampson, and B.~A. y~Arcas.
\newblock Communication-efficient learning of deep networks from decentralized data.
\newblock In {\em AISTATS}, pages 1273--1282, 2017.

\bibitem{Mordukhovich2006Variational}
B.~S. Mordukhovich.
\newblock {\em Variational analysis and generalized differentiation I: Basic Theory}.
\newblock 2006.

\bibitem{moreau1965proximite}
J.-J. Moreau.
\newblock Proximit{\'e} et dualit{\'e} dans un espace hilbertien.
\newblock {\em Bull. Soc. Math. France}, 93:273--299, 1965.

\bibitem{nguyen2022federated}
D.~C. Nguyen, Q.-V. Pham, P.~N. Pathirana, M.~Ding, A.~Seneviratne, Z.~Lin, O.~Dobre, and W.-J. Hwang.
\newblock Federated learning for smart healthcare: A survey.
\newblock {\em ACM Comput. Surv.}, 55(3):1--37, 2022.

\bibitem{nichol2018first}
A.~Nichol, J.~Achiam, and J.~Schulman.
\newblock On first-order meta-learning algorithms.
\newblock {\em arXiv preprint arXiv:1803.02999}, 2018.

\bibitem{Pillutla2022Robust}
K.~Pillutla, S.~M. Kakade, and Z.~Harchaoui.
\newblock Robust aggregation for federated learning.
\newblock {\em {IEEE} Trans. Signal Process.}, 70:1142--1154, 2022.

\bibitem{rockafellar2009variational}
R.~T. Rockafellar and R.~J.-B. Wets.
\newblock {\em Variational analysis}.
\newblock 1998.

\bibitem{Shafahi2018Poison}
A.~Shafahi, W.~R. Huang, M.~Najibi, O.~Suciu, C.~Studer, T.~Dumitras, and T.~Goldstein.
\newblock Poison frogs! targeted clean-label poisoning attacks on neural networks.
\newblock In {\em NeurIPS}, pages 6106--6116, 2018.

\bibitem{shang2022federated}
X.~Shang, Y.~Lu, G.~Huang, and H.~Wang.
\newblock Federated learning on heterogeneous and long-tailed data via classifier re-training with federated features.
\newblock In {\em IJCAI}, pages 2218--2224, 2022.

\bibitem{Shen2014Augmented}
Y.~Shen, Z.~Wen, and Y.~Zhang.
\newblock Augmented lagrangian alternating direction method for matrix separation based on low-rank factorization.
\newblock {\em Optim. Methods Softw.}, 29(2):239--263, 2014.

\bibitem{shiota1997geometry}
M.~Shiota and M.~Shiota.
\newblock {\em Geometry of subanalytic and semialgebraic sets}.
\newblock 1997.

\bibitem{smith2017federated}
V.~Smith, C.-K. Chiang, M.~Sanjabi, and A.~S. Talwalkar.
\newblock Federated multi-task learning.
\newblock In {\em NeurIPS}, 2017.

\bibitem{smith2018cocoa}
V.~Smith, S.~Forte, M.~Chenxin, M.~Tak{\'a}{\v{c}}, M.~I. Jordan, and M.~Jaggi.
\newblock Cocoa: A general framework for communication-efficient distributed optimization.
\newblock {\em J. Mach. Learn. Res.}, 18:230, 2018.

\bibitem{Sun2023Decentralized}
T.~Sun, D.~Li, and B.~Wang.
\newblock Decentralized federated averaging.
\newblock {\em {IEEE} Trans. Pattern Anal. Mach. Intell.}, 45(4):4289--4301, 2023.

\bibitem{sun2019can}
Z.~Sun, P.~Kairouz, A.~T. Suresh, and H.~B. McMahan.
\newblock Can you really backdoor federated learning?
\newblock {\em arXiv preprint arXiv:1911.07963}, 2019.

\bibitem{t2020personalized}
C.~T~Dinh, N.~Tran, and J.~Nguyen.
\newblock Personalized federated learning with moreau envelopes.
\newblock In {\em NeurIPS}, pages 21394--21405, 2020.

\bibitem{tan2022towards}
A.~Z. Tan, H.~Yu, L.~Cui, and Q.~Yang.
\newblock Towards personalized federated learning.
\newblock {\em {IEEE} Trans. Neural Networks Learn. Syst.}, 34(12):9587--9603, 2023.

\bibitem{vettoruzzo2024advances}
A.~Vettoruzzo, M.-R. Bouguelia, J.~Vanschoren, T.~Rognvaldsson, and K.~Santosh.
\newblock Advances and challenges in meta-learning: A technical review.
\newblock {\em {IEEE} Trans. Pattern Anal. Mach. Intell.}, 2024.

\bibitem{wang2023can}
B.~Wang, Y.~J. Zhang, Y.~Cao, B.~Li, H.~B. McMahan, S.~Oh, Z.~Xu, and M.~Zaheer.
\newblock Can public large language models help private cross-device federated learning?
\newblock {\em arXiv preprint arXiv:2305.12132}, 2023.

\bibitem{wang2020optimizing}
H.~Wang, Z.~Kaplan, D.~Niu, and B.~Li.
\newblock Optimizing federated learning on non-iid data with reinforcement learning.
\newblock In {\em INFOCOM}, pages 1698--1707, 2020.

\bibitem{Wang2020Attack}
H.~Wang, K.~Sreenivasan, S.~Rajput, H.~Vishwakarma, S.~Agarwal, J.~Sohn, K.~Lee, and D.~S. Papailiopoulos.
\newblock Attack of the tails: Yes, you really can backdoor federated learning.
\newblock In {\em NeurIPS}, 2020.

\bibitem{Wang2020Tackling}
J.~Wang, Q.~Liu, H.~Liang, G.~Joshi, and H.~V. Poor.
\newblock Tackling the objective inconsistency problem in heterogeneous federated optimization.
\newblock In {\em NeurIPS}, 2020.

\bibitem{Wang2019Adaptive}
S.~Wang, T.~Tuor, T.~Salonidis, K.~K. Leung, C.~Makaya, T.~He, and K.~Chan.
\newblock Adaptive federated learning in resource constrained edge computing systems.
\newblock {\em {IEEE} J. Sel. Areas Commun.}, 37(6):1205--1221, 2019.

\bibitem{wang2019global}
Y.~Wang, W.~Yin, and J.~Zeng.
\newblock Global convergence of admm in nonconvex nonsmooth optimization.
\newblock {\em J. Sci. Comput.}, 78(1):29--63, 2019.

\bibitem{wu2020fedhome}
Q.~Wu, X.~Chen, Z.~Zhou, and J.~Zhang.
\newblock Fedhome: Cloud-edge based personalized federated learning for in-home health monitoring.
\newblock {\em IEEE Trans. Mob. Comput.}, 21(8):2818--2832, 2020.

\bibitem{xiao2017fashion}
H.~Xiao, K.~Rasul, and R.~Vollgraf.
\newblock Fashion-mnist: a novel image dataset for benchmarking machine learning algorithms.
\newblock {\em arXiv preprint arXiv:1708.07747}, 2017.

\bibitem{Xie2020DBA}
C.~Xie, K.~Huang, P.~Chen, and B.~Li.
\newblock {DBA:} distributed backdoor attacks against federated learning.
\newblock In {\em ICLR}, 2020.

\bibitem{xu2020reputation}
X.~Xu and L.~Lyu.
\newblock A reputation mechanism is all you need: Collaborative fairness and adversarial robustness in federated learning.
\newblock {\em arXiv preprint arXiv:2011.10464}, 2020.

\bibitem{xu2020towards}
X.~Xu and L.~Lyu.
\newblock Towards building a robust and fair federated learning system.
\newblock {\em arXiv preprint arXiv:2011.10464}, 2020.

\bibitem{Xu2013Block}
Y.~Xu and W.~Yin.
\newblock A block coordinate descent method for regularized multiconvex optimization with applications to nonnegative tensor factorization and completion.
\newblock {\em {SIAM} J. Imaging Sci.}, 6(3):1758--1789, 2013.

\bibitem{xu2012alternating}
Y.~Xu, W.~Yin, Z.~Wen, and Y.~Zhang.
\newblock An alternating direction algorithm for matrix completion with nonnegative factors.
\newblock {\em Front. Math. China}, 7:365--384, 2012.

\bibitem{yang2022robust}
S.~Yang, H.~Park, J.~Byun, and C.~Kim.
\newblock Robust federated learning with noisy labels.
\newblock {\em IEEE Intell. Syst.}, 37(2):35--43, 2022.

\bibitem{Yang2013Trading}
T.~Yang.
\newblock Trading computation for communication: Distributed stochastic dual coordinate ascent.
\newblock In {\em NeurIPS}, pages 629--637, 2013.

\bibitem{Yang2020ADMM-CSNet}
Y.~Yang, J.~Sun, H.~Li, and Z.~Xu.
\newblock Admm-csnet: {A} deep learning approach for image compressive sensing.
\newblock {\em {IEEE} Trans. Pattern Anal. Mach. Intell.}, 42(3):521--538, 2020.

\bibitem{ye2023heterogeneous}
M.~Ye, X.~Fang, B.~Du, P.~C. Yuen, and D.~Tao.
\newblock Heterogeneous federated learning: State-of-the-art and research challenges.
\newblock {\em ACM Comput. Surv.}, 2023.

\bibitem{Yin2018Byzantine}
D.~Yin, Y.~Chen, K.~Ramchandran, and P.~L. Bartlett.
\newblock Byzantine-robust distributed learning: Towards optimal statistical rates.
\newblock In {\em ICML}, volume~80, pages 5636--5645, 2018.

\bibitem{Yu2020Fairness}
H.~Yu, Z.~Liu, Y.~Liu, T.~Chen, M.~Cong, X.~Weng, D.~Niyato, and Q.~Yang.
\newblock A fairness-aware incentive scheme for federated learning.
\newblock In {\em AIES}, pages 393--399, 2020.

\bibitem{Zeng2019Global}
J.~Zeng, T.~T. Lau, S.~Lin, and Y.~Yao.
\newblock Global convergence of block coordinate descent in deep learning.
\newblock In {\em ICML}, volume~97, pages 7313--7323, 2019.

\bibitem{Zeng2021Deep}
J.~Zeng, S.~Lin, Y.~Yao, and D.~Zhou.
\newblock On {ADMM} in deep learning: Convergence and saturation-avoidance.
\newblock {\em J. Mach. Learn. Res.}, 22:199:1--199:67, 2021.

\bibitem{zhang2022federated}
J.~Zhang, Z.~Li, B.~Li, J.~Xu, S.~Wu, S.~Ding, and C.~Wu.
\newblock Federated learning with label distribution skew via logits calibration.
\newblock In {\em ICML}, pages 26311--26329, 2022.

\bibitem{zhang2024towards}
J.~Zhang, S.~Vahidian, M.~Kuo, C.~Li, R.~Zhang, T.~Yu, G.~Wang, and Y.~Chen.
\newblock Towards building the federatedgpt: Federated instruction tuning.
\newblock In {\em ICASSP}, pages 6915--6919, 2024.

\bibitem{zhang2021fedpd}
X.~Zhang, M.~Hong, S.~Dhople, W.~Yin, and Y.~Liu.
\newblock Fedpd: A federated learning framework with adaptivity to non-iid data.
\newblock {\em IEEE Trans. Signal Process.}, 69:6055--6070, 2021.

\bibitem{zhang2021survey}
Y.~Zhang and Q.~Yang.
\newblock A survey on multi-task learning.
\newblock {\em IEEE Transactions on Knowledge and Data Engineering}, 34(12):5586--5609, 2021.

\bibitem{zhao2018federated}
Y.~Zhao, M.~Li, L.~Lai, N.~Suda, D.~Civin, and V.~Chandra.
\newblock Federated learning with non-iid data.
\newblock {\em arXiv preprint arXiv:1806.00582}, 2018.

\bibitem{zhou2023federated}
S.~Zhou and G.~Y. Li.
\newblock Federated learning via inexact admm.
\newblock {\em IEEE Trans. Pattern Anal. Mach. Intell.}, 45(8):9699--9708, 2023.

\bibitem{zhou2023fedgia}
S.~Zhou and G.~Y. Li.
\newblock Fedgia: An efficient hybrid algorithm for federated learning.
\newblock {\em IEEE Trans. Signal Process.}, 71:1493--1508, 2023.

\bibitem{Zhou2021Towards}
Z.~Zhou, L.~Chu, C.~Liu, L.~Wang, J.~Pei, and Y.~Zhang.
\newblock Towards fair federated learning.
\newblock In {\em KDD}, pages 4100--4101, 2021.

\bibitem{zhu2024SIGMOD}
S.~Zhu, Q.~Xu, J.~Zeng, S.~Wang, Z.~Yang, Y.~Chuanhui, and Z.~Peng.
\newblock F$^3$km: Federated, fair, and fast k-means.
\newblock {\em Proc. ACM Manag. Data}, 1(4), 2024.

\bibitem{Zhu2021Data}
Z.~Zhu, J.~Hong, and J.~Zhou.
\newblock Data-free knowledge distillation for heterogeneous federated learning.
\newblock In {\em ICML}, pages 12878--12889, 2021.

\end{thebibliography}
\newpage
\begin{IEEEbiography}
	 [{\includegraphics[width=1in,height=1.22in,clip,keepaspectratio]{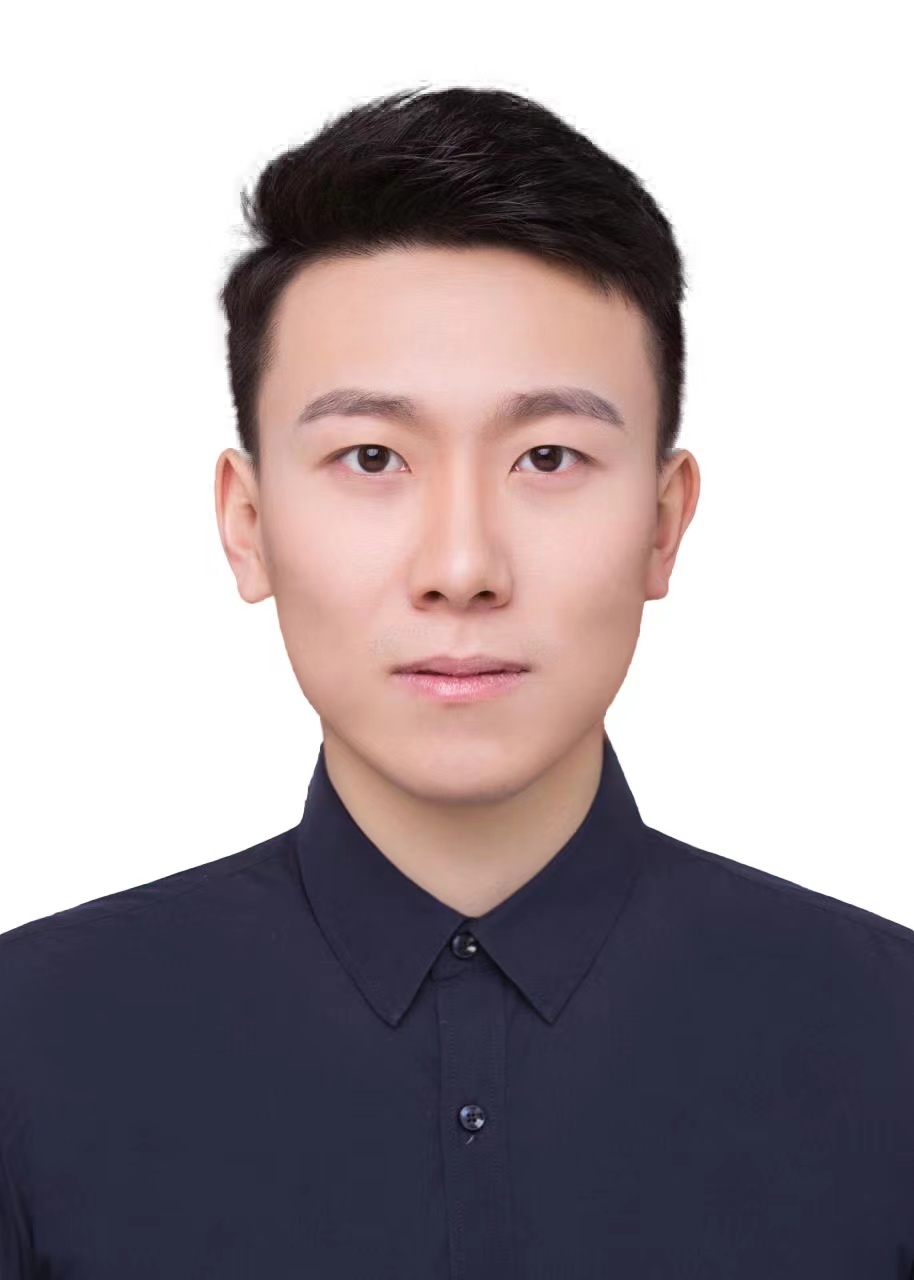}}]
	 {Shengkun Zhu} received the BE degree in electronic information engineering from Dalian University of Technology, China in 2018. He is currently working toward a Ph.D. degree in computer science and technology, School of Computer Science, Wuhan University, China. His research interests mainly include federated learning and nonconvex optimization. 
\end{IEEEbiography}
\begin{IEEEbiography}
	 [{\includegraphics[width=1in,height=1.22in,clip,keepaspectratio]{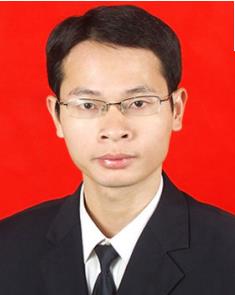}}]
	 {Jinshan Zeng} received the Ph.D. degree in mathematics from Xi’an Jiaotong University, Xi’an, China, in 2015. He is currently a Distinguished Professor with the School of Computer and Information Engineering, Jiangxi Normal University, Nanchang, China, and serves as the Director of the Department of Data Science and Big Data. He has authored more than 40 papers in high-impact journals and conferences such as IEEE TPAMI, JMLR, IEEE TSP, ICML, and AAAI. He has co-authored two papers with collaborators that received the International Consortium of Chinese Mathematicians (ICCM) Best Paper Award in 2018 and 2020). His current research interests include nonconvex optimization, machine learning (in particular deep learning), and remote sensing.
\end{IEEEbiography}

\begin{IEEEbiography}
	 [{\includegraphics[width=1in,height=1.22in,clip,keepaspectratio]{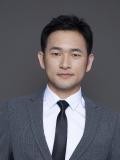}}]
	 {Sheng Wang} received the BE degree in information security, ME degree
	in computer technology from Nanjing University of Aeronautics and Astronautics,
	China in 2013 and 2016, and Ph.D. from RMIT University in 2019. He is a professor at the School of Computer Science, Wuhan University. His research interests mainly include mobile databases, multi-modal data management, and fair data analysis. He has published full research papers on top database and information systems venues as the first author, such as TKDE, SIGMOD, PVLDB, and ICDE.
\end{IEEEbiography}

\begin{IEEEbiography}
	 [{\includegraphics[width=1in,height=1.20in,clip,keepaspectratio]{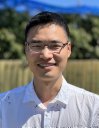}}]
	{Yuan Sun} is a Lecturer in Business Analytics and Artificial Intelligence at La Trobe University, Australia. He received his BSc in Applied Mathematics from Peking University, China, and his PhD in Computer Science from The University of Melbourne, Australia. His research interest is on artificial intelligence, machine learning, operations research, and evolutionary computation. He has contributed significantly to the emerging research area of leveraging machine learning for combinatorial optimization. His research has been published in top-tier journals and conferences such as IEEE TPAMI, IEEE TEVC, EJOR, NeurIPS, ICLR, VLDB, ICDE, and AAAI. 
\end{IEEEbiography}

\begin{IEEEbiography}
	 [{\includegraphics[width=1in,height=1.21in,clip,keepaspectratio]{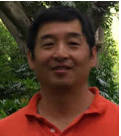}}]
	{Xiaodong Li} (Fellow, IEEE) received the B.Sc. degree from Xidian University, Xi’an, China, in 1988, and the Ph.D. degree in information science from the University of Otago, Dunedin, New Zealand, in 1998. He is a Professor at the School of Science (Computer Science and Software Engineering), RMIT University, Melbourne, VIC, Australia. His research interests include machine learning, evolutionary computation, neural networks, data analytics, multiobjective optimization, multimodal optimization, and swarm intelligence. Prof. Li is the recipient of the 2013 ACM SIGEVO Impact Award and the 2017 IEEE CIS IEEE TRANSACTIONS ON EVOLUTIONARY COMPUTATION Outstanding Paper Award. He serves as an Associate Editor for the IEEE TRANSACTIONS ON EVOLUTIONARY COMPUTATION, Swarm Intelligence (Springer), and International Journal of Swarm Intelligence Research. He is a Founding Member of IEEE CIS Task Force on Swarm Intelligence, the Vice-Chair of IEEE Task Force on Multimodal Optimization, and the Former Chair of IEEE CIS Task Force on Large-Scale Global Optimization.
\end{IEEEbiography}

\begin{IEEEbiography}
	 [{\includegraphics[width=1in,height=1.21in,clip,keepaspectratio]{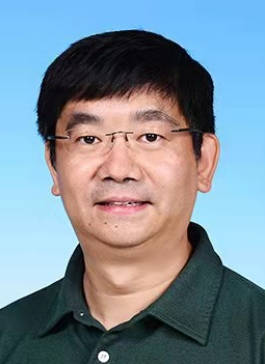}}]
	{Yuan Yao} received the B.S.E and M.S.E in control engineering both from Harbin Institute of Technology, China, in 1996 and 1998, respectively, M.Phil in mathematics from City University of Hong Kong in 2002, and Ph.D. in mathematics from the University of California, Berkeley, in 2006. Since then he has been with Stanford University and in 2009, he joined the Department of Probability and Statistics at School of Mathematical Sciences, Peking University, Beijing, China. He is currently a Professor of Mathematics, Chemical \& Biological Engineering, and by courtesy, Computer Science \& Engineering, Hong Kong University of Science and Technology, Clear Water Bay, Kowloon, Hong Kong SAR, China. His current research interests include topological and geometric methods for high-dimensional data analysis and statistical machine learning, with applications in computational biology, computer vision, and information retrieval.
\end{IEEEbiography}

\begin{IEEEbiography}
	 [{\includegraphics[width=1in,height=1.21in,clip,keepaspectratio]{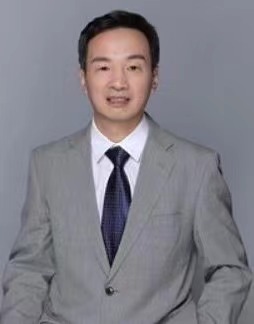}}]
	{Zhiyong Peng} received the BSc degree from
Wuhan University, in 1985, the MEng degree from the Changsha Institute of Technology of China, in 1988, and the PhD degree from the Kyoto University of Japan, in 1995. He is a professor at the School of Computer Science, Wuhan University of China. 
He worked as a researcher at Advanced Software Technology \& Mechatronics Research Institute of Kyoto from 1995 to 1997 and as a member of technical staff at Hewlett-Packard Laboratories Japan from 1997 to 2000.
His research interests include complex data management, web data management, and trusted data management. 
He is a member of IEEE Computer Society, ACM SIGMOD, and vice director of the Database Society of Chinese Computer Federation. He was general co-chair of WAIM 2011, DASFAA 2013, and PC Co-chair of DASFAA 2012, WISE 2006, and CIT 2004.
\end{IEEEbiography}

\newpage

\onecolumn
\appendices
\begin{LARGE}
    \textbf{Appendix}\\
\end{LARGE}

We provide a simple table of contents below for easier navigation of the appendix.\\

\begin{large}
    \textbf{CONTENTS}\\
\end{large}

\textbf{Section \ref{Appendix: convergence}: Convergence}\\

\quad \quad Section \ref{subappendix: useful property}: Some Useful Properties\\

\quad \quad Section \ref{subappendix: Kurdyka-Łojasiewicz Property}: Kurdyka-Łojasiewicz Property of Lagrangian Function $\ml$\\

\quad \quad Section \ref{proof of lemma 1}: Proof of Lemma \ref{lemma1}\\

\quad \quad Section \ref{Proof of Lemma 2}: Proof of Lemma \ref{lemma: relative error}\\

\quad \quad Section \ref{Proof of Theorem 1}: Proof of Theorem \ref{theorem: objective function values convergence}\\

\quad \quad Section \ref{Proof of Theorem 1.5}: Proof of Theorem \ref{theorem: sequences convergence}\\

\quad \quad Section \ref{proof of theorem 2}: Proof of Theorem \ref{theorem1}\\

\quad \quad Section \ref{Proof of Theorem 4}: Proof of Theorem \ref{theorem: convergence rate based on kl property}\\

\textbf{Section \ref{appendix: Robustness and Fairness}: Robustness and Fairness}\\

\quad \quad Section \ref{sec: Solutions of flame}: Solutions of \flame\\

\quad \quad Section \ref{secappendix: Test Loss}: Test Loss\\

\quad \quad Section \ref{appendix: robustness}: Robustness\\

\quad \quad Section \ref{appendix: fairness}: Fairness\\

\textbf{Section \ref{secappendix: Extension}: Extension}\\

\textbf{Section \ref{appendix: additional experiments}: Additional and Complete Experiment Results}\\

\quad \quad Section \ref{secappendix: Complete Results on Accuracy and Convergence}: Complete Results on Accuracy and Convergence\\

\quad \quad Section \ref{secappendix: Complete Results on Robustness}: Complete Results on Robustness\\

\quad \quad Section \ref{secappendix: Complete Results on Fairness}: Complete Results on Fairness\\
\newpage

\section{Convergence}\label{Appendix: convergence}
\subsection{Some Useful Properties}\label{subappendix: useful property}

We provide a list of useful properties and lemmas that are necessary for proving our theorems. Proposition \ref{proposition 2} provides an exposition of the property of smooth function, while Proposition \ref{jensen} presents commonly employed Jensen's inequalities. Lemmas \ref{lemma2} and \ref{lemma: L0>F*} will be employed in the proof of our main theorem.

\begin{proposition}\label{proposition 2}
    For any $L$-smooth function $f$ and $\btheta_i$, $i=1,2$, we have
\begin{align}
    f(\btheta_1)-f(\btheta_2)-\langle\nabla &f(\btheta_i), \btheta_1-\btheta_2\rangle\leq\frac{L}{2}\|\btheta_1-\btheta_2\|^2.   
\end{align}
\end{proposition}

\begin{proposition}[Jensen's inequalities]\label{jensen}
For any vectors $\btheta_1$, $\btheta_2$ and $k>0$, we have
\begin{align}
    2\langle\btheta_1,\btheta_2\rangle&\leq k\|\btheta_1\|^2+\frac{1}{k}\|\btheta_2\|^2,\\
    \|\btheta_1-\btheta_2\|^2&\leq (1+k)\|\btheta_1\|^2+(1+\frac{1}{k})\|\btheta_2\|^2,\\
    \|\sum_{i=1}^m \btheta_i\|^2&\leq m\sum_{i=1}^m\|\btheta_i\|^2.
\end{align}

\end{proposition}
\begin{lemma}\label{lemma2}
For any $t\geq 0$, the following equation holds
\begin{equation}\label{eq: relation among pi wi thetai}
\begin{split}
    \bpi_{i}^{t}=\lambda\alpha_i(\btheta_i^{t}-\bw_i^{t}).
\end{split}
\end{equation}
\end{lemma}

\begin{proof}
    Transposing Equation (\ref{eq13}) yields:
\begin{align}\label{eq: transfrom eq13}
    \bpi_i^{t}=\lambda\alpha_i\btheta_i^{t+1}+\rho\bw^t-(\lambda\alpha_i+\rho)\bw_i^{t+1}.
\end{align}

We first establish the case when $t \geq 1$.
Substituting Equation (\ref{eq: transfrom eq13}) into Equation (\ref{eq14}), we obtain:
\begin{align}
    \boldsymbol{\pi}_i^{t+1} &=\boldsymbol{\pi}_i^t+\rho(\bw_i^{t+1}-\bw^{t})\notag\\
    &=\lambda\alpha_i(\btheta_i^{t+1}-\bw_i^{t+1}).
\end{align}

Next, we establish the case when $t=0$ by initializing the parameters as 
\begin{align}
    \boldsymbol{\pi}_i^{0}=\lambda\alpha_i(\btheta_i^{0}-\bw_i^{0}).
\end{align}

Combine the two cases, we have
\begin{align}
    \bpi_{i}^{t}=\lambda\alpha_i(\btheta_i^{t}-\bw_i^{t}), t\geq 0,
\end{align}
which completes the proof.
\end{proof}

\begin{lemma}\label{lemma: L0>F*}
    When $\rho\geq\lambda\alpha_i$ for each client $i\in[m]$, the Lagrangian function value in the $(t+1)$-th iteration is lower bounded as $\ml^{t+1}\geq F^{*}$, where $F^{*}=\min_{\bw,\btheta_i}\sum_{i=1}^{m}\alpha_i \left(f_i(\btheta_i)+\frac{\lambda}{2}\|\btheta_i-\bw\|^2\right)$.
\end{lemma}
\begin{proof}
According to Equation (\ref{eq: lagrangian function}), we obtain
\begin{align}
    \ml^{t+1}&=\sum_{i=1}^m \alpha_i(f_i(\btheta_i^{t+1})+\frac{\lambda}{2}\|\btheta_i^{t+1}-\bw_i^{t+1}\|^2)+\langle\bpi_i^{t+1},\bw_i^{t+1}-\bw^{t+1}\rangle+\frac{\rho}{2}\|\bw_i^{t+1}-\bw^{t+1}\|^2\\
    &=\sum_{i=1}^m \alpha_i(f_i(\btheta_i^{t+1})+\frac{\lambda}{2}\|\btheta_i^{t+1}-\bw_i^{t+1}\|^2)+\frac{\lambda\alpha_i}{2}\langle2\btheta_i^{t+1}-\bw_i^{t+1}-\bw^{t+1},\bw_i^{t+1}-\bw^{t+1}\rangle\notag\\
    &-\frac{\lambda\alpha_i}{2}\langle2\btheta_i^{t+1}-\bw_i^{t+1}-\bw^{t+1},\bw_i^{t+1}-\bw^{t+1}\rangle+\langle\bpi_i^{t+1},\bw_i^{t+1}-\bw^{t+1}\rangle+\frac{\rho}{2}\|\bw_i^{t+1}-\bw^{t+1}\|^2\\
    &=\sum_{i=1}^m \alpha_i(f_i(\btheta_i^{t+1})+\frac{\lambda}{2}\|\btheta_i^{t+1}-\bw^{t+1}\|^2)-\frac{\lambda\alpha_i}{2}\langle2\btheta_i^{t+1}-\bw_i^{t+1}-\bw^{t+1},\bw_i^{t+1}-\bw^{t+1}\rangle+\notag\\
    &\langle\lambda\alpha_i(\btheta_i^{t+1}-\bw_i^{t+1}),\bw_i^{t+1}-\bw^{t+1}\rangle+\frac{\rho}{2}\|\bw_i^{t+1}-\bw^{t+1}\|^2\label{eq:37 follows from lemma 3}\\
    &=\sum_{i=1}^m \alpha_i(f_i(\btheta_i^{t+1})+\frac{\lambda}{2}\|\btheta_i^{t+1}-\bw^{t+1}\|^2)+\frac{\rho-\lambda\alpha_i}{2}\|\bw_i^{t+1}-\bw^{t+1}\|^2\\
    &\geq F^{*}+\sum_{i=1}^m\frac{\rho-\lambda\alpha_i}{2}\|\bw_i^{t+1}-\bw^{t+1}\|^2\geq F^{*}\label{eq: 39 follows from optimal condition},
\end{align}   
where \eqref{eq:37 follows from lemma 3} follows from Lemma \ref{lemma2} and the equality $\|\boldsymbol{a}-\boldsymbol{b}\|^2=\|\boldsymbol{a}-\boldsymbol{c}\|^2+\langle2\boldsymbol{a}-\boldsymbol{b}-\boldsymbol{c},\boldsymbol{c}-\boldsymbol{b}\rangle$, and \eqref{eq: 39 follows from optimal condition} follows from the optimal condition.
\end{proof}

\subsection{Kurdyka-Łojasiewicz Property of Lagrangian Function $\ml$}\label{subappendix: Kurdyka-Łojasiewicz Property}
To prove Theorem 1, we need to demonstrate that the KŁ property holds for the considered Lagrangian function $\ml$.
We begin by presenting the definition of KŁ property and elucidating some properties of real analytic and semialgebraic functions.
\begin{definition}[Desingularizing Function]\label{definition: Desingularizing Function}
    A function $\psi:[0,\kappa)\to(0,+\infty)$ satisfying the following conditions is a desingularizing function:
\begin{itemize}
    \item[a)] $\psi$ is concave and continuously differentiable on $(0,\kappa)$;
    \item[b)] $\psi$ is continuous at 0 and $\psi(0)=0$;
    \item[c)] For any $x\in(0,\kappa), $ $\psi'(x)>0$.
\end{itemize}
\end{definition}
\begin{definition}[Fr\'{e}chet Subdifferential \cite{rockafellar2009variational} and Limiting Subdifferential \cite{Mordukhovich2006Variational}]
    For any $\bx\in\text{dom}(h)$, the Fr\'{e}chet subdifferential of $h$ at $\bx$, represented by $\hat{\partial}h(\bx)$, is the set of vectors $\boldsymbol{z}$ which satisfies 
\begin{align*}
    \lim_{\boldsymbol{y}\neq \boldsymbol{x},}\inf_{\boldsymbol{y}\to \bx}\frac{h(\boldsymbol{y})-h(\bx)-\langle
    \boldsymbol{z},\boldsymbol{y}-\bx\rangle}{\|\bx-\boldsymbol{y}\|}\geq 0.
\end{align*}

When $\bx\notin \text{dom}(h)$, we set $\hat{\partial}h(\bx)=\emptyset$. The limiting subdifferential (or simply subdifferential) of $h$, represented by $\partial h(\bx)$ at $\bx\in\text{dom}(h)$ is defined by
\begin{align*}
    \partial h(\bx):=\{\boldsymbol{z}\in\mathbb{R}^p:\exists\bx^t\to\bx, h(\bx^t)\to h(\bx), \boldsymbol{z}^t\in\hat{\partial}h(\bx)\to\boldsymbol{z}\}.
\end{align*}
\end{definition}

\begin{definition}[Kurdyka-Łojasiewicz Property \cite{bolte2007lojasiewicz}]\label{definition: kl property}
    A function $h:\mathbb{R}^p\to\mathbb{R}\cup\{+\infty\}$ is said to have Kurdyka-Łojasiewicz (KŁ) Property at $\bx^*\in \text{dom}(\partial h)$ if there exist a neighborhood $U$ of $\bx^*$, a constant $\kappa$, and a desingularizing function $\psi$, such that for all $\bx\in U\cap \text{dom}(\partial h)$ and $h(\bx^*)<h(\bx)<h(\bx^*)+\kappa$, it holds that,
\begin{align}\label{eq: kl inequality}
    \psi'(h(\bx)-h(\bx^*))\text{dist}(0,\partial h(\bx))\geq 1,
\end{align}
where $\text{dist}(0,\partial h(\bx)):=\inf\{\|\boldsymbol{z}\|:\boldsymbol{z}\in\partial h(\bx)\}$ represents the distance between zero to the set $\partial h(\bx)$.
If $h$ satisfies the KŁ property, then $f$ is called a KŁ function.
\end{definition}

The definition of the KŁ property means that the function under consideration is sharp up to a reparametrization \cite{Attouch2009convergence}. Particularly, when $h$ is smooth, finite-valued, and $h(\bx^*)=0$, then \eqref{eq: kl inequality} can be represented as
\begin{align}\label{eq: kl inequality another case}
    \|\nabla(\psi\circ h)(\bx)\|\geq 1.
\end{align}

\eqref{eq: kl inequality another case} can be interpreted as follows:
by reparametrizing the values of $h$ through $\psi$, we obtain a well-defined function. The function $\psi$ serves the purpose of transforming a singular region, where gradients are arbitrarily small, into a regular region, characterized by gradients bounded away from zero.
The class of KŁ functions encompasses a diverse range of function types, comprising real analytic functions (as defined in Definition \ref{definition: real analytic}), semialgebraic functions (as delineated in Definition \ref{definition: Semialgebraic}), as well as tame functions defined within certain o-minimal structures \cite{kurdyka1998gradients}. Additionally, it encompasses continuous subanalytic functions \cite{bolte2007lojasiewicz} and locally strongly convex functions \cite{Xu2013Block}.
\begin{lemma}[Properties of real analytic functions \cite{krantz2002primer}]
    The sums, products, and compositions of real analytic functions are real analytic functions.
\end{lemma}

\begin{lemma}[Properties of semialgebraic functions \cite{bochnak2013real}]
\quad
\begin{itemize}
    \item[a)] The finite union, finite intersection, and complement of semialgebraic sets are semialgebraic. The closure and the interior of a semialgebraic set are semialgebraic. 
    \item[b)] The composition $g\circ h$ of semialgebraic mappings $g: A\to B$ and $h: B\to C$ is semialgebraic.
    \item[c)] The sum of two semialgebraic functions is semialgebraic.
\end{itemize}
\end{lemma}

\begin{lemma}[Subanalytic functions \cite{shiota1997geometry}]\label{lemma: Subanalytic functions}
\quad
\begin{itemize}
    \item[a)] Both real analytic functions and semialgebraic functions are subanalytic.
    \item[b)] Let $f_1$ and $f_2$ be both subanalytic functions, then the sum of $f_1$ and $f_2$, i.e., $f_1+f_2$ is a subanalytic function if at least one of them map a bounded set to a bounded set or if both of them are nonnegative.
\end{itemize}
\end{lemma}

\begin{lemma}[Property of subanalytic functions \cite{bolte2007lojasiewicz}]\label{lemma: Property of subanalytic functions}
    Let $h:\mathbb{R}^p\to \mathbb{R}\cup\{+\infty\}$ be a subanalytic function with closed domain, and assume that $h$ is continuous on its domain, then $h$ is a \kl function.
\end{lemma}

\begin{proposition}[\kl property of $\ml$]\label{proposition: kl property of L}
    Suppose that Assumption \ref{assumption: either real analytic or semialgebraic} holds, then the Lagrangian function $\ml$ defined in \eqref{eq: lagrangian function} is a \kl function.
\end{proposition}
\begin{proof}
    From \eqref{eq: lagrangian function}, we have
\begin{equation}
\begin{split}
    \ml(\Theta,W,\Pi,\bw)&:=\sum_{i=1}^m\ml_i(\btheta_i,\bw_i,\bw,\bpi_i), \\
    \ml_i(\btheta_i,\bw_i,\bpi_i,\bw)&:=\alpha_i(f_i(\btheta_i)+\frac{\lambda}{2}\|\btheta_i\!-\bw_i\|^2)+\langle\bpi_i,\bw_i-\bw\rangle+\frac{\rho}{2}\|\bw_i-\bw\|^2,
\end{split}
\end{equation}     
which mainly includes the following types of functions, i.e.,
\begin{align*}
    f_i(\btheta_i),\,\|\btheta_i\!-\bw_i\|^2,\,\langle\bpi_i,\bw_i-\bw\rangle,\,\|\bw_i-\bw\|^2.
\end{align*}

Following from Assumption \ref{assumption: either real analytic or semialgebraic}, we have $f_i(\btheta_i)$ is either real analytic or semialgebraic. Note that $\|\btheta_i\!-\bw_i\|^2,\,\langle\bpi_i,\bw_i-\bw\rangle,\,\|\bw_i-\bw\|^2$ are all polynomial functions with variables $\btheta_i,\,\bw_i,\bpi_i,$ and $\bw$, thus according to \cite{krantz2002primer} and \cite{bochnak2013real}, they are both real analytic and semialgebraic. Since each term of $\ml$ is either real analytic or semialgebraic, then according to Lemma \ref{lemma: Subanalytic functions}, $\ml$ is a subanalytic function. Next, we can infer that $\ml$ is a \kl function according to Lemma \ref{lemma: Property of subanalytic functions}.
\end{proof}

\subsection{Proof of Lemma \ref{lemma1}}\label{proof of lemma 1}
\begin{proof}
We decompose the gap between $\ml^{t+1}$ and $\ml^t$ as
\begin{align}\label{eq:Lt+1-Lt}
    \ml(\mathcal{P}^{t+1})-\ml(\mathcal{P}^{t})&=\mathcal{L}(\Theta^{t+1}, W^{t+1}, \Pi^{t+1}, \bw^{t+1})-\mathcal{L}(\Theta^{t}, W^{t}, \Pi^{t}, \bw^{t})\notag\\
    &=\underbrace{\mathcal{L}(\Theta^{t+1}, W^{t+1}, \Pi^{t+1}, \bw^{t+1})\!-\!\mathcal{L}(\Theta^{t+1}, W^{t+1}, \Pi^{t+1}, \bw^{t})}_{e_1^t}+\underbrace{\mathcal{L}(\Theta^{t+1}, W^{t+1}, \Pi^{t+1}, \bw^{t})\!-\!\mathcal{L}(\Theta^{t+1}, W^{t+1}, \Pi^{t}, \bw^{t})}_{e_2^t}\notag\\
    &\quad+\underbrace{\mathcal{L}(\Theta^{t+1}, W^{t+1}, \Pi^{t}, \bw^{t})\!-\!\mathcal{L}(\Theta^{t+1}, W^{t}, \Pi^{t}, \bw^{t})}_{e_3^t}+\underbrace{\mathcal{L}(\Theta^{t+1}, W^{t}, \Pi^{t}, \bw^{t})\!-\!\mathcal{L}(\Theta^{t}, W^{t}, \Pi^{t}, \bw^{t})}_{e_4^t}.
\end{align}

Next, we individually estimate $e_1^t$, $e_2^t$, $e_3^t$, and $e_4^t$.

\noindent\underline{Estimate of $e_1^t$}. By employing the Lagrangian function, we can deduce the subsequent expression:
\begin{align}\label{eq:e1}
    e_1^t&=\sum_{i=1}^{m}\langle\bpi_i^{t+1},\bw_i^{t+1}-\bw^{t+1}\rangle-\langle\bpi_i^{t+1},\bw_i^{t+1}-\bw^{t}\rangle+\frac{\rho}{2}\|\bw_i^{t+1}-\bw^{t+1}\|^2-\frac{\rho}{2}\|\bw_i^{t+1}-\bw^{t}\|^2\notag\\
    &=\sum_{i=1}^{m}\langle\bpi_i^{t+1}+\rho\bw_i^{t+1}-\frac{\rho}{2}(\bw^{t}+\bw^{t+1}),\bw^{t}-\bw^{t+1}\rangle\\
    &=-\frac{\rho m}{2}\|\bw^{t+1}-\bw^{t}\|^2\label{eq:46 follows from equation 15},
\end{align}    
where \eqref{eq:46 follows from equation 15} follows from Equation (\ref{eq15}).

\noindent\underline{Estimate of $e_2^t$}. Utilizing the Lagrangian function, we obtain
\begin{align}
    e_2^t&=\sum_{i=1}^m\langle\bpi^{t+1}_i-\bpi^t_i,\bw_i^{t+1}-\bw^{t}\rangle \notag\\
    &=\underbrace{\sum_{i\notin\mathcal{S}^t}\langle\bpi^{t+1}_i-\bpi^t_i,\bw_i^{t+1}-\bw^{t}\rangle}_{=0}+\sum_{i\in\mathcal{S}^t}\langle\bpi^{t+1}_i-\bpi^t_i,\bw_i^{t+1}-\bw^{t}\rangle\label{eq: 46 follows from line 18}\\
    &=\sum_{i\in\mathcal{S}^t}\frac{1}{\rho}\|\bpi^{t+1}_i-\bpi^t_i\|^2\label{eq: 47 follows from eq14}\\
    &=\sum_{i\in\mathcal{S}^t}\frac{\lambda^2\alpha_i^2}{\rho}\|\btheta_i^{t+1}-\btheta_i^t-\bw_i^{t+1}+\bw^t_i\|^2\label{eq: 48 follows from lemma2}\\
    &\leq\sum_{i\in\mathcal{S}^t}\frac{\lambda^2\alpha_i^2(1+\rho)}{\rho}(\|\btheta_i^{t+1}\!-\!\btheta_i^t\|^2\!+\!\frac{1}{\rho}\|\bw_i^{t+1}\!-\!\bw^t_i\|^2)\label{eq:e2},
\end{align}    
where \eqref{eq: 46 follows from line 18} follows from Line \ref{line 16 alg 1} in Algorithm \ref{alg: flame on server side}, \eqref{eq: 47 follows from eq14} follows from (\ref{eq14}), \eqref{eq: 48 follows from lemma2} follows from Lemma \ref{lemma2}, and \eqref{eq:e2} follows from Proposition \ref{jensen}.

\noindent\underline{Estimate of $e_3^t$}. By employing the Lagrangian function, we obtain
\begin{align}
    e_3^t&=\sum_{i=1}^m\frac{\lambda\alpha_i}{2}\|\btheta_i^{t+1}-\bw_i^{t+1}\|^2-\frac{\lambda\alpha_i}{2}\|\btheta_i^{t+1}-\bw_i^t\|^2+\langle\bpi_i^t,\bw_i^{t+1}-\bw^t\rangle-\langle\bpi_i^t,\bw_i^{t}-\bw^t\rangle+\frac{\rho}{2}\|\bw_i^{t+1}-\bw^t\|^2-\frac{\rho}{2}\|\bw_i^{t}-\bw^t\|^2\notag\\
    &=\sum_{i=1}^m\frac{\lambda\alpha_i}{2}\langle2\btheta_i^{t+1}-\bw_i^t-\bw_i^{t+1},\bw_i^t-\bw_i^{t+1}\rangle+\langle\bpi_i^t,\bw_i^{t+1}-\bw_i^{t}\rangle+\frac{\rho}{2}\langle\bw_i^{t+1}+\bw_i^{t}-2\bw^t,\bw_i^{t+1}-\bw_i^{t}\rangle\\
    &=\sum_{i\in\mathcal{S}^t}\langle\frac{\lambda\alpha_i+\rho}{2}(\bw_i^{t+1}+\bw_i^t)-\lambda\alpha_i\btheta_i^{t+1}+\bpi_i^t-\rho\bw^t,\bw_i^{t+1}-\bw_i^t\rangle\label{eq: 51 follows from line 18 in algorithm 1}\\
    &=\sum_{i\in\mathcal{S}^t}\!\langle\frac{\lambda\alpha_i+\rho}{2}(\bw_i^{t+1}\!+\!\bw_i^t)-(\lambda\alpha_i+\rho)\bw_i^{t+1},\bw_i^{t+1}-\bw_i^t\rangle\label{eq: 52 follows from eq 12}\\
    &=-\sum_{i\in\mathcal{S}^t}\frac{\lambda\alpha_i+\rho}{2}\|\bw_i^{t+1}-\bw_i^t\|^2\label{eq:e3}.
\end{align}
where \eqref{eq: 51 follows from line 18 in algorithm 1} follows from Line \ref{line 16 alg 1} in Algorithm \ref{alg: flame on server side}, \eqref{eq: 52 follows from eq 12} follows from Equation (\ref{eq13}).

\noindent\underline{Estimate of $e_4^t$}. By employing the Lagrangian function, we obtain
\begin{align}
    e_4^t&=\!\sum_{i=1}^m\!\alpha_i (f_i(\btheta_i^{t+1})\!+\!\frac{\lambda}{2}\|\btheta_i^{t+1}\!\!-\!\bw_i^t\|^2\!- \!f_i(\btheta_i^t)\!-\!\frac{\lambda}{2}\|\btheta_i^{t}\!-\!\bw_i^t\|^2)\notag\\
    &\leq\sum_{i=1}^{m}\alpha_i\langle \nabla f_i(\btheta_i^{t+1}),\btheta_i^{t+1}-\btheta_i^t\rangle+\frac{ L\alpha_i}{2}\|\btheta_i^{t+1}-\btheta_i^t\|^2+\frac{\lambda\alpha_i}{2}\langle\btheta_i^{t+1}+\btheta_i^{t}-2\bw_i^t,\btheta_i^{t+1}-\btheta_i^t\rangle\label{eq: 54 follows from assumption lipschitz}\\
    &=\sum_{i\in\mathcal{S}^t}\alpha_i\langle\nabla f_i(\btheta_i^{t+1})+\lambda(\btheta_i^{t+1}-\bw_i^t),\btheta_i^{t+1}-\btheta_i^t\rangle+\frac{\alpha_i (L-\lambda)}{2}\|\btheta_i^{t+1}-\btheta_i^t\|^2\\
    &\leq\sum_{i\in\mathcal{S}^t}(\frac{1-\rho^2}{\rho^2}\lambda^2\alpha_i^2-\frac{L\alpha_i+\rho}{2})\|\btheta_i^{t+1}-\btheta_i^t\|^2+\frac{\alpha_i(L-\lambda)}{2}\|\btheta_i^{t+1}-\btheta_i^t\|^2+\frac{\alpha_i^2}{(\frac{1}{\rho^2}-1)\lambda^2\alpha_i^2-\frac{L\alpha_i+\rho}{2}}\epsilon_i^t\label{eq: 56 follows from eq16}\\
    &=\sum_{i\in\mathcal{S}^t}(\frac{1-\rho^2}{\rho^2}\lambda^2\alpha_i^2-\frac{\lambda\alpha_i+\rho}{2})\|\btheta_i^{t+1}-\btheta_i^t\|^2+\frac{\alpha_i^2}{(\frac{1}{\rho^2}-1)\lambda^2\alpha_i^2-\frac{L\alpha_i+\rho}{2}}\epsilon_i^{t+1}\\
    &\leq\sum_{i\in\mathcal{S}^t}\Bigl(\frac{1-\rho^2}{\rho^2}\lambda^2\alpha_i^2-\frac{\lambda\alpha_i+\rho}{2})\|\btheta_i^{t+1}-\btheta_i^t\|^2-\frac{\alpha_i^2}{\bigl((\frac{1}{\rho^2}-1)\lambda^2\alpha_i^2-\frac{L\alpha_i+\rho}{2}\bigr)(1-\upsilon_i)}(\epsilon_i^{t+1}-\epsilon_i^{t})\Bigr)\label{eq:e4}.
\end{align}
where \eqref{eq: 54 follows from assumption lipschitz} follows from Assumption \ref{assmption lipschitz}, and \eqref{eq: 56 follows from eq16} follows from Equation (\ref{eq16}) and Proposition \ref{jensen}. 
Next, substituting Equations (\ref{eq:e1}), (\ref{eq:e2}), (\ref{eq:e3}), and (\ref{eq:e4}) into Equation (\ref{eq:Lt+1-Lt}), we obtain
\begin{align}\label{eq: final Lt-Lt+1}
    \ml(\mathcal{P}^{t+1})-\ml(\mathcal{P}^{t})\leq&-\frac{\rho m}{2}\|\bw^{t+1}-\bw^{t}\|^2+\sum_{i\in\mathcal{S}^t}\Bigl(
    \frac{\lambda^2\alpha_i^2(1+\rho)}{\rho}\times(\|\btheta_i^{t+1}-\btheta_i^t\|^2+\frac{1}{\rho}\|\bw_i^{t+1}-\bw^t_i\|^2)-\frac{\lambda\alpha_i+\rho}{2}\|\bw_i^{t+1}-\bw_i^t\|^2\notag\\
    &+(\frac{1-\rho^2}{\rho^2}\lambda^2\alpha_i^2-\frac{\lambda\alpha_i+\rho}{2})\|\btheta_i^{t+1}-\btheta_i^t\|^2-\frac{\alpha_i^2}{\bigl((\frac{1}{\rho^2}-1)\lambda^2\alpha_i^2-\frac{L\alpha_i+\rho}{2}\bigr)(1-\upsilon_i)}(\epsilon_i^{t+1}-\epsilon_i^{t})\Bigr)\notag\\
    =&-\frac{\rho m}{2}\|\bw^{t+1}-\bw^{t}\|^2+\sum_{i\in\mathcal{S}^t}\Bigl((\frac{\lambda^2\alpha_i^2(1+\rho)}{\rho^2}-\frac{\lambda\alpha_i+\rho}{2})\times(\|\bw_i^{t+1}-\bw_i^t\|^2+\|\btheta_i^{t+1}-\btheta_i^t\|^2)\notag\\
    &-\frac{\alpha_i^2}{\bigl((\frac{1}{\rho^2}-1)\lambda^2\alpha_i^2-\frac{L\alpha_i+\rho}{2}\bigr)(1-\upsilon_i)}(\epsilon_i^{t+1}-\epsilon_i^{t})\Bigr).
\end{align}
Following the fact that the local parameters remain unchanged for clients not included in $\mathcal{S}^t$ and according to \eqref{eq: iteration condition}, we have
\begin{align}\label{eq: sufficient descent lemma in appendix}
    \tilde{\ml}(\mathcal{P}^{t})-\tilde{\ml}(\mathcal{P}^{t+1})\geq\mathcal{D}_1\sum_{i=1}^m(\|\bw^{t+1}-\bw^{t}\|^2+\|\bw^{t+1}_i-\bw_i^{t}\|^2+\|\btheta_i^{t+1}-\btheta_i^t\|^2),
\end{align}
where $\mathcal{D}_1:=\min_i\{\frac{\rho}{2},\frac{\lambda\alpha_i+\rho}{2}-\frac{\lambda^2\alpha_i^2(1+\rho)}{\rho^2}\}$.
\end{proof}

\subsection{Proof of Lemma \ref{lemma: relative error}} \label{Proof of Lemma 2}
\begin{proof}
    Since we have 
\begin{align*}
    \nabla \tilde{\ml}(\mathcal{P}^t)=(\{\nabla_{\btheta_i}\tilde{\ml}\}_{i=1}^m,\{\nabla_{\bw_i}\tilde{\ml}\}_{i=1}^m,\{\nabla_{\bpi_i}\tilde{\ml}\}_{i=1}^m,\{\nabla_{\bw}\tilde{\ml}\})(\mathcal{P}^t),
\end{align*}
then it is obvious that
\begin{equation}\label{eq: L(P)}
\begin{split}
    \|\nabla \tilde{\ml}(\mathcal{P}^t)\|^2=\|\nabla_{\bw} \tilde{\ml}(\mathcal{P}^t)\|^2+\sum_{i=1}^m\|\nabla_{\btheta_i} \tilde{\ml}(\mathcal{P}^t)\|^2+\|\nabla_{\bw_i} \tilde{\ml}(\bw_i^t)\|^2+\|\nabla_{\bpi_i} \tilde{\ml}(\bpi_i^t)\|^2.
\end{split}
\end{equation}
Next, we individually bound $\|\nabla_{\bw} \tilde{\ml}(\mathcal{P}^t)\|^2$, $\|\nabla_{\btheta_i} \tilde{\ml}(\mathcal{P}^t)\|^2$, $\|\nabla_{\bw_i} \tilde{\ml}(\bw_i^t)\|^2$ and $\|\nabla_{\bpi_i} \tilde{\ml}(\bpi_i^t)\|^2$.

\noindent\underline{Bound of $\|\nabla_{\bw} \tilde{\ml}(\mathcal{P}^t)\|^2$}. By utilizing the Lagrangian function, we derive   
\begin{equation}\label{eq:bound of w}
\begin{split}
    &\|\nabla_{\bw} \tilde{\ml}(\mathcal{P}^t)\|^2=\|\sum_{i=1}^m\rho(\bw_i^t-\bw^t)-\bpi_i^t\|=0,
\end{split}
\end{equation}
where \eqref{eq:bound of w} follows from Equation (\ref{eq15}).

\noindent\underline{Bound of $\|\nabla_{\btheta_i} \tilde{\ml}(\mathcal{P}^t)\|^2$}. According to the Lagrangian function, We denote
\begin{align}
    \|\nabla_{\btheta_i} \tilde{\ml}(\mathcal{P}^t)\|^2&=\|\alpha_i(\nabla f_i(\btheta_i^t)+\lambda(\btheta_i^t-\bw_i^t))\|^2=\alpha_i^2 \|\nabla f_i(\btheta_i^t)+\lambda\btheta_i^t-\nabla f_i(\btheta_i^{t+1})-\lambda\btheta_i^{t+1}+\nabla f_i(\btheta_i^{t+1})+\lambda(\btheta_i^{t+1}-\bw_i^t)\|^2\notag\\
    &\leq2\alpha_i^2\|\nabla f_i(\btheta_i^t)+\lambda\btheta_i^t-\nabla f_i(\btheta_i^{t+1})-\lambda\btheta_i^{t+1}\|^2+2\alpha_i^2\|\nabla f_i(\btheta_i^{t+1})+\lambda(\btheta_i^{t+1}-\bw_i^t)\|^2\label{eq:63 follows from proposition jensen}\\
    &\leq4\alpha_i^2\|\nabla f_i(\btheta_i^t)-\nabla f_i(\btheta_i^{t+1})\|^2+4\alpha_i^2\lambda^2\|\btheta_i^t-\btheta_i^{t+1}\|^2+2\epsilon_i^{t+1}\label{eq:64 follows from jensen and 16}\\
    &\leq4\alpha_i^2(L^2+\lambda^2)\|\btheta_i^{t+1}-\btheta_i^{t}\|^2+2\epsilon_i^{t+1}\label{eq: gradient of theta},
\end{align}
where \eqref{eq:63 follows from proposition jensen} follows from Proposition \ref{jensen}, \eqref{eq:64 follows from jensen and 16} follows from Proposition \ref{jensen} and (\ref{eq16}), and \eqref{eq: gradient of theta} is due to the $L$-smoothness of $f_i$.

\noindent\underline{Bound of $\|\nabla_{\bw_i} \tilde{\ml}(\mathcal{P}^t)\|^2$}. Leveraging the Lagrangian function, we denote    
\begin{align}
    \|\nabla_{\bw_i} \tilde{\ml}(\mathcal{P}^t)\|^2&=\|\alpha_i\lambda(\bw_i^t-\btheta_i^t)+\bpi_i^t+\rho(\bw_i^t-\bw^t)\|^2 \notag\\
    &=\|\rho(\bw_i^t-\bw^t)\|^2\label{eq:66 follows from lemma 2}\\
    &=\|\rho(\bw_i^t-\bw_i^{t+1}+\bw_i^{t+1}-\bw^t)\|^2  \label{eq: 67 follows from eq14}\\
    &\leq2\rho^2\|\bw_i^t-\bw_i^{t+1}\|^2+2\|\bpi_i^{t+1}-\bpi_i^t\|\\
    &\leq(2\rho^2+4\lambda^2\alpha_i^2)\|\bw_i^t-\bw_i^{t+1}\|^2+4\lambda^2\alpha_i^2\|\btheta_i^{t+1}-\btheta_i^t\|^2.\label{eq: wi-w}
\end{align}   
where \eqref{eq:66 follows from lemma 2} follows from Lemma \ref{lemma2}, \eqref{eq: 67 follows from eq14} follows from Equation (\ref{eq14}) and Proposition \ref{jensen}, and \eqref{eq: wi-w} follows from Lemma \ref{lemma2} and Proposition \ref{jensen}.

\noindent\underline{Bound of $\|\nabla_{\bpi_i} \tilde{\ml}(\mathcal{P}^t)\|^2$}. By using the Lagrangian function, we denote    
\begin{align}\label{eq: bound of pi}
    \|\nabla_{\bpi_i} \tilde{\ml}(\mathcal{P}^t)\|^2&=\|\bw_i^t-\bw^t\|^2\notag\\
    &\leq\frac{2\rho^2+4\lambda^2\alpha_i^2}{\rho^2}\|\bw_i^t-\bw_i^{t+1}\|^2+\frac{4\lambda^2\alpha_i^2}{\rho^2}\|\btheta_i^{t+1}-\btheta_i^t\|^2.    
\end{align}
where \eqref{eq: bound of pi} follows from \eqref{eq:66 follows from lemma 2} and (\ref{eq: wi-w}). Next, by substituting (\ref{eq:bound of w}), (\ref{eq: gradient of theta}), (\ref{eq: wi-w}), and (\ref{eq: bound of pi}) into (\ref{eq: L(P)}), we obtain
\begin{align}\label{eq:48}
    &\|\nabla \tilde{\ml}(\mathcal{P}^t)\|^2=\|\nabla_{\bw} \tilde{\ml}(\mathcal{P}^t)\|^2+\sum_{i=1}^m\|\nabla_{\btheta_i} \tilde{\ml}(\mathcal{P}^t)\|^2+\|\nabla_{\bw_i} \tilde{\ml}(\bw_i^t)\|^2+\|\nabla_{\bpi_i} \tilde{\ml}(\bpi_i^t)\|^2\notag\\
    &\leq\sum_{i=1}^m\left(4\lambda^2\alpha_i^2(1+\frac{1}{\rho^2})+2\alpha_i^2(L^2+2\lambda^2)\right)\|\btheta_i^{t+1}-\btheta_i^{t}\|^2+\left(4\lambda^2\alpha_i^2(1+\frac{1}{\rho^2})+2\rho^2+2\right)\|\bw_i^{t+1}-\bw_i^{t}\|^2+2\epsilon_i^{t+1}\notag\\
    &\leq D_2\sum_{i=1}^m\left(\|\btheta_i^{t+1}-\btheta_i^{t}\|^2+\|\bw_i^{t+1}-\bw_i^{t}\|^2+\|\bw^{t+1}-\bw^{t}\|^2+\epsilon_{i}^{t+1}\right)
\end{align}
where $D_2:=\max_i\{4\lambda^2\alpha_{i}^2(1+\frac{1}{\rho^2})+2\alpha_{i}^2(L^2+2\lambda^2)+2\rho^2+2\}$.
\end{proof}

\subsection{Proof of Theorem \ref{theorem: objective function values convergence}} \label{Proof of Theorem 1}
\begin{lemma}\label{lemma: non increasing results}
    Let $\mathcal{P}^t$ be the sequence generated by Algorithm \ref{alg: flame on server side}, then the following results hold under Assumption \ref{assmption lipschitz}.
    \begin{itemize}
        \item[a)] Sequences $\{\tilde{L}(\mathcal{P}^t)\}$ is non-increasing.
        \item[b)] $\tilde{L}(\mathcal{P}^t)\geq f(\Theta^t,W^t)\geq f^*\geq-\infty$ for any $t\geq 1$.
        \item[c)] For any $i\in[m]$, the limits of the following terms are zero,
        \begin{align}\label{eq: 71 in lemma 9}
            (\epsilon_i^{t+1},\btheta_i^{t+1}-\btheta_i^{t},\bw^{t+1}-\bw^{t},\bw_i^{t+1}-\bw_i^{t},\bw_i^{t+1}-\bw^{t},\bpi_i^{t+1}-\bpi_i^{t})\to 0.
        \end{align}
    \end{itemize}
\end{lemma}
\begin{proof}
a) Following from Lemma \ref{lemma1}, we have
\begin{align}
    \tilde{\ml}(\mathcal{P}^{t})-\tilde{\ml}(\mathcal{P}^{t+1})\geq\mathcal{D}_1\sum_{i=1}^m(\|\bw^{t+1}-\bw^{t}\|^2+\|\bw^{t+1}_i-\bw_i^{t}\|^2+\|\btheta_i^{t+1}-\btheta_i^t\|^2)\geq 0,
\end{align}
which implies $\{\ml(\mathcal{P}^t)\}$ is non-increasing.

b) We consider bounding the following term:
\begin{align}
    \tilde{\ml}(\mathcal{P}^t)-f(\Theta^t,\bw^t)&=\sum_{i=1}^m \Bigl(\frac{\lambda\alpha_i}{2}\bigl(\|\btheta_i^t-\bw_i^t\|^2-\|\btheta_i^t-\bw^t\|^2\bigr)+\langle\bpi_i^t,\bw_i^t-\bw^t\rangle+\frac{\rho}{2}\|\bw_i^t-\bw^t\|^2+\iota_i\epsilon_i^{t}\Bigr)\notag\\
    &=\sum_{i=1}^m \Bigl(\frac{\lambda\alpha_i}{2}\langle\bw^t+\bw_i^t-2\btheta_i^t,\bw_i^t-\bw^t\rangle+\langle\lambda\alpha_i(\btheta_i^t-\bw_i^t),\bw_i^t-\bw^t\rangle+\frac{\rho}{2}\|\bw_i^t-\bw^t\|^2+\iota_i\epsilon_i^{t}\Bigr)\label{eq74 follows from lemma 3}\\
    &=\sum_{i=1}^m\Bigl(\frac{\rho-\lambda\alpha_i}{2}\|\bw_i^t-\bw^t\|^2+\iota_i\epsilon_i^{t}\Bigr)\geq0\label{eq: 75 follows from rho>lambda alpha},
\end{align}
where \eqref{eq74 follows from lemma 3} follows from Lemma \ref{lemma2} and \eqref{eq: 75 follows from rho>lambda alpha} follows from the fact that $\rho\geq\lambda\alpha_i$.

c) Following from the fact that $\epsilon_i^{t+1}=\upsilon_i\epsilon_i^t$ and $0<\upsilon_i<1$, then it can be easily indicated that $\epsilon_i^{t+1}\to 0$.
Next, under Lemma \ref{lemma1} and $\tilde{L}^t>-\infty$, we obtain
\begin{align}
    D_1\sum_{t=0}^{\infty}\Delta\Gamma^{t+1}\leq \sum_{t=0}^{\infty} \Bigl(\tilde{\ml}(\mathcal{P}^{t})-\tilde{\ml}(\mathcal{P}^{t+1})\Bigr)\leq\tilde{\ml}(\mathcal{P}^{0})-f^*\leq+\infty,
\end{align}
which implies the sequences $\|\bw^{t+1}-\bw^{t}\|\to0$, $\|\bw_{i}^{t+1}-\bw_{i}^{t}\|\to0$, and $\|\btheta_{i}^{t+1}-\btheta_{i}^{t}\|\to0$. Next, we consider
\begin{align}
    \|\bpi_i^{t+1}-\bpi_i^{t}\|&=\lambda\alpha_i\|\btheta_i^{t+1}-\btheta_i^{t}+\bw_i^{t}-\bw_i^{t+1}\|\notag\\
    &\leq 2\lambda\alpha_i(\|\btheta_i^{t+1}-\btheta_i^{t}\|+\|\bw_i^{t}-\bw_i^{t+1}\|),
\end{align}
which implies $\|\bpi_i^{t+1}-\bpi_i^{t}\|\to 0$. Furthermore, following from the fact that $\boldsymbol{\pi}_i^{t+1}=\boldsymbol{\pi}_i^t+\rho(\bw_i^{t+1}-\bw^{t})$, we have
\begin{align}
    \|\bw_i^{t+1}-\bw^t\|=\frac{1}{\rho}\|\bpi_i^{t+1}-\bpi_i^{t}\|\to 0,
\end{align}
which completes the proof.

\end{proof}

\begin{proof}[Proof of Theorem \ref{theorem: objective function values convergence}]
a) Based on Lemma \ref{lemma: non increasing results}, we have $\tilde{\ml}(\mathcal{P}^1)\geq f(\Theta^t,\bw^t)=\sum_{i=1}^{m}\alpha_i \Bigl(f_i(\btheta_i)+\frac{\lambda}{2}||\btheta_i-\bw||^2\Bigr)$. 
Along with the coercive property of $f_i$, we can infer that the sequences $\{\Theta^t\}$ and $\{\bw^t\}$ are bounded. Following from the fact that $\|\bw_i^{t+1}-\bw^{t}\|\to 0$, we can infer that the sequence $\{\bw_i^t\}$ is bounded. Finally, we consider bounding $\bpi_i^t$ as
\begin{align}
    \|\bpi_{i}^{t}\|=\lambda\alpha_i\|\btheta_i^{t}-\bw_i^{t}\|\leq\lambda\alpha_i(\|\btheta_i^{t}\|+\|\bw_i^{t}\|)\leq+\infty,
\end{align}
which completes the proof.

b) Lemma \ref{lemma: non increasing results} indicates that $\{\tilde{\ml}(\mathcal{P}^t)\}$ is non-increasing and lower bounded. Then it is not hard to indicate that $\tilde{\ml}(\mathcal{P}^t)\to \ml(\mathcal{P}^t)$ due to the fact that $\epsilon_i^{t+1}\to 0$. Next, we consider
\begin{equation}
\begin{split}
    \ml(\mathcal{P}^{t+1})-\tilde{f}(\Theta^{t+1},W^{t+1})
    =\sum_{i=1}^m \langle\bpi_i^{t+1},\bw_i^{t+1}-\bw^{t+1}\rangle+\frac{\rho}{2}\|\bw_i^{t+1}-\bw^{t+1}\|^2\to 0,
\end{split}
\end{equation}
\begin{equation}
\begin{split}
    \ml(\mathcal{P}^{t+1})-f(\Theta^{t+1},\bw^{t+1})
    =\sum_{i=1}^m\Bigl(\frac{\rho-\lambda\alpha_i}{2}\|\bw_i^{t+1}-\bw^{t+1}\|^2\Bigr)\to 0,
\end{split}
\end{equation}
which completes the proof.

c) We begin by defining
\begin{align}
    g_i(\btheta_i,\bw_i):&=\nabla_{\btheta_i}\ml(\btheta_i,\bw_i,\bpi_i,\bw)\notag\\
    &=\alpha_i\Bigl(\nabla f_i(\btheta_i)+\lambda(\btheta_i-\bw_i)\Bigr).
\end{align}

Then it holds that $\|g_i(\btheta_i^{t+1},\bw_i^t)\|^2\leq\epsilon^{t+1}_i$. Next, we consider the following terms
\begin{align}
    \|\nabla_{\Theta}\tilde{f}(\Theta^{t+1},W^{t+1})\|&=\|\sum_{i=1}^m g_i(\btheta_i^{t+1},\bw_i^t)+\lambda \alpha_i(\bw_i^{t}-\bw_i^{t+1})\|\notag\\
    &\leq\sum_{i=1}^m\|g_i(\btheta_i^{t+1},\bw_i^t)\|+\lambda\alpha_i \|(\bw_i^{t}-\bw_i^{t+1})\|\to 0,\label{eq: 82 follows from lemma 9 c}\\
    \|\nabla_{\Theta}f(\Theta^{t+1},W^{t+1})\|&=\|\sum_{i=1}^m g_i(\btheta_i^{t+1},\bw_i^t)+\lambda\alpha_i (\bw_i^{t}-\bw^{t+1})\|\notag\\
    &\leq\sum_{i=1}^m\|g_i(\btheta_i^{t+1},\bw_i^t)\|+\lambda\alpha_i \|(\bw_i^{t}-\bw^{t+1})\|\to 0,\label{eq: 83 follows from eq 71 in lemma 9}\\
    \|\nabla_{W}\tilde{f}(\Theta^{t+1},W^{t+1})\|&=\|\sum_{i=1}^m \lambda\alpha_i(\bw_i^{t+1}-\btheta_i^{t+1})\|\notag\\
    &=\|\sum_{i=1}^m\bpi_i^{t+1}\|=\|\rho\sum_{i=1}^m (\bw^{t+1}-\bw_i^{t+1})\|\label{eq:60}\\
    &\leq \|\rho\sum_{i=1}^m (\bw^{t+1}-\bw^t)\|+\|\rho\sum_{i=1}^m(\bw^t-\bw_i^{t+1})\|\to 0,\label{eq:61}\\
    \|\nabla_{\bw}f(\Theta^{t+1},W^{t+1})\|&=\|\sum_{i=1}^m \lambda\alpha_i(\bw^{t+1}-\btheta_i^{t+1})\|\notag\\
    &=\|\sum_{i=1}^m\bpi_i^{t+1}+\bw^{t+1}-\bw_i^{t+1}\|\to 0\label{eq: 86 follows from 84 and 85},
\end{align}
where \eqref{eq: 82 follows from lemma 9 c} and \eqref{eq: 83 follows from eq 71 in lemma 9} follow from \eqref{eq: 71 in lemma 9} in Lemma \ref{lemma: non increasing results}, \eqref{eq:60} follows from Lemma \ref{lemma2} and \eqref{eq15}, \eqref{eq:60} follows from \eqref{eq: 71 in lemma 9} in Lemma \ref{lemma: non increasing results}, and \eqref{eq: 86 follows from 84 and 85} follows from \eqref{eq:60} and \eqref{eq:61}.
\end{proof}

\subsection{Proof of Theorem \ref{theorem: sequences convergence}}\label{Proof of Theorem 1.5}
\begin{proof}
a) Let $(\Theta^{\infty}, W^{\infty}, \Pi^{\infty}, \bw^{\infty})$ be any accumulating point of the sequence $\{(\Theta^t,W^t,\Pi^t,\bw^t)\}$. Then we consider 
\begin{align}\label{eq:stationary point condition 1}
    g_i(\btheta_i^{\infty},\bw_i^{\infty})=\alpha_i\Bigl(\nabla f_i(\btheta_i^{\infty})+\lambda(\btheta_i^{\infty}-\bw_i^{\infty})\Bigr)=0.
\end{align}

Based on \eqref{eq:60}, we have
\begin{align}\label{eq:stationary point condition 2}
    \sum_{i=1}^m \bpi_i^{\infty} = 0.
\end{align}

According to Lemma \ref{lemma2}, we have
\begin{align}\label{eq:stationary point condition 3}
    \lambda\alpha_i(\btheta_i^{\infty}-\bw_i^{\infty})-\bpi_{i}^{\infty} = 0
\end{align}

Following from Lemma \ref{lemma: non increasing results}, we have
\begin{align}\label{eq:stationary point condition 4}
    \bw_i^{\infty}-\bw^{\infty} =0.
\end{align}

Combining \eqref{eq:stationary point condition 1}, \eqref{eq:stationary point condition 2}, \eqref{eq:stationary point condition 3}, and \eqref{eq:stationary point condition 4} yields the stationary point condition in \eqref{eq: optimal condition of (7)}. Moreover, it can be easily inferred that $(\Theta^{\infty}, W^{\infty})$ is a stationary point of Problem \eqref{eq: bi-variable problem}, which completes the proof.

b) Following from Proposition \ref{proposition: kl property of L}, it holds that $\tilde{\ml}$ is a \kl function.
According to Definition \ref{definition: kl property}, we have
\begin{align}
    \psi'\Bigl(\tilde{\ml}(\mathcal{P}^t)-\tilde{\ml}(\mathcal{P}^{\infty})\Bigr)\text{dist}\Bigl(0,\,\partial\tilde{\ml}(\mathcal{P}^t)\Bigr)\geq 1,
\end{align}
which implies 
\begin{align}
    \psi'\Bigl(\tilde{\ml}(\mathcal{P}^t)-\tilde{\ml}(\mathcal{P}^{\infty})\Bigr)\geq\frac{1}{\text{dist}\Bigl(0,\,\partial\tilde{\ml}(\mathcal{P}^t)\Bigr)}\geq \frac{1}{\sqrt{\mathcal{D}_2(\Delta\Gamma^{t+1}+\varepsilon^{t+1})}}.
\end{align}

Since $\psi$ is a concave function, we have
\begin{align}
    \psi\Bigl(\tilde{\ml}(\mathcal{P}^{t+1})-\tilde{\ml}(\mathcal{P}^{\infty})\Bigr)-\psi\Bigl(\tilde{\ml}(\mathcal{P}^t)-\tilde{\ml}(\mathcal{P}^{\infty})\Bigr)&\leq\psi'\Bigl(\tilde{\ml}(\mathcal{P}^t)-\tilde{\ml}(\mathcal{P}^{\infty})\Bigr)\Bigl(\tilde{\ml}(\mathcal{P}^{t+1})-\tilde{\ml}(\mathcal{P}^t)\Bigr)\notag\\
    &\leq-\frac{\mathcal{D}_1\Delta\Gamma^{t+1}}{\sqrt{\mathcal{D}_2(\Delta\Gamma^{t+1}+\varepsilon^{t+1})}},
\end{align}
which implies 
\begin{align}
    \sqrt{\Delta\Gamma^{t+1}}&\leq \sqrt{\frac{\sqrt{\mathcal{D}_2(\Delta\Gamma^{t+1}+\varepsilon^{t+1})}}{\mathcal{D}_1}\Biggl(\psi\Bigl(\tilde{\ml}(\mathcal{P}^t)-\tilde{\ml}(\mathcal{P}^{\infty})\Bigr)-\psi\Bigl(\tilde{\ml}(\mathcal{P}^{t+1})-\tilde{\ml}(\mathcal{P}^{\infty})\Bigr)\Biggr)}\notag\\
    &\leq\frac{1}{2}\frac{\sqrt{\mathcal{D}_2(\Delta\Gamma^{t+1}+\varepsilon^{t+1})}}{\mathcal{D}_1}+\frac{1}{2}\Biggl(\psi\Bigl(\tilde{\ml}(\mathcal{P}^t)-\tilde{\ml}(\mathcal{P}^{\infty})\Bigr)-\psi\Bigl(\tilde{\ml}(\mathcal{P}^{t+1})-\tilde{\ml}(\mathcal{P}^{\infty})\Bigr)\Biggr)\\
    &\leq\frac{1}{2}\frac{\sqrt{\mathcal{D}_2\Delta\Gamma^{t+1}}}{\mathcal{D}_1}+\frac{1}{2}\frac{\sqrt{\mathcal{D}_2\varepsilon^{t+1}}}{\mathcal{D}_1}+\frac{1}{2}\Biggl(\psi\Bigl(\tilde{\ml}(\mathcal{P}^t)-\tilde{\ml}(\mathcal{P}^{\infty})\Bigr)-\psi\Bigl(\tilde{\ml}(\mathcal{P}^{t+1})-\tilde{\ml}(\mathcal{P}^{\infty})\Bigr)\Biggr)\label{eq: 79}.
\end{align}
Summing over \eqref{eq: 79} from 0 to $+\infty$ yields
\begin{align}
    \sum_{t=0}^{+\infty}\sqrt{\Delta\Gamma^{t+1}}\leq\sum_{t=0}^{+\infty}\frac{1}{2}\frac{\sqrt{\mathcal{D}_2\Delta\Gamma^{t+1}}}{\mathcal{D}_1}+\frac{1}{2}\frac{\sqrt{\mathcal{D}_2\varepsilon^{t+1}}}{\mathcal{D}_1}+ \frac{1}{2}\psi\Bigl(\tilde{\ml}(\mathcal{P}^0)-\tilde{\ml}(\mathcal{P}^{\infty}\Bigr),
\end{align}
which implies 
\begin{align}
    \sum_{t=0}^{+\infty}\sqrt{\Delta\Gamma^{t+1}}\leq \frac{D_1}{2D_1-\sqrt{D_2}}\psi\Bigl(\tilde{\ml}(\mathcal{P}^0)-\tilde{\ml}(\mathcal{P}^{\infty}\Bigr)+\sum_{t=1}^{+\infty}\frac{\sqrt{D_2\varepsilon^{t+1}}}{2D_1-\sqrt{D_2}}\leq \infty
\end{align}
Following from the definition of $\Delta\Gamma^{t+1}$, we have
\begin{align}\label{eq: 80}
    \Delta\Gamma^{t+1}=\sum_{i=1}^m(\|\bw^{t+1}-\bw^{t}\|^2+\|\bw^{t+1}_i-\bw_i^{t}\|^2+\|\btheta_i^{t+1}-\btheta_i^t\|^2)<\infty.
\end{align}
Then we can infer from \eqref{eq: 80} that the sequence $\{\Theta^t,\bw^t,W^t\}$ is convergent. Next, by Lemma \ref{lemma2}, we can infer that $\{\Pi^t\}$ is convergent. Overall, the sequence $\{\mathcal{P}^t\}$ is convergent.
\end{proof}

\subsection{Proof of Theorem \ref{theorem1}}\label{proof of theorem 2}

\begin{proof}
Following from Lemma \ref{lemma1} and Lemma \ref{lemma: relative error}, it holds that
\begin{align}\label{eq: 99 to telescope}
    \|\partial \tilde{L}(\mathcal{P}^t)\|^2\leq D_2(\Delta\Gamma^{t+1}+\varepsilon^{t+1})\leq\frac{D_2}{D_1}\Bigl(\tilde{\ml}(\mathcal{P}^{t})-\tilde{\ml}(\mathcal{P}^{t+1})\Bigr)+D_2\varepsilon^{t+1}
\end{align}
Telescoping \eqref{eq: 99 to telescope} and dividing $T$, we obtain
\begin{align}
    \frac{1}{T}\sum_{t=0}^{T-1}\|\partial \tilde{L}(\mathcal{P}^t)\|^2&\leq\frac{D_2}{D_1 T}\Bigl(\tilde{\ml}(\mathcal{P}^{0})-\tilde{\ml}(\mathcal{P}^{T})\Bigr)+D_2\varepsilon^{1}\notag\\
    &\leq\frac{D_2}{D_1 T}\Bigl(\tilde{\ml}(\mathcal{P}^{0})-f^*\Bigr)+D_2\varepsilon^{1},
\end{align}
which completes the proof.
\end{proof}

\subsection{Proof of Theorem \ref{theorem: convergence rate based on kl property}} \label{Proof of Theorem 4}
\begin{proof}
Given the desingularizing function $\psi(x)=\frac{\sqrt{c}}{1-\tau}x^{1-\tau}$, let $\varepsilon^{t+1}:=\sum_{i=1}^m \epsilon_i^{t+1}$ we have
\begin{align}
    1&\leq\psi'\Bigl(\tilde{\ml}(\mathcal{P}^{t+1})-\tilde{\ml}(\mathcal{P}^{\infty})\Bigr)^2 \text{dist}\Bigl(0,\,\partial\tilde{\ml}(\mathcal{P}^t)\Bigr)^2\\
    &\leq c \Bigl(\tilde{\ml}(\mathcal{P}^{t+1})-\tilde{\ml}(\mathcal{P}^{\infty})\Bigr)^{-2\tau} \mathcal{D}_2(\Delta\Gamma^{t+1}+\varepsilon^{t+1})\label{eq:103 follows from lemma 1}\\
    &\leq  c \Bigl(\tilde{\ml}(\mathcal{P}^{t+1})-\tilde{\ml}(\mathcal{P}^{\infty})\Bigr)^{-2\tau}D_2\Bigl(\frac{\tilde{\ml}(\mathcal{P}^t)-\tilde{\ml}(\mathcal{P}^{t+1})}{D_1}+\varepsilon^{t+1}\Bigr)\label{eq: 104 follows from lemma 2},
\end{align}
where \eqref{eq:103 follows from lemma 1} follows from Lemma \ref{lemma1} and \eqref{eq: 104 follows from lemma 2} follows from Lemma \ref{lemma: relative error}. Then we can obtain
\begin{align}\label{eq: kl final}
    \frac{D_1}{cD_2} \Bigl(\tilde{\ml}(\mathcal{P}^{t+1})-\tilde{\ml}(\mathcal{P}^{\infty})\Bigr)^{2\tau}\leq\Bigl(\tilde{\ml}(\mathcal{P}^t)-\tilde{\ml}(\mathcal{P}^{\infty})\Bigr)-\Bigl(\tilde{\ml}(\mathcal{P}^{t+1})-\tilde{\ml}(\mathcal{P}^{\infty})\Bigr)+D_1 \varepsilon^{t+1}.
\end{align}
Next, we consider the three cases with respect to $\tau$.
\begin{itemize}
    \item If $\tau=0$, then according to \eqref{eq: kl final}, it holds that 
\begin{align}\label{eq:conflict}
    \Bigl(\tilde{\ml}(\mathcal{P}^t)-\tilde{\ml}(\mathcal{P}^{\infty})\Bigr)-\Bigl(\tilde{\ml}(\mathcal{P}^{t+1})-\tilde{\ml}(\mathcal{P}^{\infty})\Bigr)+D_1 \varepsilon^{t+1}\geq \frac{D_1}{c D_2}.
\end{align}

However, $\tilde{\ml}(\mathcal{P}^t)-\tilde{\ml}(\mathcal{P}^{\infty})\to 0$ and $\varepsilon^{t+1}\to 0$, which are conflict with \eqref{eq:conflict}. Therefore, we can deduce that there must exist $t\geq t_1>0$ such that $\tilde{\ml}(\mathcal{P}^t)-\tilde{\ml}(\mathcal{P}^{\infty})= 0$.
    \item If $\tau\in (0,1/2]$, there must exist $t\geq t_2>0$, such that $0\leq\tilde{\ml}(\mathcal{P}^{t+1})-\tilde{\ml}(\mathcal{P}^{\infty})\leq 1$. Then according to \eqref{eq: kl final}, we have
\begin{align*}
    \Bigl(\tilde{\ml}(\mathcal{P}^t)-\tilde{\ml}(\mathcal{P}^{\infty})\Bigr)-\Bigl(\tilde{\ml}(\mathcal{P}^{t+1})-\tilde{\ml}(\mathcal{P}^{\infty})\Bigr)+D_1 \varepsilon^{t+1}\geq\frac{D_1}{c D_2} \Bigl(\tilde{\ml}(\mathcal{P}^{t+1})-\tilde{\ml}(\mathcal{P}^{\infty})\Bigr)^{2\tau}\geq \frac{D_1}{c D_2} \Bigl(\tilde{\ml}(\mathcal{P}^{t+1})-\tilde{\ml}(\mathcal{P}^{\infty})\Bigr),
\end{align*}
which implies 
\begin{align}
    \tilde{\ml}(\mathcal{P}^{t+1})-\tilde{\ml}(\mathcal{P}^{\infty})&\leq\frac{cD_2}{D_1+cD_2}\Bigl(\tilde{\ml}(\mathcal{P}^t)-\tilde{\ml}(\mathcal{P}^{\infty})+D_1 \varepsilon^{t+1}\Bigr)\\
    &\leq\Bigl(\frac{cD_2}{D_1+cD_2}\Bigr)^2\Bigl(\tilde{\ml}(\mathcal{P}^{t-1})-\tilde{\ml}(\mathcal{P}^{\infty})+D_1 \varepsilon^{t+1}\Bigr)\\
    &\leq\Bigl(\frac{cD_2}{D_1+cD_2}\Bigr)^{t-t_2+1}\Bigl(\tilde{\ml}(\mathcal{P}^{t_2})-\tilde{\ml}(\mathcal{P}^{\infty})+D_1 \varepsilon^{t+1}\Bigr)\\
    &\leq\Bigl(\frac{cD_2}{D_1+cD_2}\Bigr)^{t-t_2+1}\Bigl(\tilde{\ml}(\mathcal{P}^{t_2})-f^*+D_1 \varepsilon^{t+1}\Bigr).
\end{align}
\item If $\tau\in(1/2,1)$, we define a function $\phi(z):=\frac{CD_2}{D_1(1-2\tau)}z^{1-2\tau}$. Let $1>\frac{1}{R}>\min_i{\nu_i}>0$ be a constant, we consider two cases, if $\Bigl(\tilde{\ml}(\mathcal{P}^t)-\tilde{\ml}(\mathcal{P}^{\infty})\Bigr)^{-2\tau}\geq\Bigl(\tilde{\ml}(\mathcal{P}^{t+1})-\tilde{\ml}(\mathcal{P}^{\infty})\Bigr)^{-2\tau}/R$, we obtain
\begin{align}\label{eq: geq 0}
    &\phi\Bigl(\tilde{\ml}(\mathcal{P}^t)-\tilde{\ml}(\mathcal{P}^{\infty})\Bigr)-\phi\Bigl(\tilde{\ml}(\mathcal{P}^{t+1})-\tilde{\ml}(\mathcal{P}^{\infty})\Bigr)\\
    =&\int_{\tilde{\ml}(\mathcal{P}^{t+1})-\tilde{\ml}(\mathcal{P}^{\infty})}^{\tilde{\ml}(\mathcal{P}^t)-\tilde{\ml}(\mathcal{P}^{\infty})}\phi'(z)dz=\int_{\tilde{\ml}(\mathcal{P}^{t+1})-\tilde{\ml}(\mathcal{P}^{\infty})}^{\tilde{\ml}(\mathcal{P}^t)-\tilde{\ml}(\mathcal{P}^{\infty})}\frac{cD_2}{D_1}z^{-2\tau}dz\\
    \geq&\frac{cD_2}{D_1}\Bigl(\tilde{\ml}(\mathcal{P}^t)-\tilde{\ml}(\mathcal{P}^{t+1})\Bigr)\Bigl(\tilde{\ml}(\mathcal{P}^t)-\tilde{\ml}(\mathcal{P}^{\infty})\Bigr)^{-2\tau}\\
    \geq&\frac{cD_2}{RD_1}\Bigl(\tilde{\ml}(\mathcal{P}^t)-\tilde{\ml}(\mathcal{P}^{t+1})\Bigr)\Bigl(\tilde{\ml}(\mathcal{P}^{t+1})-\tilde{\ml}(\mathcal{P}^{\infty})\Bigr)^{-2\tau}\\
    =&\frac{cD_2}{RD_1}\Bigl(\tilde{\ml}(\mathcal{P}^t)-\tilde{\ml}(\mathcal{P}^{t+1})+D_1 \varepsilon^{t+1}\Bigr)\Bigl(\tilde{\ml}(\mathcal{P}^{t+1})-\tilde{\ml}(\mathcal{P}^{\infty})\Bigr)^{-2\tau}-\frac{cD_2}{RD_1}D_1 \varepsilon^{t+1}\Bigl(\tilde{\ml}(\mathcal{P}^{t+1})-\tilde{\ml}(\mathcal{P}^{\infty})\Bigr)^{-2\tau}\\
    =&\frac{1}{R}-\frac{cD_2}{R}\frac{\varepsilon^{t+1}}{\bigl(\tilde{\ml}(\mathcal{P}^{t+1})-\tilde{\ml}(\mathcal{P}^{\infty})\bigr)^{2\tau}}.
\end{align}

\begin{itemize}
    \item[a)]
If $\varepsilon^{t+1}$ is a lower-order infinitesimal of $\Bigl(\tilde{\ml}(\mathcal{P}^{t+1})-\tilde{\ml}(\mathcal{P}^{\infty})\Bigr)^{2\tau}$, then we can obtain
\begin{align}\label{eq: 116 lower-order infinitesimal}
    \lim_{t\to +\infty}\frac{\varepsilon^{t+1}}{\bigl(\tilde{\ml}(\mathcal{P}^{t+1})-\tilde{\ml}(\mathcal{P}^{\infty})\bigr)^{2\tau}}= \lim_{t\to +\infty}\frac{\sum_{i=1}^m\nu_i\epsilon_i^{t+1}}{\bigl(\tilde{\ml}(\mathcal{P}^{t+1})-\tilde{\ml}(\mathcal{P}^{\infty})\bigr)^{2\tau}}=\infty.
\end{align}
Based on \eqref{eq: 116 lower-order infinitesimal}, we can derive the convergence rate of $\bigl(\tilde{\ml}(\mathcal{P}^{t+1})-\tilde{\ml}(\mathcal{P}^{\infty})\bigr)^{2\tau}\to 0$ is faster than $\sum_{i=1}^m\nu_i\epsilon_i^{t+1}\to 0$, which implies 
\begin{align}\label{eq:117 if confilit with our assumption}
    \bigl(\tilde{\ml}(\mathcal{P}^{t+1})-\tilde{\ml}(\mathcal{P}^{\infty})\bigr)^{2\tau}<\min_i \nu_i \times \bigl(\tilde{\ml}(\mathcal{P}^{t})-\tilde{\ml}(\mathcal{P}^{\infty})\bigr)^{2\tau}<\frac{1}{R} \bigl(\tilde{\ml}(\mathcal{P}^{t})-\tilde{\ml}(\mathcal{P}^{\infty})\bigr)^{2\tau},
\end{align}
Note that \eqref{eq:117 if confilit with our assumption} is conflict with $\Bigl(\tilde{\ml}(\mathcal{P}^t)-\tilde{\ml}(\mathcal{P}^{\infty})\Bigr)^{-2\tau}\geq\Bigl(\tilde{\ml}(\mathcal{P}^{t+1})-\tilde{\ml}(\mathcal{P}^{\infty})\Bigr)^{-2\tau}/R$, then $\varepsilon^{t+1}$ cannot be a lower-order infinitesimal of $\Bigl(\tilde{\ml}(\mathcal{P}^{t+1})-\tilde{\ml}(\mathcal{P}^{\infty})\Bigr)^{2\tau}$.
\item[b)]
If $\varepsilon^{t+1}$ is a higher-order infinitesimal of $\Bigl(\tilde{\ml}(\mathcal{P}^{t+1})-\tilde{\ml}(\mathcal{P}^{\infty})\Bigr)^{2\tau}$, then we obtain
\begin{align}
    \lim_{t\to +\infty}\frac{\varepsilon^{t+1}}{\bigl(\tilde{\ml}(\mathcal{P}^{t+1})-\tilde{\ml}(\mathcal{P}^{\infty})\bigr)^{2\tau}}=0.
\end{align}

If $\varepsilon^{t+1}$ and $\Bigl(\tilde{\ml}(\mathcal{P}^{t+1})-\tilde{\ml}(\mathcal{P}^{\infty})\Bigr)^{2\tau}$ are same-order infinitesimals, then we obtain
\begin{align}
    \lim_{t\to +\infty}\frac{\varepsilon^{t+1}}{\bigl(\tilde{\ml}(\mathcal{P}^{t+1})-\tilde{\ml}(\mathcal{P}^{\infty})\bigr)^{2\tau}}=\alpha\neq0.
\end{align}

Then there exists $Q>0$ such that $\frac{\varepsilon^{t+1}}{\bigl(\tilde{\ml}(\mathcal{P}^{t+1})-\tilde{\ml}(\mathcal{P}^{\infty})\bigr)^{2\tau}}\leq Q$. Therefore, we obtain that
\begin{align}
    \phi\Bigl(\tilde{\ml}(\mathcal{P}^t)-\tilde{\ml}(\mathcal{P}^{\infty})\Bigr)-\phi\Bigl(\tilde{\ml}(\mathcal{P}^{t+1})-\tilde{\ml}(\mathcal{P}^{\infty})\Bigr)\geq\frac{1}{R}-\frac{cD_2Q}{R}.
\end{align}
\end{itemize}


Consider another case, if $\Bigl(\tilde{\ml}(\mathcal{P}^t)-\tilde{\ml}(\mathcal{P}^{\infty})\Bigr)^{-2\tau}<\Bigl(\tilde{\ml}(\mathcal{P}^{t+1})-\tilde{\ml}(\mathcal{P}^{\infty})\Bigr)^{-2\tau}/R$, then we can infer
\begin{align}
    R^{\frac{1}{2\tau}}\Bigl(\tilde{\ml}(\mathcal{P}^{t+1})-\tilde{\ml}(\mathcal{P}^{\infty})\Bigr)< \tilde{\ml}(\mathcal{P}^{t})-\tilde{\ml}(\mathcal{P}^{\infty}),
\end{align}
which implies
\begin{align}
    \Bigl(\tilde{\ml}(\mathcal{P}^{t+1})-\tilde{\ml}(\mathcal{P}^{\infty})\Bigr)^{1-2\tau}&>\overline{R} \Bigl(\tilde{\ml}(\mathcal{P}^{t})-\tilde{\ml}(\mathcal{P}^{\infty})\Bigr)^{1-2\tau},\\
    \Bigl(\tilde{\ml}(\mathcal{P}^{t+1})-\tilde{\ml}(\mathcal{P}^{\infty})\Bigr)^{1-2\tau}-\Bigl(\tilde{\ml}(\mathcal{P}^{t})-\tilde{\ml}(\mathcal{P}^{\infty})\Bigr)^{1-2\tau}&>(\overline{R}-1)\Bigl(\tilde{\ml}(\mathcal{P}^{t})-\tilde{\ml}(\mathcal{P}^{\infty})\Bigr)^{1-2\tau},
\end{align}
where $\overline{R}=R^{\frac{2\tau-1}{\tau}}> 1$. Since $\overline{R}-1>0$ and $\tilde{\ml}(\mathcal{P}^{t})-\tilde{\ml}(\mathcal{P}^{\infty})\to 0^{+}$, then there exists $\overline{\mu}>0$ such that $(\overline{R}-1)\Bigl(\tilde{\ml}(\mathcal{P}^{t})-\tilde{\ml}(\mathcal{P}^{\infty})\Bigr)^{1-2\tau}>\overline{\mu}$ for all $t\geq t_3$. 
Therefore, we obtain that
\begin{align}
    \Bigl(\tilde{\ml}(\mathcal{P}^{t+1})-\tilde{\ml}(\mathcal{P}^{\infty})\Bigr)^{1-2\tau}-\Bigl(\tilde{\ml}(\mathcal{P}^{t})-\tilde{\ml}(\mathcal{P}^{\infty})\Bigr)^{1-2\tau}\geq\overline{\mu}>0
\end{align}
for all $t\geq t_3$. Then we bound $\phi\Bigl(\ml(\mathcal{P}^t)-\ml(\mathcal{P}^{\infty})\Bigr)-\phi\Bigl(\ml(\mathcal{P}^{t+1})-\ml(\mathcal{P}^{\infty})\Bigr)$ as 
\begin{align}\label{eq: geq 1}
    \phi\Bigl(\ml(\mathcal{P}^t)-\ml(\mathcal{P}^{\infty})\Bigr)-\phi\Bigl(\ml(\mathcal{P}^{t+1})-\ml(\mathcal{P}^{\infty})\Bigr)
    =&\frac{cD_2}{(1-2\tau)D_1}\Biggl(\Bigl(\ml(\mathcal{P}^t)-\ml(\mathcal{P}^{\infty})\Bigr)^{1-2\tau}-\Bigl(\ml(\mathcal{P}^{t+1})-\ml(\mathcal{P}^{\infty})\Bigr)^{1-2\tau}\Biggr)\notag\\
    \geq&\frac{c\overline{\mu}D_2}{(2\tau-1)D_1}.
\end{align}


If we define $\mu:=\min\{\frac{1}{R}-\frac{cD_2Q}{R},\frac{c\overline{\mu}D_2}{(2\tau-1)D_1}\}>0$, one can combine \eqref{eq: geq 0} and \eqref{eq: geq 1} to obtain that
\begin{align}\label{eq: to sum inequality}
    \phi\Bigl(\ml(\mathcal{P}^t)-\ml(\mathcal{P}^{\infty})\Bigr)-\phi\Bigl(\ml(\mathcal{P}^{t+1})-\ml(\mathcal{P}^{\infty})\Bigr)\geq\mu
\end{align}
for all $t\geq t_3$. By summing \eqref{eq: to sum inequality} from $t_3$ to some $t$ greater than $t_3$, we obtain
\begin{align}
    \phi\Bigl(\ml(\mathcal{P}^{t_3})-\ml(\mathcal{P}^{\infty})\Bigr)-\phi\Bigl(\ml(\mathcal{P}^{t+1})-\ml(\mathcal{P}^{\infty})\Bigr)\geq (t-t_3)\mu,
\end{align}
which implies 
\begin{align}
    \ml(\mathcal{P}^{t+1})-\ml(\mathcal{P}^{\infty})\leq\Bigl(\frac{(2\tau-1)\mu D_1}{cD_2}(t-t_3)\Bigr)^{\frac{1}{1-2\tau}},
\end{align}
which completes the proof.
\end{itemize}
\end{proof}


\newpage
\section{Robustness and Fairness} \label{appendix: Robustness and Fairness}
In this section, we are inspired by \cite{Lin2022Personalized} to consider employing an example of \textit{federated linear regression} to analyze the fairness and robustness of \flame. 
Suppose the truly personalized model on client $i$ is $\btheta_i$, each client possesses $N$ samples\footnote{For simplicity, we assume each client possesses an equal number of samples.}, and the covariate on client $i$ is $\{\bx_{i,j}\}_{j=1}^N$ with $\bx_{i, j}\in\mathbb{R}^d$ is fixed. The observations are generated by $y_{i,j}=\bx_{i,j}^{\top}\btheta_i+z_{i,j}$, where $z_{i,j}$ denotes an i.i.d. Gaussian noise with distribution $\mathcal{N}(0,\sigma^2)$. 
Then the loss on client $i$ is $f_i(\btheta_i)=\frac{1}{2N}\sum_{j=1}^N(y_{i,j}-\bx_{i,j}^{\top}\btheta_i)$.
\subsection{Solutions of \flame}\label{sec: Solutions of flame}
We consider to derive the solution of $\flame$ under the case of federated linear regression.
Let $X_i=(\bx_{i,1},\bx_{i,2},\cdots,\bx_{i,N})$ and $\by_i=(y_{i,1},y_{i,2},\cdots,y_{i,N})^{\top}$, then the observations can be rewritten as $\by_i=X_i\btheta_i + \bz_i$ with $\bz_i\sim\mathcal{N}(\boldsymbol{0},\sigma^2\biden_d)$, and the loss on client $i$ can be rewritten as $f_i(\btheta_i)=\frac{1}{2N}\|X_i\btheta_i-\by_i\|^2$. Suppose $X_i^{\top}X_i$ is invertible, then the estimator of $\btheta_i$ is 
\begin{align}\label{eq: 129 solution of linear regression}
    \hat{\btheta}_i=(X_i^{\top}X_i)^{-1}X_i^{\top}\by_i.
\end{align}
Recall the optimization problem of \flame is 
\begin{align}\label{eq: appendix optimization problem}
    \min_{\bw,\btheta_i}\Bigl\{f(\Theta,\bw):=\sum_{i=1}^{m}\alpha_i \Bigl(f_i(\btheta_i)+\frac{\lambda}{2}||\btheta_i-\bw||^2\Bigr)\Bigr\}
\end{align}
with optimal conditions
\begin{align}\label{eq: appendix optimal condition}
\left\{\begin{aligned}
\nabla f_i(\btheta_i^*)+\lambda(\btheta_i^*-\bw^*)& =0, & i \in[m], \\
\bw^* - \frac{1}{m}\sum_{i=1}^m\btheta_i^*& =0. &  \\
\end{aligned}\right.    
\end{align}
Substituting the loss function $f_i(\btheta_i)=\frac{1}{2N}\|X_i\btheta_i-\by_i\|^2$ of linear regression into \eqref{eq: appendix optimization problem} and \eqref{eq: appendix optimal condition} yields
\begin{align}
    f(\Theta,\bw)=\sum_{i=1}^{m}\alpha_i \Bigl(\frac{1}{2N}\|X_i\btheta_i-\by_i\|^2+\frac{\lambda}{2}||\btheta_i-\bw||^2\Bigr),
\end{align}
and
\begin{align}\label{eq: derive from eq 133}
\left\{\begin{aligned}
\frac{1}{N}X_i^{\top}(X_i\btheta_i^*-\by_i)+\lambda(\btheta_i^*-\bw^*)& =0, & i \in[m], \\
\bw^* - \sum_{i=1}^m\alpha_i\btheta_i^*& =0. &  \\
\end{aligned}\right.     
\end{align}
Without loss of generality, we let $\alpha_i=\frac{1}{m}$, then we can derive from \eqref{eq: derive from eq 133} that
\begin{align}\label{eq: solutions of flame appendix}
\left\{\begin{aligned}
\bw^*& =\Bigl(\biden_d-\frac{1}{m}\sum_{i=1}^m \lambda\Bigl(\frac{1}{N}X_i^{\top}X_i+\lambda \biden_d\Bigr)^{-1}\Bigr)^{-1}\frac{1}{m}\sum_{i=1}^m \lambda\Bigl(\frac{1}{N}X_i^{\top}X_i+\lambda \biden_d\Bigr)^{-1}\frac{1}{N}X_i^{\top}X_i\hat{\btheta}_i, & \\
\btheta_i^*& =\Bigl(\frac{1}{N}X_i^{\top}X_i+\lambda \biden_d\Bigr)^{-1}\Bigl(\frac{1}{N}X_i^{\top}X_i\hat{\btheta}_i+\lambda\bw^*\Bigr),\quad i \in[m]. &  \\
\end{aligned}\right.     
\end{align}
However, for the general $X_i$, obtaining a concise expression for $\bw^{*}$ and $\btheta_i^{*}$ is challenging. To streamline the calculations, we assume $X_i^{\top}X_i=N b_i \biden_d$. Consequently, the solution in \eqref{eq: solutions of flame appendix} can be simplified as 
\begin{align}\label{eq: simplified solutions of flame appendix}
\left\{\begin{aligned}
\bw^*& =\frac{\sum_{i=1}^m b_i\hat{\btheta}_i/(b_i+\lambda)}{\sum_{i=1}^m b_i/(b_i+\lambda)}, & \\
\btheta_i^*& =\frac{b_i\hat{\btheta}_i+\lambda\bw^*}{b_i+\lambda},\quad i \in[m]. &  \\
\end{aligned}\right.     
\end{align}
If we further assume $b_i=b$\footnote{This assumption is reasonable since datasets are often normalized.}, then we obtain 
\begin{align}\label{eq: solutions of flame when b_i=b in appendix}
\left\{\begin{aligned}
\bw^*& =\frac{1}{m}\sum_{i=1}^m \hat{\btheta}_i, & \\
\btheta_i^*& =\frac{b\hat{\btheta}_i+\lambda\bw^*}{b+\lambda},\quad i \in[m]. &  \\
\end{aligned}\right.     
\end{align}
According to \cite{Lin2022Personalized}, we can obtain the solutions of \pfedme and \ditto are the same with \flame when $b_i=b,\,i\in[m]$.
\subsection{Test Loss}\label{secappendix: Test Loss}
Recall that the local dataset on client $i$ is $(X_i, \by_i)$, where $X_i$ is fixed and $\by_i$ follows a Gaussian distribution $\mathcal{N}(X_i\btheta_i, \sigma^2 \biden_d)$. Consequently, the data heterogeneity among clients arises solely from the heterogeneity of $\btheta_i$. Then based on \eqref{eq: 129 solution of linear regression} and \eqref{eq: solutions of flame when b_i=b in appendix}, we can derive the distribution of the \flame's solutions as follows,
\begin{align}\label{eq: 137 plug into in appendix}
    \bw^*\sim \mathcal{N}\Bigl(\overline{\btheta}, \frac{\sigma^2}{bmN}\biden_d\Bigr),\,\btheta_i^*\sim\mathcal{N}\Bigl(\frac{b\btheta_i+\lambda\overline{\btheta}}{b+\lambda}, \frac{(b^2+\frac{2b\lambda}{m})\frac{\sigma^2}{bN}+\frac{\lambda^2\sigma^2}{bmN}}{(b+\lambda)^2}\biden_d\Bigr)
\end{align}
where $\overline{\btheta}=\frac{1}{m}\sum_{i=1}^m\btheta_i$. Given that $X_i$ is fixed, we assume the test dataset is $(X_i,\by_i')$ where $\by_i'=X_i\btheta_i+\bz_i'$ with $\bz_i'\sim \mathcal{N}(\boldsymbol{0},\sigma^2\biden_d)$, which is independent of $\bz_i$. Next, according to \cite{Lin2022Personalized}, the test losses of the global and personalized models on client $i$ are defined as follows,
\begin{align}\label{eq: 138 be plugged in appendix}
    f_i(\bw^*)&:=\frac{1}{2N}\mathbb{E}\|X_i\bw^*-\by_i'\|^2=\frac{\sigma^2}{2}+\frac{b}{2}\trace(\var(\bw^*))+\frac{b}{2}\|\mathbb{E}\bw^*-\btheta_i\|^2               \\
    f_i(\btheta_i^*)&:=\frac{1}{2N}\mathbb{E}\|X_i\btheta_i^*-\by_i'\|^2=\frac{\sigma^2}{2}+\frac{b}{2}\trace(\var(\btheta_i^*))+\frac{b}{2}\|\mathbb{E}\btheta_i^*-\btheta_i\|^2,
\end{align}
and the average losses are defined as follows,
\begin{align}
    \frac{1}{m}\sum_{i=1}^m f_i(\bw^*)&=\frac{\sigma^2}{2}+\frac{b}{2}\trace(\var(\bw^*))+\frac{b}{2m}\sum_{i=1}^m\|\mathbb{E}\bw^*-\btheta_i\|^2,\label{eq: 139 be plugged in appendix}\\
    \frac{1}{m}\sum_{i=1}^m f_i(\btheta_i^*)&=\frac{\sigma^2}{2}+\frac{b}{2m}\sum_{i=1}^m\trace(\var(\btheta_i^*))+\frac{b}{2m}\sum_{i=1}^m\|\mathbb{E}\btheta_i^*-\btheta_i\|^2.\label{eq: 141 be plugged in appendix}
\end{align}
Substitute the mean and variance of $\bw^*$ and $\btheta_i^*$ in \eqref{eq: 137 plug into in appendix} into \eqref{eq: 139 be plugged in appendix} and \eqref{eq: 141 be plugged in appendix}, we obtain
\begin{align}
    \frac{1}{m}\sum_{i=1}^m f_i(\bw^*)&= \frac{\sigma^2}{2}+ \frac{\sigma^2 d}{2 m N} +\frac{b}{2m}  \sum_{i=1}^m\|\overline{\btheta}-\btheta_i\|^2,     \\  
    \frac{1}{m}\sum_{i=1}^m f_i(\btheta_i^*)&=\frac{\sigma^2}{2}+\frac{mb^2+2b\lambda+\lambda^2}{m(b+\lambda)^2}\cdot\frac{\sigma^2 d}{2N}+\frac{b\lambda^2}{2m(b+\lambda^2)}\sum_{i=1}^m\|\overline{\btheta}-\btheta_i\|^2.
\end{align}
\subsection{Robustness}\label{appendix: robustness}

In this subsection, we analyze the robustness of \flame against three types of Byzantine attacks. In the previous subsection, we concentrate solely on the exact solutions of \flame and overlook the algorithmic processes. To thoroughly assess robustness, it is essential to examine the algorithm's procedures. We consider a simplified scenario with an infinite number of local update steps, a single round of communication, and participation from all clients.
Throughout this subsection, we assume there are \(m_a\) malicious clients and \(m_b\) benign clients, with \(m_a + m_b = m\), and let \(\mathcal{S}_a\) denote the indices of malicious clients and \(\mathcal{S}_b\) denote the indices of benign clients. 

Before we analyze the impact of these attacks on the solutions of \flame and compare the average test losses on benign clients, we demonstrate that after one round of communication, \flame obtains inexact solutions that are different from the exact solutions as defined in Appendix \ref{sec: Solutions of flame}.
Consider the update step of $\bw_i$:
\begin{align}
    \bw_i^{1}=\frac{1}{\lambda\alpha_i+\rho}(\lambda\alpha_i\btheta_i^{1}+\rho\bw^{0}-\bpi_i^{0})
\end{align}
where $\btheta_i^{1}=\argmin_{\btheta_i}\Bigl\{f_i(\btheta_i)+\frac{\lambda}{2}\|\btheta_i-\bw_i^0\|\Bigr\}=\frac{b\hat{\btheta}_i+\lambda\bw_i^{0}}{b+\lambda}$\footnote{For the local update step of $\btheta_i$: $\btheta_i^{t,h+1}=\btheta_i^{t,h}-\eta(\nabla f_i(\btheta_i^{t,h},\xi)+\lambda(\btheta_i^{t,h}-\bw_i^t))$,
it is noteworthy that when the number of local update steps is infinite, it is reasonable to assume that we can obtain the exact solution of $\btheta_i$.}.
Then we can rewrite the update step of $\bw_i$ as follows
\begin{align}
    \bw_i^{1}=\frac{1}{\lambda\alpha_i+\rho}(\lambda\alpha_i \frac{b\hat{\btheta}_i+\lambda\bw_i^{0}}{b+\lambda} +\rho\bw^0-\bpi_i^0)
\end{align}

Suppose we initialize $\bw^0$, $\bpi^0$, and $\bw_i^0$ to \textbf{0}, we obtain $\bw_i^{1}=\frac{\lambda\alpha_i}{\lambda\alpha_i+\rho}\frac{b}{b+\lambda}\hat{\btheta}_i$. Then based on \eqref{eq14} and \eqref{eq15}, each client sends the update messages as $\bu_i=\frac{2\lambda\alpha_i}{\lambda\alpha_i+\rho}\frac{b}{b+\lambda}\hat{\btheta}_i$. Next, the server aggregates the update messages and updates the global model as $\bw^{\flame}=\frac{1}{m}\sum_{i=1}^m \frac{2\lambda\alpha_i}{\lambda\alpha_i+\rho}\frac{b}{b+\lambda}\hat{\btheta}_i$ to each client. Finally, each client solves the personalized model as $\btheta_i^\flame=\frac{b\hat{\btheta}_i+\lambda\bw^\flame}{b+\lambda}$.
\subsubsection{Same-value Attacks}
Consider the following two cases:
\begin{itemize}
    \item When client $i$ is benign, then it will send $\bu_i^{(be)}=\frac{2\lambda\alpha_i}{\lambda\alpha_i+\rho}\frac{b}{b+\lambda}\hat{\btheta}_i$ to the server. Recall that $\hat{\btheta}_i=(X_i^{\top}X_i)^{-1}X_i^{\top}\by_i\sim\mathcal{N}(\btheta_i,\frac{\sigma^2}{bN}\biden_d)$. Then we can derive that $\bu_i^{(be)}\sim\mathcal{N}(\frac{2\lambda\alpha_i}{\lambda\alpha_i+\rho}\frac{b}{b+\lambda}\btheta_i,(\frac{2\lambda\alpha_i}{\lambda\alpha_i+\rho}\frac{b}{b+\lambda})^2\cdot\frac{\sigma^2}{bN}\biden_d)$. If we set $\rho\geq\lambda\alpha_i$ (this condition should also be satisfied in our convergence analysis), then we can define $q:=\frac{2\lambda\alpha_i}{\lambda\alpha_i+\rho}\frac{b}{b+\lambda}<1$, which implies $\bu_i^{(be)}\sim\mathcal{N}(q \btheta_i,\frac{q^2\sigma^2}{bN}\biden_d)$.
    \item When client $i$ is malicious, then it will send $\bu_i^{(ma)}=p\bone_d\in\mathbb{R}^d$ to the server with $p\sim\mathcal{N}(0,\gamma^2)$, which implies that $\bu_i^{(ma)}\sim\mathcal{N}(\bzero,\gamma^2\boldsymbol{1}_d)$, where $\boldsymbol{1}_d=\begin{pmatrix}  
  1 & 1 & \cdots & 1 \\  
  1 & 1 & \cdots & 1 \\  
  \vdots & \vdots & \ddots & \vdots \\  
  1 & 1 & \cdots & 1  
\end{pmatrix}\in\mathbb{R}^{d\times d}$.
\end{itemize}

Then the server aggregates the update messages from each client and obtains $\bw^\sva=\frac{1}{m}\Bigl(\sum_{i\in\mathcal{S}_a}\bu_i^{(ma)} +\sum_{i\in\mathcal{S}_b}\bu_i^{(be)}\Bigr)$, which implies $\bw^{\sva}\sim\mathcal{N}\Bigl(\frac{1}{m}\sum_{i\in\mathcal{S}_b}q\btheta_i,\frac{1}{m^2}(m_a\gamma^2\boldsymbol{1}_d+\frac{m_b q^2\sigma^2}{bN}\biden_d)\Bigr)$. Next, based on \eqref{eq: 138 be plugged in appendix}, we obtain the average test loss of the global model $\bw^\sva$ on benign clients as 
\begin{align}
    \loss^{\flamegm,\sva}(q) =& \frac{1}{m_b}\sum_{i\in\mathcal{S}_b}f_i(\bw^\sva)\notag\\
    =&\frac{\sigma^2}{2}+\frac{b}{2}\trace\bigl({\var(\bw^\sva)}\bigr)+\frac{b}{2m_b}\sum_{i\in\mathcal{S}_b}\|\mathbb{E}\bw^\sva-\btheta_i\|^2\\
    =& \frac{\sigma^2}{2}+\frac{bd}{2m^2}\Bigl(m_a\gamma^2+\frac{m_b q^2\sigma^2}{bN}\Bigr)+\frac{b}{2m_b}\sum_{i\in\mathcal{S}_b}\|\frac{\sum_{i'\in\mathcal{S}_b}q\btheta_{i'}}{m}-\btheta_i\|^2.
\end{align}
Next, we consider obtaining the distribution of the personalized model. Recall that $\btheta_i^*=\frac{b\hat{\btheta}_i+\lambda\bw^*}{b+\lambda}$, then the benign clients compute the personalized model as $\btheta_i^\sva=\frac{b\hat{\btheta}_i+\lambda\bw^\sva}{b+\lambda}$, which implies the distribution of $\btheta_i^\sva$ is
\begin{align}
    \btheta_i^\sva\sim\mathcal{N}\Bigl(\frac{b\btheta_i+\frac{\lambda}{m}\sum_{i'\in\mathcal{S}_b}q\btheta_{i'}}{b+\lambda},\frac{\Bigl(\bigl(b+\frac{q\lambda}{m}\bigr)^2\frac{\sigma^2}{bN}+(m_b-1)\frac{q^2\lambda^2}{m^2}\frac{\sigma^2}{bN}\Bigr)\biden_d+\frac{m_a\lambda^2\gamma^2}{m^2}\boldsymbol{1}_d}{(b+\lambda)^2}\Bigr).
\end{align}
Then the average test loss of the personalized model on benign clients is
\begin{align}
    \loss^{\flamepm,\sva}(q)=&\frac{1}{m_b}\sum_{i\in\mathcal{S}_b}f_i(\btheta_i^\sva) \notag\\
    =&\frac{\sigma^2}{2}+\frac{b}{2m_b}\sum_{i\in\mathcal{S}_b}\trace\bigl({\var(\btheta_i^\sva)}\bigr)+\frac{b}{2m_b}\sum_{i\in\mathcal{S}_b}\|\mathbb{E}\btheta_i^\sva-\btheta_i\|^2\\
    =&\frac{\sigma^2}{2}+\frac{bd}{2}\frac{\Bigl(b^2+\frac{2bq\lambda}{m}+\frac{m_b q^2\lambda^2}{m^2}\Bigr)\frac{\sigma^2}{bN}+\frac{m_a\lambda^2\gamma^2}{m^2}}{(b+\lambda)^2}+\frac{b\lambda^2}{2m_b(b+\lambda)^2}\sum_{i\in\mathcal{S}_b}\|\frac{\sum_{i'\in\mathcal{S}_b}q\btheta_{i'}}{m}-\btheta_i\|^2.
\end{align}
According to \cite{Lin2022Personalized}, the average test loss of personalized and global models on benign clients for \pfedme and \ditto are 
\begin{align}
   \text{Loss}^{{\pfedmepm,\sva}}= &\text{Loss}^{{\dittopm,\sva}}=\frac{\sigma^2}{2}+\frac{bd}{2}\frac{\Bigl(b^2+\frac{2b\lambda}{m}+\frac{m_b \lambda^2}{m^2}\Bigr)\frac{\sigma^2}{bN}+\frac{m_a\lambda^2\gamma^2}{m^2}}{(b+\lambda)^2}+\frac{b\lambda^2}{2m_b(b+\lambda)^2}\sum_{i\in\mathcal{S}_b}\|\frac{\sum_{i'\in\mathcal{S}_b}\btheta_{i'}}{m}-\btheta_i\|^2,\\
   \text{Loss}^{{\pfedmegm,\sva}}= &\text{Loss}^{{\dittogm,\sva}}=\frac{\sigma^2}{2}+\frac{bd}{2m^2}\Bigl(m_a\gamma^2+\frac{m_b \sigma^2}{bN}\Bigr)+\frac{b}{2m_b}\sum_{i\in\mathcal{S}_b}\|\frac{\sum_{i'\in\mathcal{S}_b}\btheta_{i'}}{m}-\btheta_i\|^2.
\end{align}
Note that the testing losses of \pfedme and \ditto can be considered as the testing loss of \flame when $q=1$. Next, we take the derivatives of the testing loss of \flame with respect to $q$. They are
\begin{align}
    \frac{\partial \text{Loss}^{{\flamegm,\sva}}(q)}{\partial q}&=\frac{bd}{2m^2}\frac{2m_bq\sigma^2}{bN}+\frac{b}{2m_b}2(qm_b-m)\Bigl(\frac{\sum_{i'\in\mathcal{S}_b}\btheta_{i'}}{m}\Bigr)^\top\frac{\sum_{i'\in\mathcal{S}_b}\btheta_{i'}}{m}, \\
    \frac{\partial \text{Loss}^{{\flamepm,\sva}}(q)}{\partial q} &= \frac{bd}{2}\frac{\frac{\sigma^2}{bN}}{(b+\lambda)^2}\Bigl(\frac{2b\lambda}{m}+\frac{2m_bq\lambda^2}{m^2}\Bigr)+\frac{b\lambda^2}{2m_b(b+\lambda)^2}\cdot 2(q m_b-m)\Bigl(\frac{\sum_{i'\in\mathcal{S}_b}\btheta_{i'}}{m}\Bigr)^\top\frac{\sum_{i'\in\mathcal{S}_b}\btheta_{i'}}{m}.
\end{align}
One can check $\frac{\partial \text{Loss}^{{\flamegm,\sva}}(q)}{\partial q}\geq 0$ when $q\geq \frac{m N b\overline{\btheta}^\top_{b}\overline{\btheta}_{b}}{d\sigma^2+m_b Nb}$, and 
$\frac{\partial \text{Loss}^{{\flamepm,\sva}}(q)}{\partial q}\geq 0$ when $q\geq \frac{mb (m_bN\lambda\overline{\btheta}^\top_{b}\overline{\btheta}_{b}-d\sigma^2)}{d\sigma^2m_b\lambda+m_b^2 Nb\lambda\overline{\btheta}^\top_{b}\overline{\btheta}_{b}}$, 
where $\overline{\btheta}_{b} =\frac{\sum_{i'\in\mathcal{S}_b}\btheta_{i'}}{m_b}$.
Then we can infer that $\loss^{\flamepm,\sva}\leq\text{Loss}^{{\pfedmepm,\sva}}(\lambda)= \text{Loss}^{{\dittopm,\sva}}(\lambda)$ and $\loss^{\flamegm,\sva}\leq\text{Loss}^{{\pfedmegm,\sva}}= \text{Loss}^{{\dittogm,\sva}}$.

\subsubsection{Sign-flipping Attacks} 
We consider computing the test loss of \flame under Sign-flipping attacks. 
\begin{itemize}
    \item When client $i$ is benign, then it will send $\bu_i^{(be)}$ to the server and $\bu_i^{(be)}\sim\mathcal{N}(q \btheta_i,\frac{q^2\sigma^2}{bm}\biden_d)$.
    \item When client $i$ is malicious, then it will send $\bu_i^{(ma)}=-|p|\bu_i$ to the server, where $p\sim \mathcal{N}(0,\gamma^2)$. Then we can obtain the mean and covariance of $\bu_i^{(ma)}$ by the independence of $|p|$ and $\bu_i$ as \cite{Lin2022Personalized}
\begin{align}
    \mathbb{E}[\bu_i^{(ma)}]&=\mathbb{E}[-|p|\bu_i^{(be)}] = -\sqrt{\frac{2}{\pi}}\gamma q\btheta_i,\\
    \var[\bu_i^{(ma)}] &=\var[-|p|\bu_i^{(be)}] \notag\\
    &= \mathbb{E}[p^2 \bu_i^{(be)}(\bu_i^{(be)})^{\top}]-\mathbb{E}[|c|\bu_i^{(be)}] \mathbb{E}[|c|(\bu_i^{(be)})^{\top}]   \notag  \\
    &= \gamma^2\Bigl(q^2\btheta_i\btheta_i^{\top}+\frac{q^2\sigma^2}{bm}\biden_d\Bigr)-\frac{2}{\pi}\gamma^2 q^2 \btheta_i\btheta_i^{\top}      \notag\\
    &= \frac{\pi-2}{\pi}\gamma^2 q^2 \btheta_i\btheta_i^{\top} + \gamma^2\frac{q^2 \sigma^2}{bm}\biden_d=:q^2 \bV_i. \label{eq: define vi}
\end{align}
where \eqref{eq: define vi} defines $\bV_i:=\frac{\pi-2}{\pi}\gamma^2  \btheta_i\btheta_i^{\top} + \gamma^2\frac{ \sigma^2}{bm}\biden_d$.
\end{itemize}

Then the server aggregates the updated messages from each client and updates the global model by $\bw^\sfa=\frac{1}{m}\Bigl(\sum_{i\in\mathcal{S}_b}\bu_i^{(be)}+\sum_{i\in\mathcal{S}_a}\bu_i^{(ma)}\Bigr)$, which implies 
\begin{align}
    \mathbb{E}[\bw^\sfa]&=\frac{1}{m}\Bigl(\sum_{i\in\mathcal{S}_b} q\btheta_i -\sum_{i\in\mathcal{S}_a}\sqrt{\frac{2}{\pi}}\gamma q \btheta_i\Bigr),\\
    \var[\bw^\sfa]&=\frac{1}{m^2}\Bigl(\frac{m_b q^2\sigma^2}{bm}\biden_d + \sum_{i\in\mathcal{S}_a}q^2 \bV_i\Bigr).
\end{align}
Then the average test loss of the global model on benign clients is 
\begin{align}
    \!\!\!\!\!\!\loss^{\flamegm,\sva}(q)=& \frac{1}{m_b}\sum_{i\in\mathcal{S}_b}f_i(\bw^\sfa)    \notag\\
    =&\frac{\sigma^2}{2}+\frac{b}{2}\trace\bigl({\var(\bw^\sfa)}\bigr)+\frac{b}{2m_b}\sum_{i\in\mathcal{S}_b}\|\mathbb{E}\bw^\sfa-\btheta_i\|^2\\
    =&\frac{\sigma^2}{2}+\frac{bd}{2m^2}\Bigl(\frac{m_b q^2\sigma^2}{bm} + q^2\sum_{i\in\mathcal{S}_a} \frac{\trace(\bV_i)}{d}\Bigr)+\frac{b}{2m_b} \sum_{i\in\mathcal{S}_a}\Bigl\|\frac{1}{m}\Bigl(\sum_{i'\in\mathcal{S}_b} q\btheta_{i'} -\sum_{i'\in\mathcal{S}_a}\sqrt{\frac{2}{\pi}}\gamma q \btheta_{i'}\Bigr)-\btheta_i\Bigr\|^2.
\end{align}
Next, we consider obtaining the mean and variance of the personalized model. Recall that $\btheta_i^*=\frac{b\hat{\btheta}_i+\lambda\bw^*}{b+\lambda}$, then the benign clients compute the personalized model as $\btheta_i^\sfa=\frac{b\hat{\btheta}_i+\lambda\bw^\sfa}{b+\lambda}$, which implies that
\begin{align}
    \mathbb{E}[\btheta_i^\sfa]=&\frac{1}{b+\lambda}\Biggl(b\theta_i+\frac{\lambda}{m}\Bigl(\sum_{i'\in\mathcal{S}_b} q\btheta_{i'} -\sum_{i'\in\mathcal{S}_a}\sqrt{\frac{2}{\pi}}\gamma q \btheta_{i'}\Bigr)\Biggr)\\
    \var[\btheta_i^\sfa] =&\frac{\Bigl(\bigl(b+\frac{q\lambda}{m}\bigr)^2\frac{\sigma^2}{bN}+(m_b-1)\frac{q^2\lambda^2}{m^2}\frac{\sigma^2}{bN}\Bigr)\biden_d+\frac{q^2 \lambda^2}{m^2} \sum_{i'\in\mathcal{S}_a} \bV_i}{(b+\lambda)^2}
\end{align}
Then the test loss of the personalized model on benign clients is 
\begin{align}
    \loss^{\flamepm,\sfa}(q) &= \frac{1}{m_b}\sum_{i\in\mathcal{S}_b}f_i(\btheta_i^\sfa)    \notag\\
    &=\frac{\sigma^2}{2}+\frac{b}{2m_b}\sum_{i\in\mathcal{S}_b}\trace\bigl({\var(\btheta_i^\sfa)}\bigr)+\frac{b}{2m_b}\sum_{i\in\mathcal{S}_b}\|\mathbb{E}\btheta_i^\sfa-\btheta_i\|^2\\
    &=\frac{\sigma^2}{2}+\frac{bd}{2}\frac{\Bigl(b^2+\frac{2bq\lambda}{m}+\frac{m_b q^2\lambda^2}{m^2}\Bigr)\frac{\sigma^2}{bN}+\frac{q^2\lambda^2}{m^2}\sum_{i'\in\mathcal{S}_a} \frac{\trace({\bV_i}) }{d} }{(b+\lambda)^2}\notag\\
    &+\frac{b\lambda^2}{2(b+\lambda)^2}\frac{1}{m_b}\sum_{i\in\mathcal{S}_b}\Bigl\|\frac{q}{m}\Bigl(\sum_{i'\in\mathcal{S}_b} \btheta_{i'} -\sum_{i'\in\mathcal{S}_a}\sqrt{\frac{2}{\pi}}\gamma  \btheta_{i'}\Bigr) -\btheta_i\Bigr\|^2
\end{align}
According to \cite{Lin2022Personalized}, the average test loss of personalized and global models on benign clients for \pfedme and \ditto are 
\begin{align*}
   \text{Loss}^{{\pfedmepm,\sfa}}= \text{Loss}^{{\dittopm,\sfa}}&=\frac{\sigma^2}{2}+\frac{bd}{2}\frac{\Bigl(b^2+\frac{2b\lambda}{m}+\frac{m_b \lambda^2}{m^2}\Bigr)\frac{\sigma^2}{bN}+\frac{\lambda^2}{m^2}\sum_{i\in\mathcal{S}_a} \frac{\trace({\bV_i}) }{d} }{(b+\lambda)^2}\notag\\
    &+\frac{b\lambda^2}{2(b+\lambda)^2}\frac{1}{m_b}\sum_{i\in\mathcal{S}_b}\Bigl\|\frac{1}{m}\Bigl(\sum_{i'\in\mathcal{S}_b} \btheta_{i'} -\sum_{i'\in\mathcal{S}_a}\sqrt{\frac{2}{\pi}}\gamma  \btheta_{i'}\Bigr) -\btheta_i\Bigr\|^2,
\end{align*}
\begin{align*}
   \text{Loss}^{{\pfedmegm,\sfa}}= \text{Loss}^{{\dittogm,\sfa}}&=\frac{\sigma^2}{2}+\frac{bd}{2m^2}\Bigl(\frac{m_b \sigma^2}{bm} + \sum_{i\in\mathcal{S}_a} \frac{\trace(\bV_{i})}{d}\Bigr)+\frac{b}{2m_b} \sum_{i\in\mathcal{S}_a}\Bigl\|\frac{1}{m}\Bigl(\sum_{i'\in\mathcal{S}_b} \btheta_{i'} -\sum_{i'\in\mathcal{S}_a}\sqrt{\frac{2}{\pi}}\gamma  \btheta_{i'}\Bigr)-\btheta_i\Bigr\|^2.
\end{align*}
Note that the testing losses of \pfedme and \ditto can be considered as the testing loss of \flame when $q=1$. Next, we take the derivatives of the testing loss of \flame with respect to $q$. They are
\begin{align}
    \frac{\partial \text{Loss}^{{\flamegm,\sfa}}(q)}{\partial q} &= \frac{bd}{m^2}\Bigl(\frac{m_b\sigma^2q}{bm}+ q\sum_{i\in\mathcal{S}_a}\frac{\trace(\bV_i)}{d} \Bigr)+\frac{bm_a}{m_b}\overline{\btheta}_m^\top\overline{\btheta}_m q - \frac{b}{m_b} \sum_{i\in\mathcal{S}_a}\btheta_i^\top \overline{\btheta}_m,\\
    \frac{\partial \text{Loss}^{{\flamepm,\sfa}}(q)}{\partial q} &= \frac{bd}{(b+\lambda)^2}\Bigl[\frac{\sigma^2}{bN}\Bigl(\frac{b\lambda}{m}+\frac{m_b\lambda^2q}{m^2} \Bigr)+\frac{q\lambda^2}{m^2}\sum_{i'\in\mathcal{S}_a } \frac{\trace(\bV_{i'})}{d}\Bigr] + \frac{\lambda^2}{(b+\lambda)^2}\Bigl(\frac{bm_a}{m_b}\overline{\btheta}_m^\top \overline{\btheta}_m q-\frac{b}{m_b}\sum_{i'\in\mathcal{S}_a}\btheta_i^\top \overline{\btheta}_m\Bigr).
\end{align}
One can check $\frac{\partial \text{Loss}^{{\flamegm,\sfa}}(q)}{\partial q}\geq 0$ when $q\geq\frac{\frac{1}{m_b} \sum_{i\in\mathcal{S}_a}\btheta_i^\top \overline{\btheta}_m }{\frac{d}{m^2}\Bigl(\frac{m_b\sigma^2}{bm}+\sum_{i\in\mathcal{S}_a} \frac{\trace(\bV_i)}{d} \Bigr)+\frac{m_a}{m_b}\overline{\btheta}_m^\top\overline{\btheta}_m}$, and $\frac{\partial \text{Loss}^{{\flamepm,\sfa}}(q)}{\partial q}\geq 0$ when $q\geq \frac{\frac{b\lambda}{m_b}\sum_{i'\in\mathcal{S}_a}\btheta_i^\top \overline{\btheta}_m-\frac{bd\sigma^2}{mN} }{\frac{d\sigma^2m_b\lambda}{m^2N}+ \frac{b\lambda d}{m^2}\sum_{i\in\mathcal{S}_a} \frac{\trace(\bV_i)}{d} + \frac{b\lambda m_a}{m_b} \overline{\btheta}_m^\top \overline{\btheta}_m}$, where $\overline{\btheta}_m = \frac{1}{m}\Bigl(\sum_{i'\in\mathcal{S}_b}\btheta_{i'} -\sum_{i'\in\mathcal{S}_a}\sqrt{\frac{2}{\pi}} \gamma\btheta_{i'} \Bigr)$.
Then, following the monotonic increasing property of the testing loss function, we can infer that $\loss^{\flamepm,\sfa}\leq\text{Loss}^{{\pfedmepm,\sfa}}(\lambda)= \text{Loss}^{{\dittopm,\sfa}}(\lambda)$ and $\loss^{\flamegm,\sfa}\leq\text{Loss}^{{\pfedmegm,\sfa}}= \text{Loss}^{{\dittogm,\sfa}}$.


\subsubsection{Gaussian Attacks}
When client $i$ is malicious, then it will send $\bu_i^{(ma)}= p\biden_d$ to the server, where $p\sim\mathcal{N}(0,\gamma^2)$, which implies that $\bu_i^{(ma)}\sim\mathcal{N}(\boldsymbol{0}_d,\gamma^2\biden_d)$. Note that $\trace(\biden_d)=\trace(\boldsymbol{1 }_d)$, then the test loss on benign clients is similar to that of the same-value attacks.

\subsection{Fairness}\label{appendix: fairness}
We compute the variance of the test loss of the global model and personalized model to measure fairness. Recall that after one communication, the personalized and global models are updated as 
\begin{align}
     \bw^{\flame} &=\frac{1}{m}\sum_{i=1}^m q\hat{\btheta}_i \sim \mathcal{N}\Bigl(q\overline{\btheta}, \frac{q^2\sigma^2}{bmN}\biden_d  \Bigr), \\
     \btheta_i^\flame &=\frac{b\hat{\btheta}_i+\lambda\bw^\flame}{b+\lambda} \sim \mathcal{N}\Bigl(\frac{b\btheta_i+q\lambda\overline{\btheta}}{b+\lambda}, \frac{(b^2+\frac{2b q \lambda}{m})\frac{\sigma^2}{bN}+\frac{q^2 \lambda^2\sigma^2}{bmN}}{(b+\lambda)^2}\biden_d \Bigr), i\in[m].
\end{align}
Then the test losses of the personalized and global models on each client are
\begin{align}
     f_i(\bw^{\flame})&=\frac{\sigma^2}{2}+\frac{b}{2}\trace(\var(\bw^{\flame}))+\frac{b}{2}\|\mathbb{E}\bw^{\flame}-\btheta_i\|^2 \notag\\
     &= \frac{\sigma^2}{2}+ \frac{q^2\sigma^2 d}{2mN} +\frac{b}{2} \|q \overline{\btheta}-\btheta_i\|^2,     \\  
     f_i(\btheta_i^{\flame})&=\frac{\sigma^2}{2}+\frac{b}{2}\trace(\var(\btheta_i^\flame))+\frac{b}{2}\|\mathbb{E}\btheta_i^{\flame}-\btheta_i\|^2 \notag\\
     &=\frac{\sigma^2}{2}+\frac{b^2+\frac{2b q \lambda}{m} + \frac{q^2\lambda^2}{m}}{(b+\lambda)^2}\cdot\frac{\sigma^2 d}{2N}+\frac{b\lambda^2}{2(b+\lambda)^2}\|q \overline{\btheta}-\btheta_i\|^2.
\end{align}
Then the variances of test losses are obtained as follows,
\begin{align}
    \var[f_i(\bw^{\flame})] &= \frac{b^2}{4} \var[\|q \overline{\btheta}-\btheta_i\|^2], \label{eq:variance of flame global} \\
    \var[f_i(\btheta_i^{\flame})] &= \frac{b^2\lambda^4}{4(b+\lambda)^4} \var [\|q \overline{\btheta}-\btheta_i\|^4].\label{eq:variance of flame personal} 
\end{align}
According to \cite{Lin2022Personalized}, the variances \pfedme and \ditto correspond to the special cases where $q=1$ in \eqref{eq:variance of flame global} and \eqref{eq:variance of flame personal}.
Next, we consider investigating the monotonicity of $\var[\|q \overline{\btheta}-\btheta_i\|^2]$ with respect to $q$. We now give an equivalent form of $\var[\|q \overline{\btheta}-\btheta_i\|^2]$ as 
\begin{align}
    \var[\|q \overline{\btheta}-\btheta_i\|^2] &= \frac{1}{m}\sum_{i=1}^m \|q \overline{\btheta}-\btheta_i\|^4 - \Bigl(\frac{1}{m} \sum_{i=1}^m\|q \overline{\btheta}-\btheta_i\|^2    \Bigr)^2\notag\\
    & = \frac{1}{m^2}\Bigl( (m-1) \sum_{i=1}^m \|q \overline{\btheta}-\btheta_i\|^4 -  \sum_{i\neq i'} \|q \overline{\btheta}-\btheta_i\|^2 \|q \overline{\btheta}-\btheta_{i'}\|^2 \Bigr) \\
    & = \frac{1}{m^2}  \sum_{i\neq i'}\frac{(\|q \overline{\btheta}-\btheta_i\|^2-\|q \overline{\btheta}-\btheta_{i'}\|^2)^2}{2}
\end{align}
Then we can obtain the derivative of $\var[\|q \overline{\btheta}-\btheta_i\|^2]$ with respect to $q$
\begin{align}
    \frac{\partial \var[\|q \overline{\btheta}-\btheta_i\|^2]}{\partial q} = \frac{2}{m^2} \sum_{i\neq i'} (\|q \overline{\btheta}-\btheta_i\|^2-\|q \overline{\btheta}-\btheta_{i'}\|^2) (\btheta_{i'}  -\btheta_i)^{\top}\overline{\btheta}.
\end{align}
Next, we derive the second-order derivative of $\var[\|q \overline{\btheta}-\btheta_i\|^2]$ with respect to $q$
\begin{align}
    \frac{\partial^2 \var[\|q \overline{\btheta}-\btheta_i\|^2]}{\partial q^2} = -\frac{4}{m^2} \sum_{i\neq i'} \overline{\btheta}^\top (\btheta_{i'}  -\btheta_i)(\btheta_{i'}  -\btheta_i)^{\top}\overline{\btheta}\leq 0
\end{align}
Thus, we can deduce that $ \frac{\partial \var[\|q \overline{\btheta}-\btheta_i\|^2]}{\partial q}$ decreases monotonically with $q$, reaching its minimum when $q=1$. That is
\begin{align}
    \frac{\partial \var[\|q \overline{\btheta}-\btheta_i\|^2]}{\partial q} \Big |_{q=1} &= \frac{2}{m^2} \sum_{i\neq i'} (\|\overline{\btheta}-\btheta_i\|^2-\|\overline{\btheta}-\btheta_{i'}\|^2) (\btheta_{i'}  -\btheta_i)^{\top}\overline{\btheta}\notag \\
    &= \frac{2}{m^2} \sum_{i\neq i'} (2\overline{\btheta}-\btheta_{i'}-\btheta_i)^\top(\btheta_{i'}  -\btheta_i) (\btheta_{i'}  -\btheta_i)^{\top}\overline{\btheta}\\
    &= \frac{2}{m^2} \sum_{i\neq i'} 2\overline{\btheta}^\top(\btheta_{i'}  -\btheta_i) (\btheta_{i'}  -\btheta_i)^{\top}\overline{\btheta}-\frac{2}{m^2} \sum_{i\neq i'} (\btheta_{i'}  +\btheta_i)^\top(\btheta_{i'}  -\btheta_i) (\btheta_{i'}  -\btheta_i)^{\top}\overline{\btheta}\\
    &= \frac{2}{m^2} \sum_{i\neq i'} 2\overline{\btheta}^\top(\btheta_{i'}  -\btheta_i) (\btheta_{i'}  -\btheta_i)^{\top}\overline{\btheta}-\frac{2}{m^2} \sum_{i\neq i'} (\|\btheta_{i'}\|^2  -\|\btheta_i\|^2) (\btheta_{i'}  -\btheta_i)^{\top}\overline{\btheta}.
\end{align}
Suppose that all personalized models are uniform, then we can obtain $\|\btheta_{i'}\|^2  =\|\btheta_i\|^2$ for any $i'\neq i$, which implies
\begin{align}
    \frac{\partial \var[\|q \overline{\btheta}-\btheta_i\|^2]}{\partial q} \Big |_{q=1} = \frac{2}{m^2} \sum_{i\neq i'} 2\overline{\btheta}^\top(\btheta_{i'}  -\btheta_i) (\btheta_{i'}  -\btheta_i)^{\top}\overline{\btheta}\geq 0.
\end{align}
Therefore, we can infer that $\var[\|q \overline{\btheta}-\btheta_i\|^2]$ increases monotonically with respect to $q$, which completes the proof.

\newpage
\section{Extension}\label{secappendix: Extension}

We extend the established ADMM-based training method to the recently proposed fair and robust PFL framework: \lpproj \cite{Lin2022Personalized}. 
We present the specific algorithm of ADMM-based \lpproj in Algorithm \ref{alg: lpproj-ADMM}.
\begingroup    
\begin{algorithm}[h]
  \caption{ADMM-based \lpprojtwo}
  \label{alg: lpproj-ADMM}
  \SetAlgoLined
  \KwIn{$T:$ the total communication rounds, $\rho$: the penalty parameter, $\lambda$: the hyperparameter, $m$: the number of clients, $X_i,i\in[m]$: the local dataset, $\eta$: the learning rate, $H$: the number of local iterations, $\bP$: the projection matrix.}
  \textbf{Initialize:} $\btheta_i^0,\bw_i^0,\bpi_i^0,\bu_i^{0}=\bw_i^{0}+\frac{1}{\rho}\bpi_i^{0}$, $i\in[m]$.
  
  \For{$t=0,1,\dots, T-1$}{
  \tcc{On the server side.}
  Randomly select $s$ clients $\mathcal{S}^t\subset[m]$\;

  Call each client, uploading $\{\boldsymbol{z}^{t}_1,\dots,\boldsymbol{z}^{t}_m\}$ to the server\;

  Update $\bw^{t}=\frac{1}{m}\sum_{i=1}^{m}\bu_i^{t}$\;

  Broadcast $\bw^{t}$ to the selected clients\;

  \tcc{On the client side.}

  \For{each client $i\in\mathcal{S}^t$}{
  Create Batches $\mathcal{B}$\;

  \For{$h=0,1,\dots,H-1$}{
    \For{batch $\xi\in\mathcal{B}$}{
    Compute the gradient $f_i(\btheta_i^{t,h},\xi)$\;

    $\btheta_i^{t,h+1}=\btheta_i^{t,h}-\eta(\nabla f_i(\btheta_i^{t,h},\xi)+\lambda\bP^\top(\bP\btheta_i^{t,h}-\bw_i^t))$\;

    }
  }
    $\btheta_i^{t+1}=\btheta_i^{t,H-1}$\;
    
    $\bw_i^{t+1}=\frac{1}{\lambda\alpha_i+\rho}(\lambda\alpha_i\bP\btheta_i^{t+1}+\rho\bw^t-\bpi_i^t)$\;

    $\bpi_i^{t+1}=\bpi_i^t+\rho(\bw_i^{t+1}-\bw^t)$\;

    $\bu_i^{t+1}=\bw_i^{t+1}+\frac{1}{\rho}\bpi_i^{t+1}$\;
    
  }
  \For{each client $i\notin\mathcal{S}^t$}{
  $(\btheta_i^{t+1},\bw_i^{t+1},\bpi_i^{t+1},\bu_i^{t+1})=(\btheta_i^{t},\bw_i^{t},\bpi_i^{t},\bu_i^{t})$\;
  
  }
  }

\end{algorithm}    
\setlength{\textfloatsep}{2pt}

\endgroup

Next, we introduce the objective function and propose our training framework for \lpproj.
Lin et al. \cite{Lin2022Personalized} proposed to project local models into a shared-and-fixed low-dimensional random subspace and uses infimal convolution to control the deviation between the global model and projected local models. Specifically, they construct the objective function for PFL as
\begin{equation}\label{eq: lpproj}
\begin{split}
    \min_{\bw}\!\sum_{i=1}^{m}\alpha_i F_i(\bw),\,\,
    \text{where } F_i(\bw):=\min_{\btheta_i}f_i(\btheta_i)+\frac{\lambda}{p}||\bw-\bP\btheta_i||^p_p,\,\, i\in[m].
\end{split}
\end{equation} 
where $p \geq 1$ and $\bP\in \mathbb{R}^{d_{sub} \times d}$ is a random matrix that is generated initially and will not vary anymore.
$d_{sub}$ is the dimension of the shared-and-fixed random subspace. 
Without loss of generality, we consider the case where $p=2$, referring to \lpprojtwo.
Next, we consider converting \eqref{eq: lpproj} into the following optimization problem using the same approach as applied to \flame,
\begin{equation}\label{eq:problem lpproj}
\begin{split}
    \min_{\btheta_i,\bw_i,\bw}\Bigr\{\tilde{f}(\Theta,W):=\sum_{i=1}^{m}\alpha_i \Bigl(f_i(\btheta_i)+\frac{\lambda}{2}||\bw_i-\bP\btheta_i||^2\Bigr)\Bigl\},\,\,\text{ s.t. }\,\bw_i=\bw,\,\, i\in[m].
\end{split}
\end{equation}   
To implement ADMM for Problem (\ref{eq:problem lpproj}), we establish the corresponding augmented Lagrangian function as follows:
\begin{equation}\label{eq: lagrangian function for lpproj}
\begin{split}
    \ml(\Theta,W,\Pi,\bw)&:=\sum_{i=1}^m\ml_i(\btheta_i,\bw_i,\bw,\bpi_i), \\
    \ml_i(\btheta_i,\bw_i,\bpi_i,\bw)&:=\alpha_i(f_i(\btheta_i)+\frac{\lambda}{2}||\bw_i-\bP\btheta_i||^2)+\langle\bpi_i,\bw_i-\bw\rangle+\frac{\rho}{2}||\bw_i-\bw||^2,
\end{split}
\end{equation} 
The ADMM framework for solving Problem (\ref{eq:problem lpproj}) can be summarized as follows: after initializing the variables with $(\Theta^0,W^0,\Pi^0,\bw^0)$, the following update steps are executed iteratively for each $t\geq0, $
\begin{align}
\btheta_i^{t+1}&=\operatorname{argmin}_{\btheta_i}\ml_i(\btheta_i,  \bw_i^t, \boldsymbol{\pi}^t_i,\bw^{t},),\label{eq183}\\ 
\bw_i^{t+1} & =\operatorname{argmin}_{\bw_i} \ml_i(\btheta_i^{t+1}, \bw_i, \boldsymbol{\pi}^t_i, \bw^{t})  \notag\\ 
&=\frac{1}{\lambda\alpha_i+\rho}(\lambda\alpha_i\bP\btheta_i^{t+1}+\rho\bw^t-\bpi_i^t),\label{eq184}\\
\boldsymbol{\pi}_i^{t+1} & =\boldsymbol{\pi}_i^t+\rho(\bw_i^{t+1}-\bw^{t}),\label{eq185}\\
\bw^{t+1} & =\operatorname{argmin}_{\bw} \mathcal{L}(\Theta^{t+1}, W^{t+1}, \bw, \Pi^{t+1})\notag\\
&=\frac{1}{m} \sum_{i=1}^m(\bw_i^{t+1}+\frac{1}{\rho}\boldsymbol{\bpi}_i^{t+1})\label{eq186}. 
\end{align}

In Algorithm \ref{alg: lpproj-ADMM}, we introduce the ADMM-based \lpprojtwo. Except for the updates of $\btheta_i$ and $\bw_i$, all other aspects are identical to \flame.
After being mapped by the projection matrix $\bP$, the global model is compressed into a lower-dimensional matrix, rendering it no longer useful.

\newpage



\newpage

\section{Additional and Complete Experiment Results}\label{appendix: additional experiments}
\subsection{Complete Results on Accuracy and Convergence}\label{secappendix: Complete Results on Accuracy and Convergence}
Fig. \ref{fig:ACC and loss_communication_q-label-skew_2_attack_0_num_malicious_0} - Fig. \ref{fig:ACC and loss_communication_dir-quantity-skew_2_attack_0_num_malicious_0} show how the accuracy of different methods changes with communication rounds for different types of heterogeneous data. We can see that \flamehm achieves the highest accuracy and convergence in most datasets and data partitioning methods. In label skew scenarios, personalized models typically perform better, while in other types of heterogeneous data, the global model generally performs better.
\subsection{Complete Results on Robustness}\label{secappendix: Complete Results on Robustness}
Complete results for different methods under various types of Byzantine attacks and heterogeneous data are shown in Fig. \ref{fig:acc_malicious_attack_1_dir-label-skew_2} - Fig. \ref{fig:acc_malicious_attack_4_dir-quantity-skew_2}. 
\flame shows stable performance and is more robust than other methods against most attacks and with most types of heterogeneous data.
\subsection{Complete Results on Fairness}\label{secappendix: Complete Results on Fairness}
Fig. \ref{fig:variance_accuracy_dir-label-skew_q=2} - Fig. \ref{fig:variance_accuracy_dir-quantity-skew_q=2} show the trade-off between variance and accuracy for different methods on various types of heterogeneous data. We can see that \flamehm achieves the best trade-off in most datasets and data partitioning methods, and its accuracy is consistently the highest. In the case of label skew, personalized models usually achieve a better trade-off, while for other types of heterogeneous data, the global model usually achieves a better trade-off.

\begin{figure*}[ht]
  \centering
  \includegraphics[width=1\textwidth]{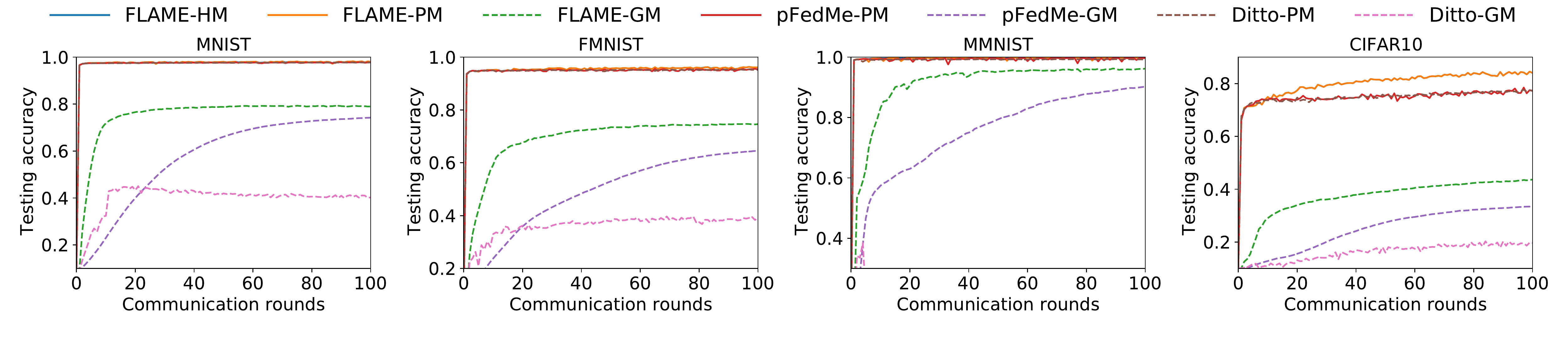} 
  \vspace{-2em}
  \caption{A comparison of the test accuracy across different methods with quantity-based label imbalance ($q=2$).}
  \label{fig:ACC and loss_communication_q-label-skew_2_attack_0_num_malicious_0}
  \vspace{-1em}
\end{figure*}

\begin{figure*}[ht]
  \centering
  \includegraphics[width=1\textwidth]{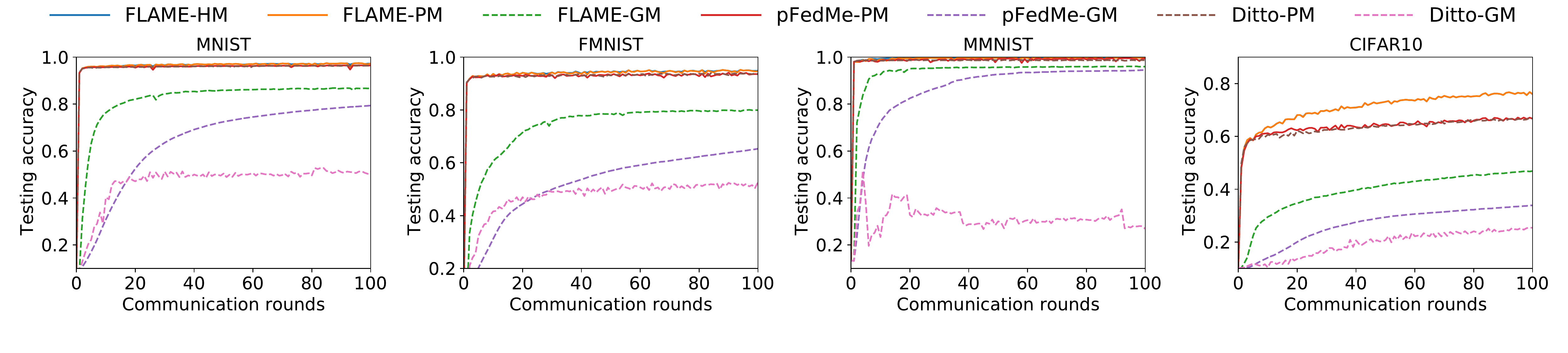} 
  \vspace{-2em}
  \caption{A comparison of the test accuracy across different methods with quantity-based label imbalance ($q=3$).}
  \label{fig:ACC and loss_communication_q-label-skew_3_attack_0_num_malicious_0}
  \vspace{-1em}
\end{figure*}

\begin{figure*}[ht]
  \centering
  \includegraphics[width=1\textwidth]{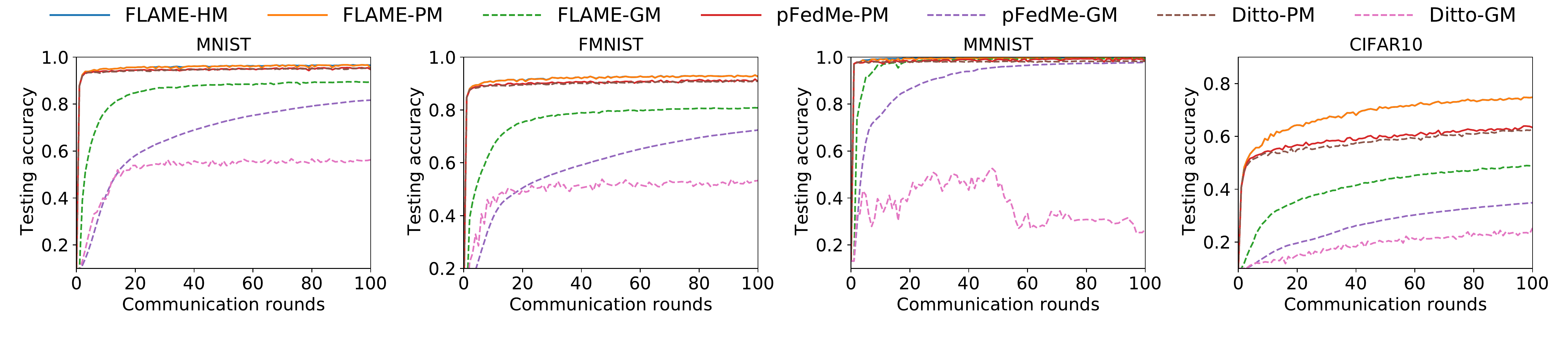} 
  \vspace{-2em}
  \caption{A comparison of the test accuracy across different methods with quantity-based label imbalance ($q=4$).}
  \label{fig:ACC and loss_communication_q-label-skew_4_attack_0_num_malicious_0}
  \vspace{-1em}
\end{figure*}

\begin{figure*}[ht]
  \centering
  \includegraphics[width=1\textwidth]{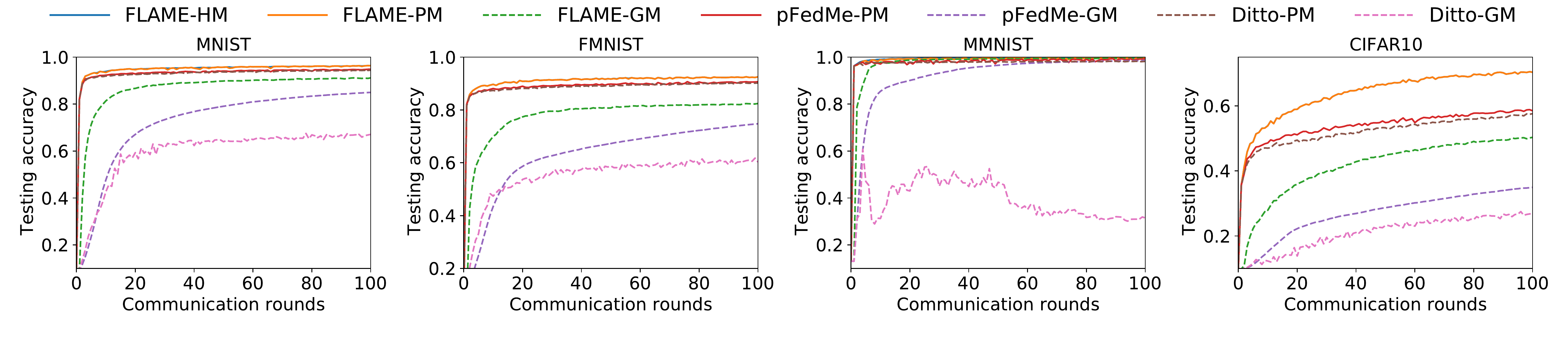} 
  \vspace{-2em}
  \caption{A comparison of the test accuracy across different methods with quantity-based label imbalance ($q=5$).}
  \label{fig:ACC and loss_communication_q-label-skew_5_attack_0_num_malicious_0}
  \vspace{-1em}
\end{figure*}

\begin{figure*}[ht]
  \centering
  \includegraphics[width=1\textwidth]{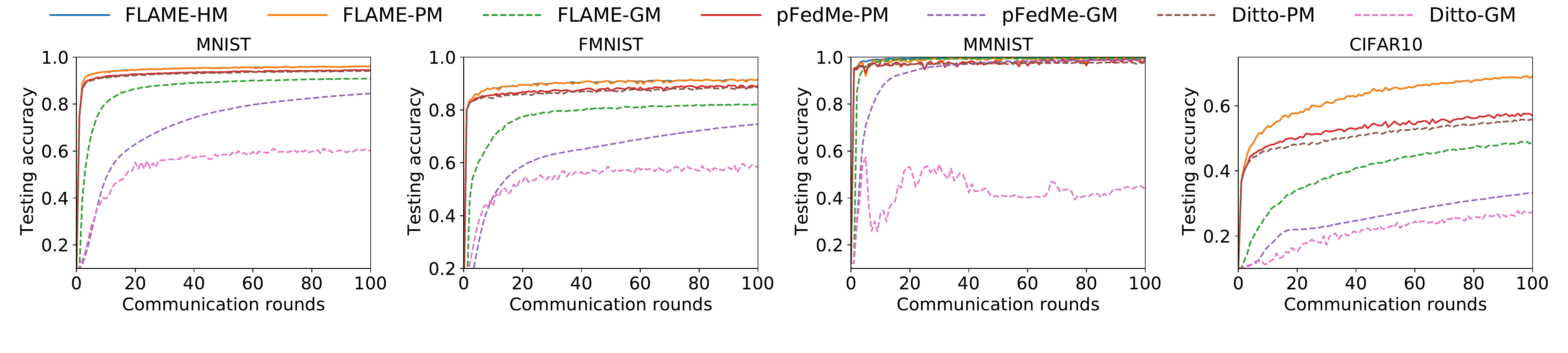} 
  \vspace{-2em}
  \caption{A comparison of the test accuracy across different methods with quantity-based label imbalance ($q=6$).}
  \label{fig:ACC and loss_communication_q-label-skew_6_attack_0_num_malicious_0}
  \vspace{-1em}
\end{figure*}

\begin{figure*}[ht]
  \centering
  \includegraphics[width=1\textwidth]{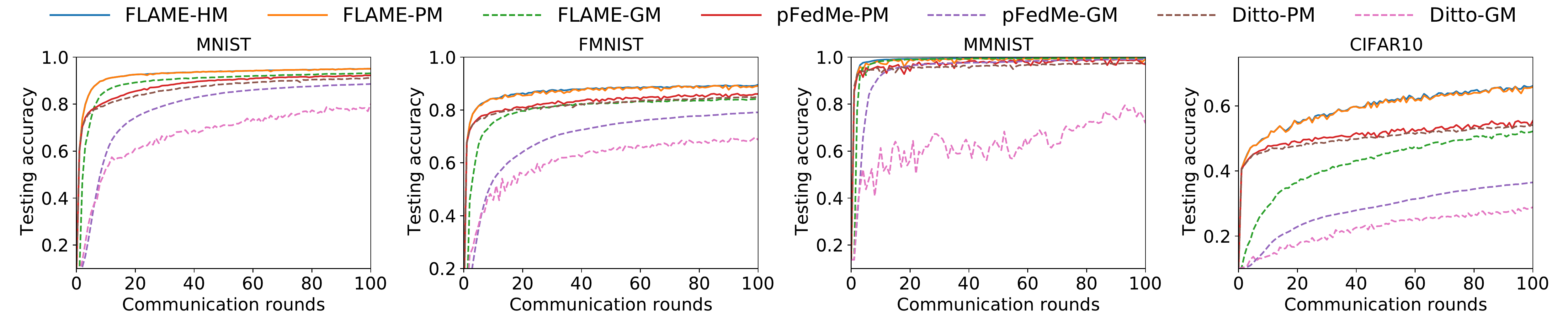} 
  \vspace{-2em}
  \caption{A comparison of the test accuracy across different methods with distribution-based label imbalance.}
  \label{fig:ACC and loss_communication_dir-label-skew_2_attack_0_num_malicious_0}
  \vspace{-1em}
\end{figure*}

\begin{figure*}[ht]
  \centering
  \includegraphics[width=1\textwidth]{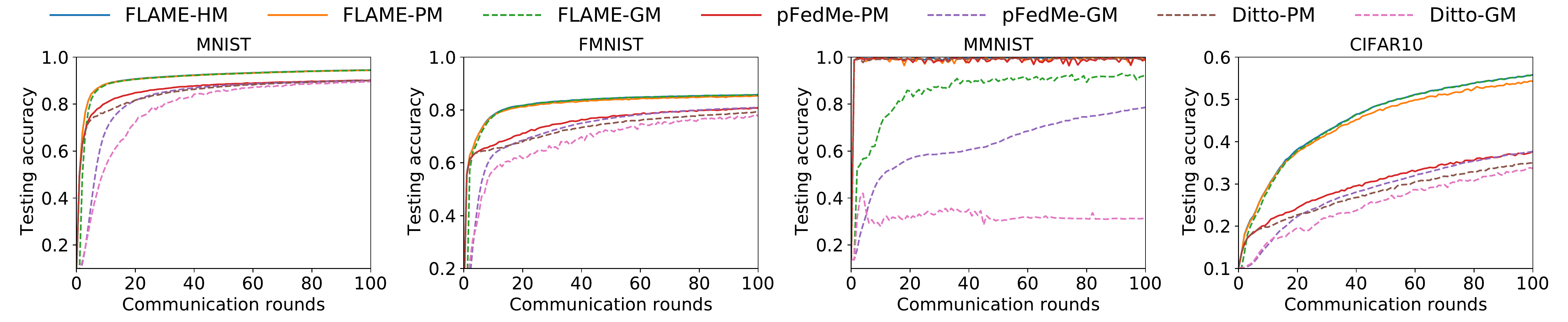} 
  \vspace{-2em}
  \caption{A comparison of the test accuracy across different methods with quality skew.}
  \label{fig:ACC and loss_communication_quality-skew_2_attack_0_num_malicious_0}
  \vspace{-1em}
\end{figure*}

\begin{figure*}[ht]
  \centering
  \includegraphics[width=1\textwidth]{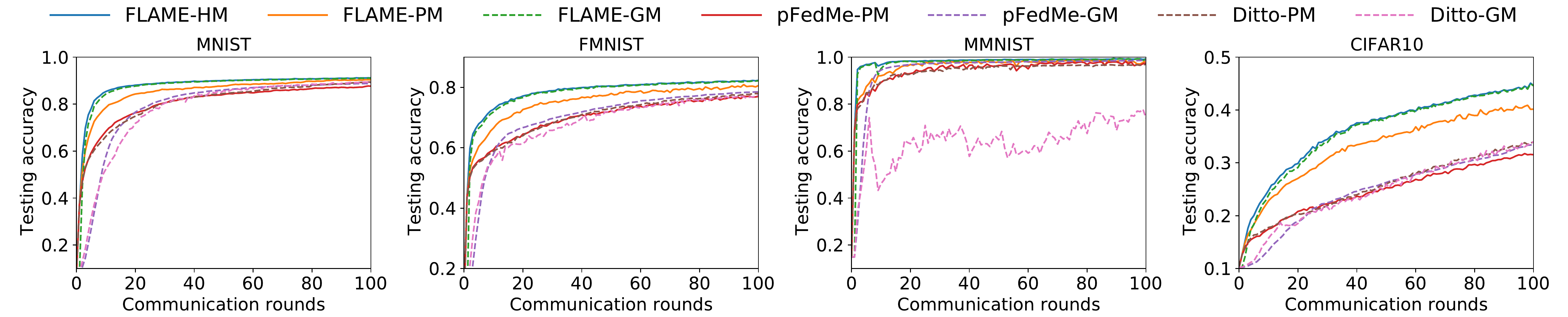} 
  \vspace{-2em}
  \caption{A comparison of the test accuracy across different methods with quantity skew.}
  \label{fig:ACC and loss_communication_dir-quantity-skew_2_attack_0_num_malicious_0}
  \vspace{-1em}
\end{figure*}

\newpage
\begin{figure*}[ht]
    \centering
    \includegraphics[width=1\linewidth]{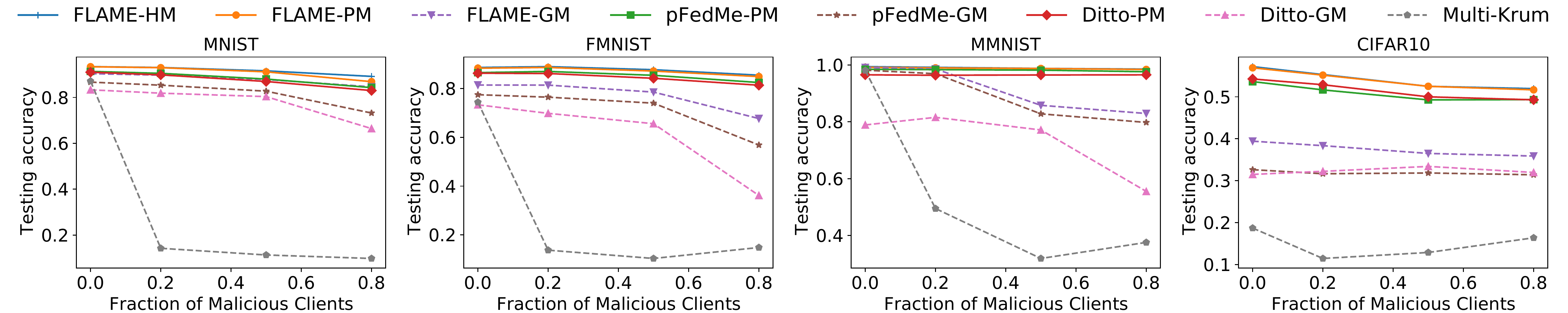}
    \vspace{-2em}
    \caption{Robustness comparison of different methods under label poisoning attacks with distribution-based label imbalance.}
    \label{fig:acc_malicious_attack_1_dir-label-skew_2}
    \vspace{-1em}
\end{figure*}
\begin{figure*}[ht]
    \centering
    \includegraphics[width=1\linewidth]{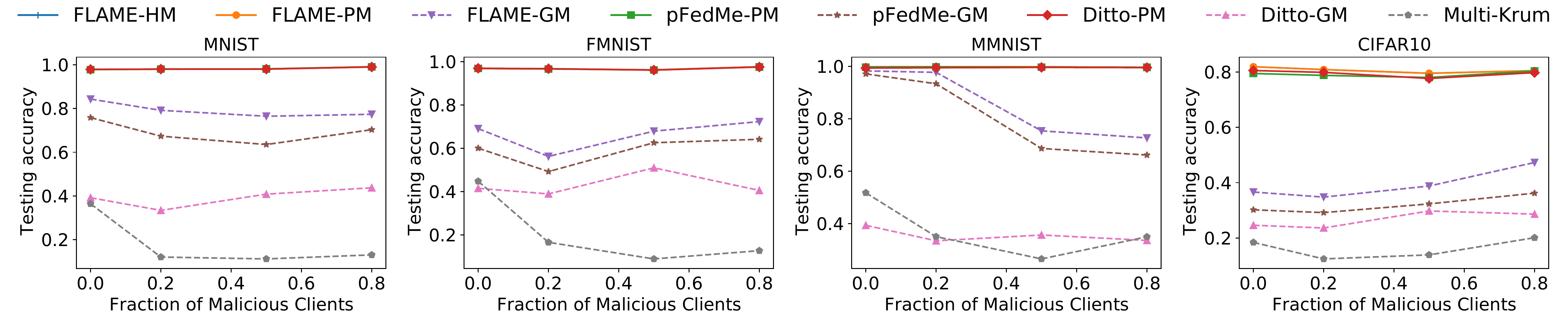}
    \vspace{-2em}
    \caption{Robustness comparison of different methods under label poisoning attacks with quantity-based label imbalance.}
    \label{fig:acc_malicious_attack_1_q-label-skew_2}
    \vspace{-1em}
\end{figure*}
\begin{figure*}[ht]
    \centering
    \includegraphics[width=1\linewidth]{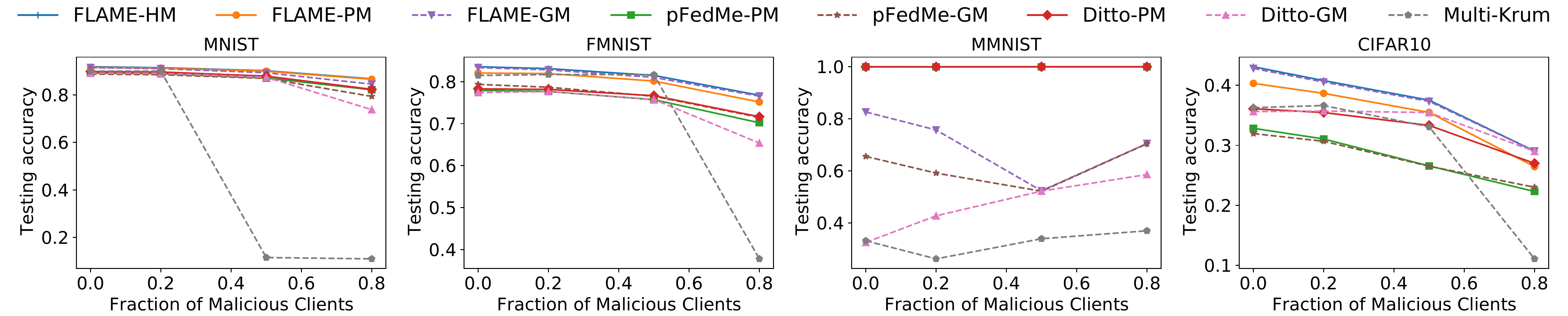}
    \vspace{-2em}
    \caption{Robustness comparison of different methods under label poisoning attacks with quality skew.}
    \label{fig:acc_malicious_attack_1_quality-skew_2}
    \vspace{-1em}
\end{figure*}
\begin{figure*}[ht]
    \centering
    \includegraphics[width=1\linewidth]{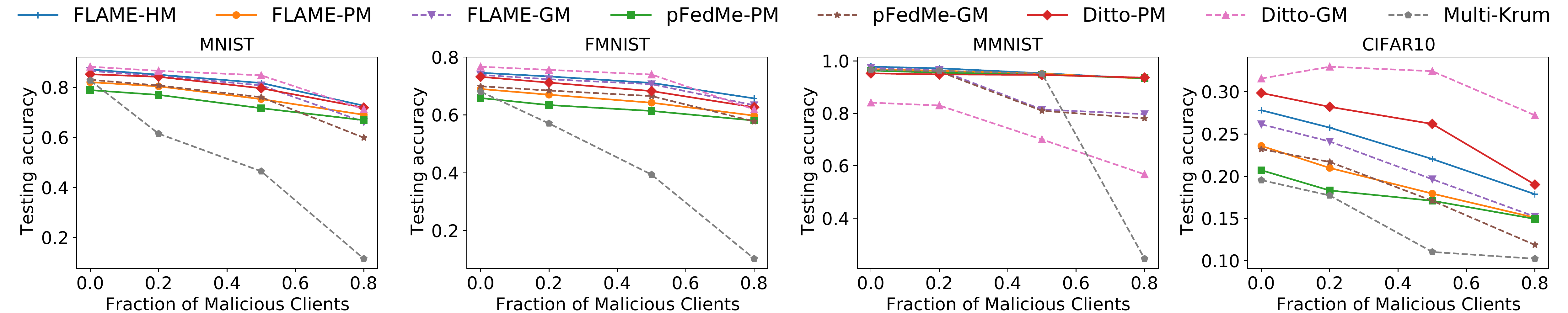}
    \vspace{-2em}
    \caption{Robustness comparison of different methods under label poisoning attacks with quantity skew.}
    \label{fig:acc_malicious_attack_1_dir-quantity-skew_2}
    \vspace{-1em}
\end{figure*}

\begin{figure*}[ht]
    \centering
    \includegraphics[width=1\linewidth]{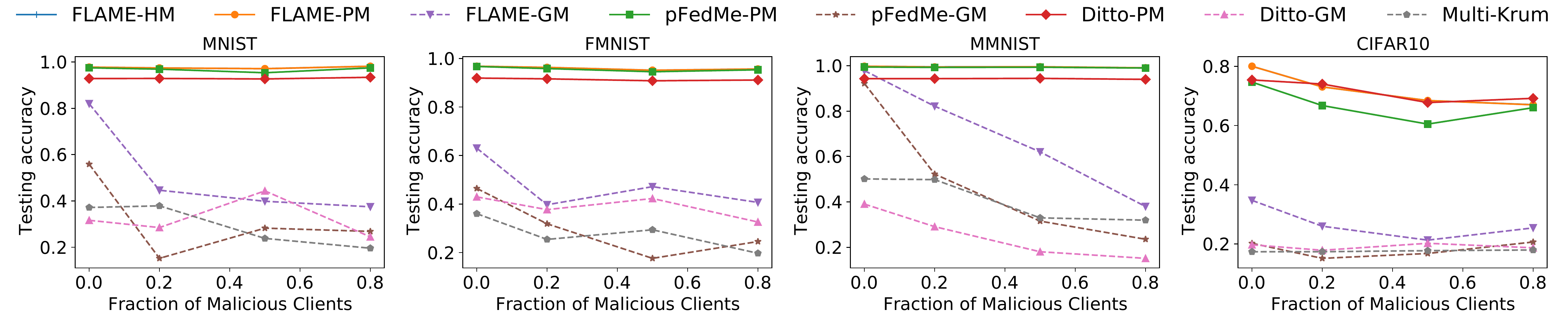}
    \vspace{-2em}
    \caption{Robustness comparison of different methods under same-value attacks with quantity-based label imbalance ($q=2$).}
    \label{fig:acc_malicious_attack_2_q-label-skew_2}
    \vspace{-1em}
\end{figure*}
\begin{figure*}[ht]
    \centering
    \includegraphics[width=1\linewidth]{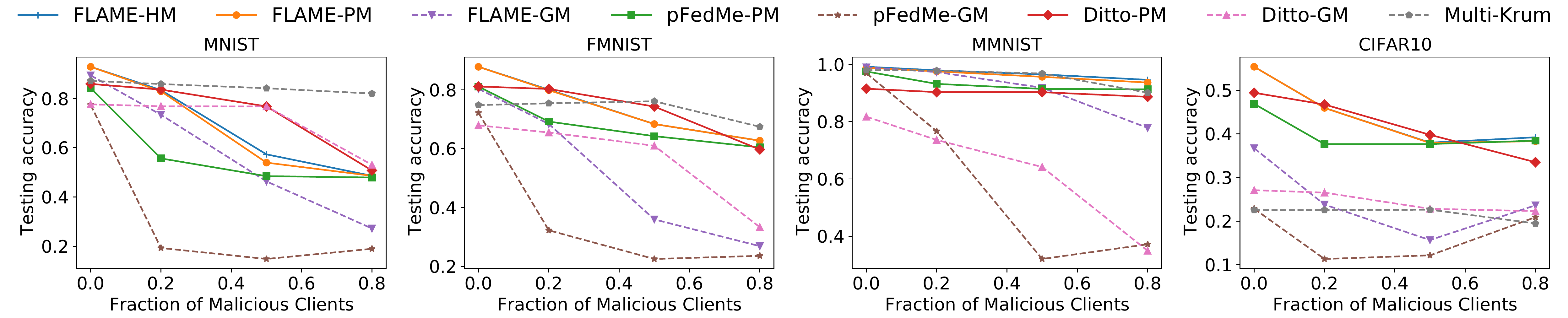}
    \vspace{-2em}
    \caption{Robustness comparison of different methods under same-value attacks with distribution-based label imbalance.}
    \label{fig:acc_malicious_attack_2_dir-label-skew_2}
    \vspace{-1em}
\end{figure*}
\begin{figure*}[ht]
    \centering
    \includegraphics[width=1\linewidth]{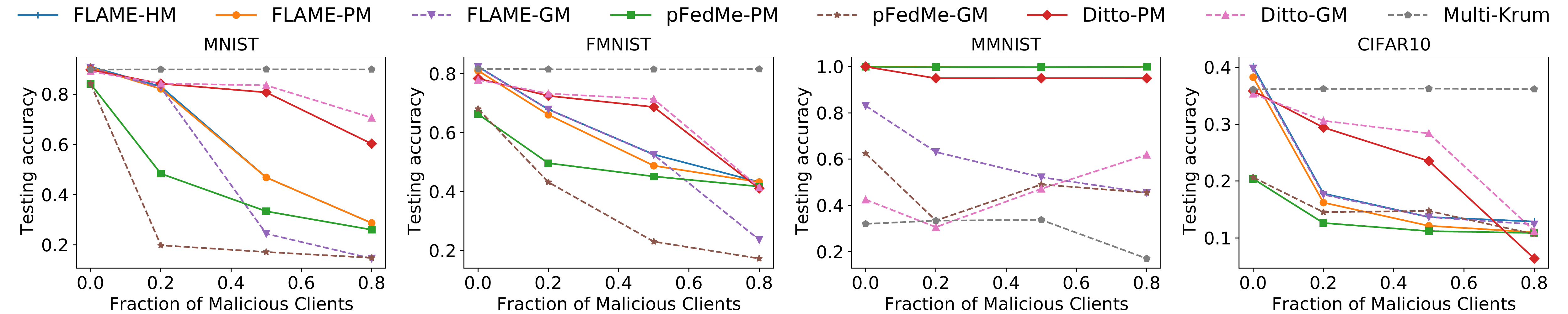}
    \vspace{-2em}
    \caption{Robustness comparison of different methods under same-value attacks with quality skew.}
    \label{fig:acc_malicious_attack_2_quality-skew_2}
    \vspace{-1em}
\end{figure*}
\begin{figure*}[ht]
    \centering
    \includegraphics[width=1\linewidth]{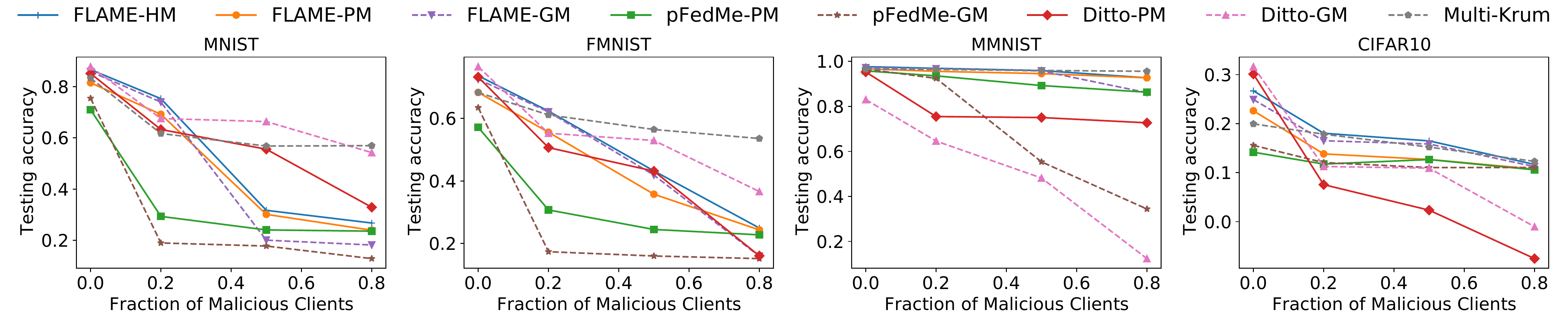}
    \vspace{-2em}
    \caption{Robustness comparison of different methods under same-value attacks with quantity skew.}
    \label{fig:acc_malicious_attack_2_dir-quantity-skew_2}
    \vspace{-1em}
\end{figure*}
\begin{figure*}[ht]
    \centering
    \includegraphics[width=1\linewidth]{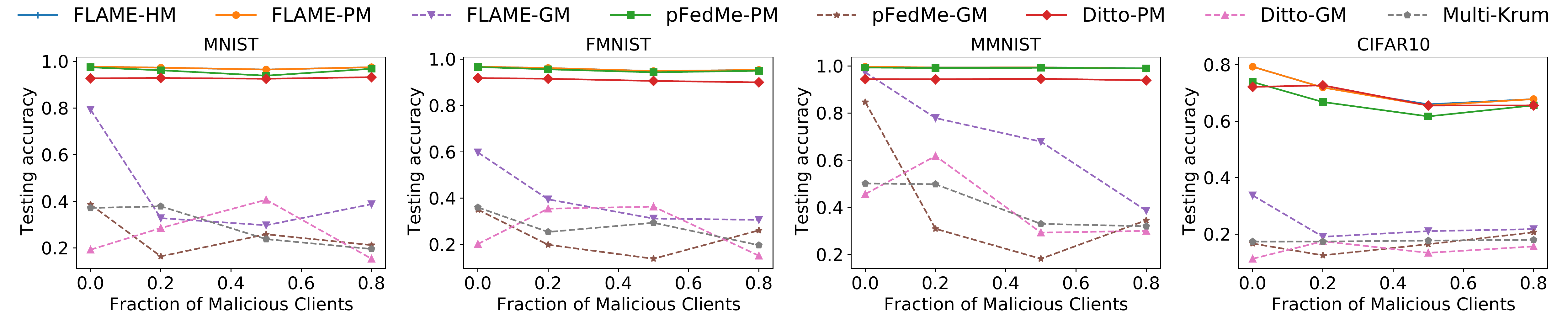}
    \vspace{-2em}
    \caption{Robustness comparison of different methods under sign-flipping attacks with quantity-based label imbalance ($q=2$).}
    \label{fig:acc_malicious_attack_3_q-label-skew_2}
    \vspace{-1em}
\end{figure*}
\begin{figure*}[ht]
    \centering
    \includegraphics[width=1\linewidth]{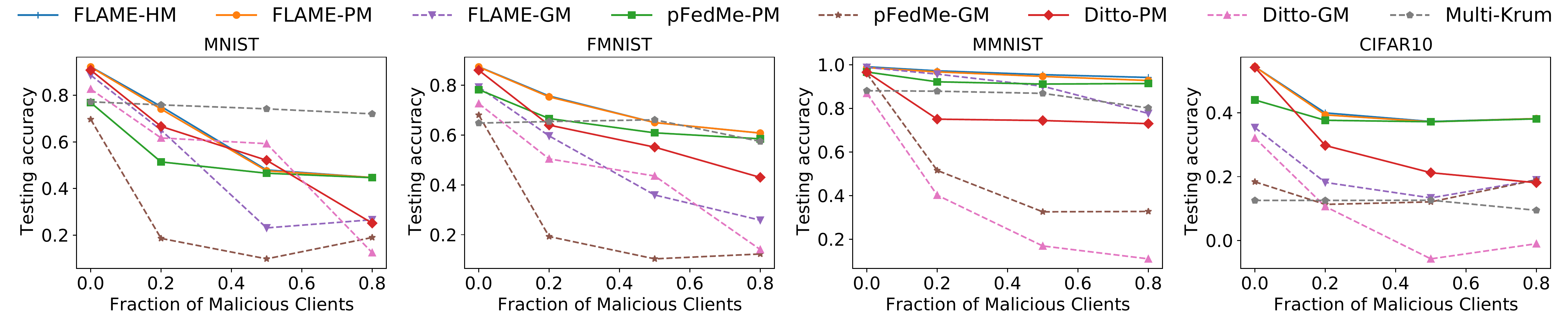}
    \vspace{-2em}
    \caption{Robustness comparison of different methods under sign-flipping attacks with distribution-based label imbalance.}
    \label{fig:acc_malicious_attack_3_dir-label-skew_2}
    \vspace{-1em}
\end{figure*}
\begin{figure*}[ht]
    \centering
    \includegraphics[width=1\linewidth]{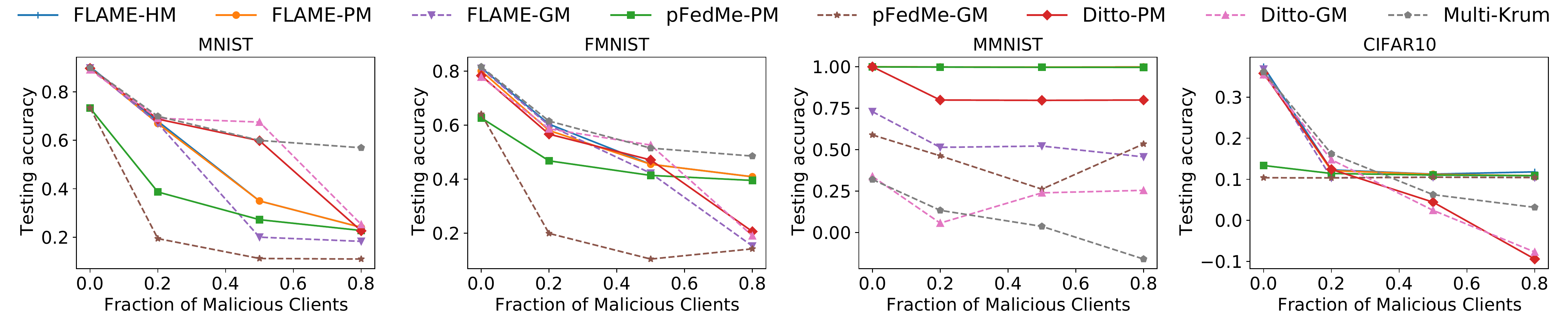}
    \vspace{-2em}
    \caption{Robustness comparison of different methods under sign-flipping attacks with quality skew.}
    \label{fig:acc_malicious_attack_3_quality-skew_2}
    \vspace{-1em}
\end{figure*}
\begin{figure*}[ht]
    \centering
    \includegraphics[width=1\linewidth]{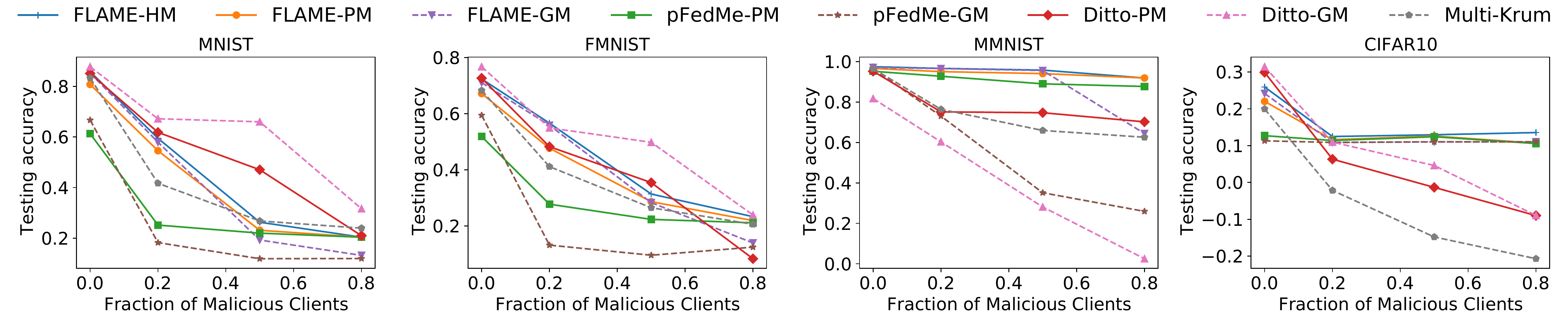}
    \vspace{-2em}
    \caption{Robustness comparison of different methods under sign-flipping attacks with quantity skew.}
    \label{fig:acc_malicious_attack_3_dir-quantity-skew_2}
    \vspace{-1em}
\end{figure*}
\begin{figure*}[ht]
    \centering
    \includegraphics[width=1\linewidth]{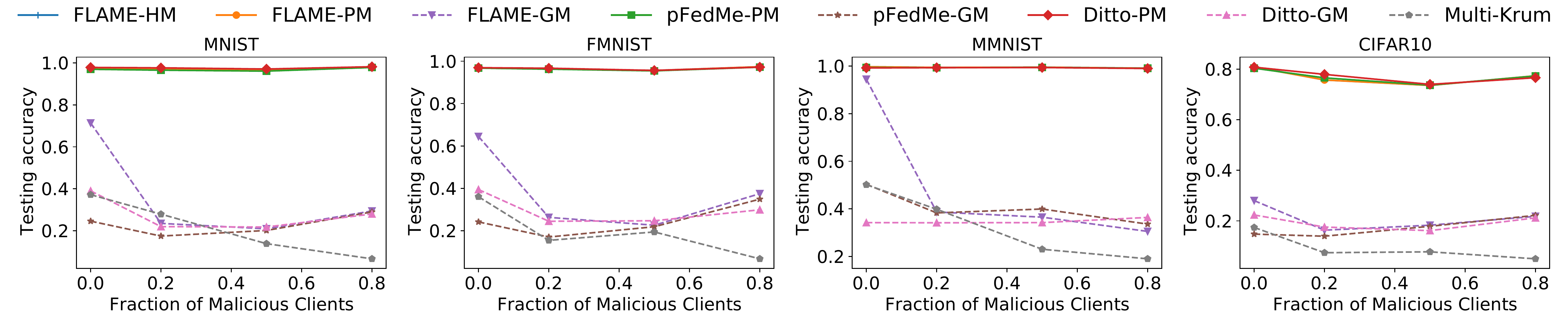}
    \vspace{-2em}
    \caption{Robustness comparison of different methods under Gaussian attacks with quantity-based label imbalance ($q=2$).}
    \label{fig:acc_malicious_attack_4_q-label-skew_2}
    \vspace{-1em}
\end{figure*}
\begin{figure*}[ht]
    \centering
    \includegraphics[width=1\linewidth]{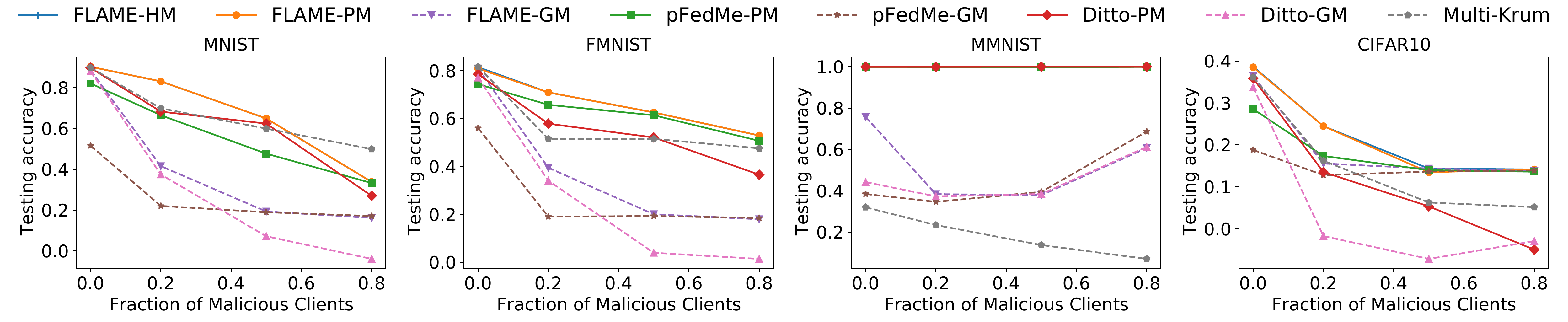}
    \vspace{-2em}
    \caption{Robustness comparison of different methods under Gaussian attacks with quality skew.}
    \label{fig:acc_malicious_attack_4_quality-skew_2}
    \vspace{-0.5em}
\end{figure*}
\begin{figure*}[ht]
    \centering
    \includegraphics[width=1\linewidth]{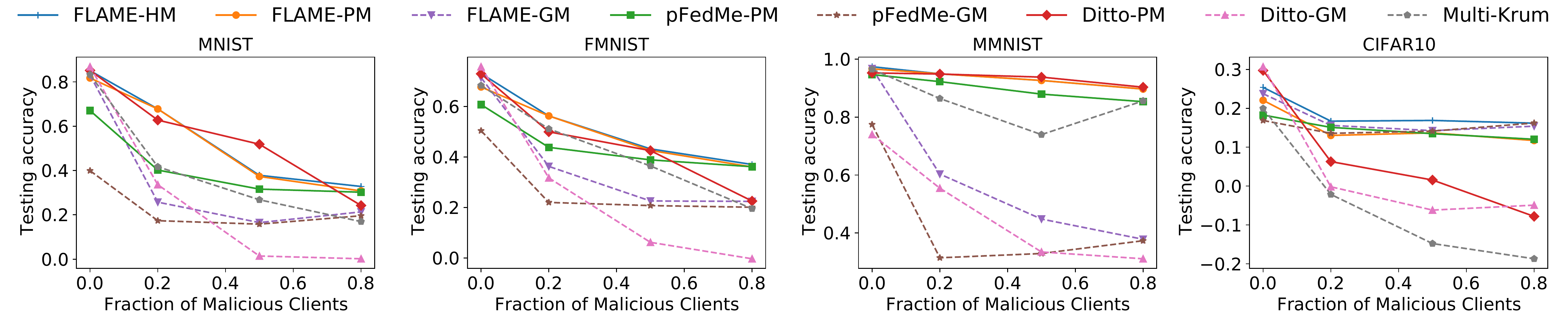}
    \vspace{-2em}
    \caption{Robustness comparison of different methods under Gaussian attacks with quantity skew.}
    \label{fig:acc_malicious_attack_4_dir-quantity-skew_2}
    \vspace{-1em}
\end{figure*}
\newpage

\begin{figure*}[ht]
    \centering
    \includegraphics[width=1\linewidth]{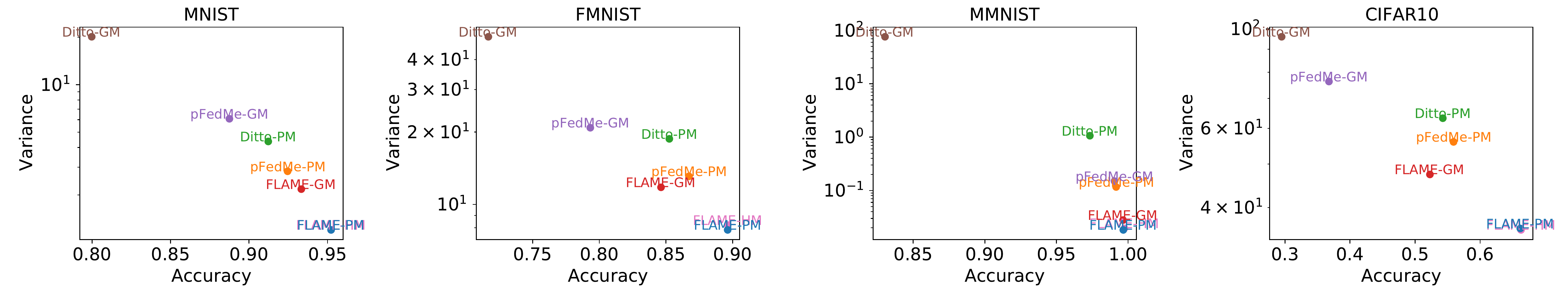}
    \vspace{-2em}
    \caption{Accuracy-fairness trade-off of competing methods with distribution-based label imbalance.}
    \label{fig:variance_accuracy_dir-label-skew_q=2}
    \vspace{-1em}
\end{figure*}
\begin{figure*}[ht]
    \centering
    \includegraphics[width=1\linewidth]{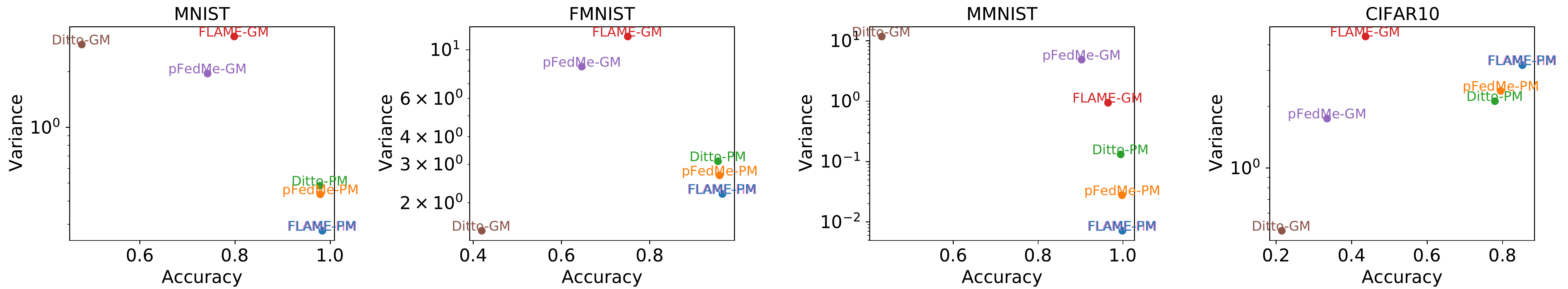}
    \vspace{-2em}
    \caption{Accuracy-fairness trade-off of competing methods with quantity-based label imbalance ($q=2$).}
    \label{fig:variance_accuracy_q-label-skew_q=2}
    \vspace{-1em}
\end{figure*}
\begin{figure*}[ht]
    \centering
    \includegraphics[width=1\linewidth]{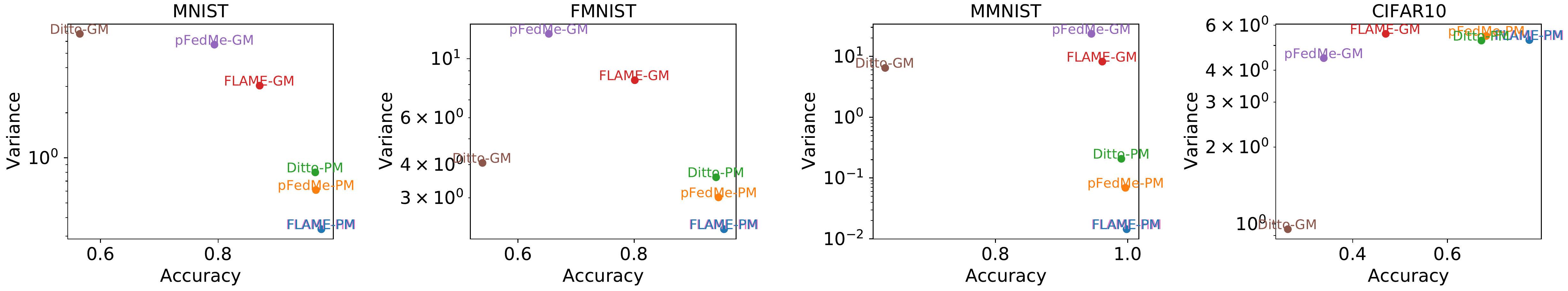}
    \vspace{-2em}
    \caption{Accuracy-fairness trade-off of competing methods with quantity-based label imbalance ($q=3$).}
    \label{fig:variance_accuracy_q-label-skew_q=3}
    \vspace{-1em}
\end{figure*}
\begin{figure*}[ht]
    \centering
    \includegraphics[width=1\linewidth]{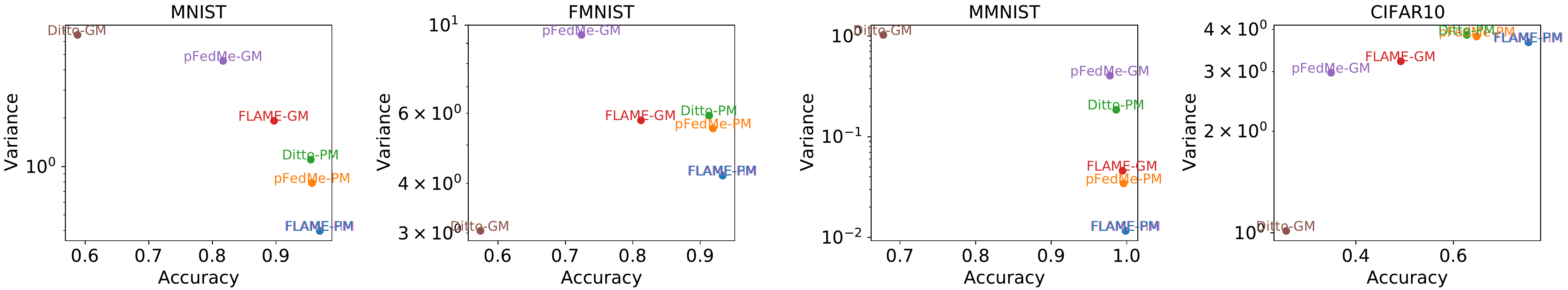}
    \vspace{-2em}
    \caption{Accuracy-fairness trade-off of competing methods with quantity-based label imbalance ($q=4$).}
    \label{fig:variance_accuracy_q-label-skew_q=4}
    \vspace{-1em}
\end{figure*}
\begin{figure*}[ht]
    \centering
    \includegraphics[width=1\linewidth]{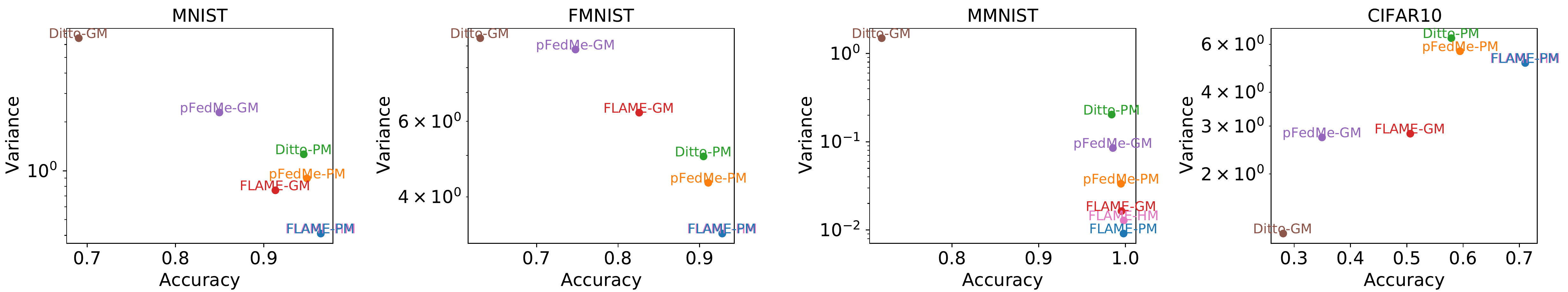}
    \vspace{-2em}
    \caption{Accuracy-fairness trade-off of competing methods with quantity-based label imbalance ($q=5$).}
    \label{fig:variance_accuracy_q-label-skew_q=5}
    \vspace{-1em}
\end{figure*}
\begin{figure*}[ht]
    \centering
    \includegraphics[width=1\linewidth]{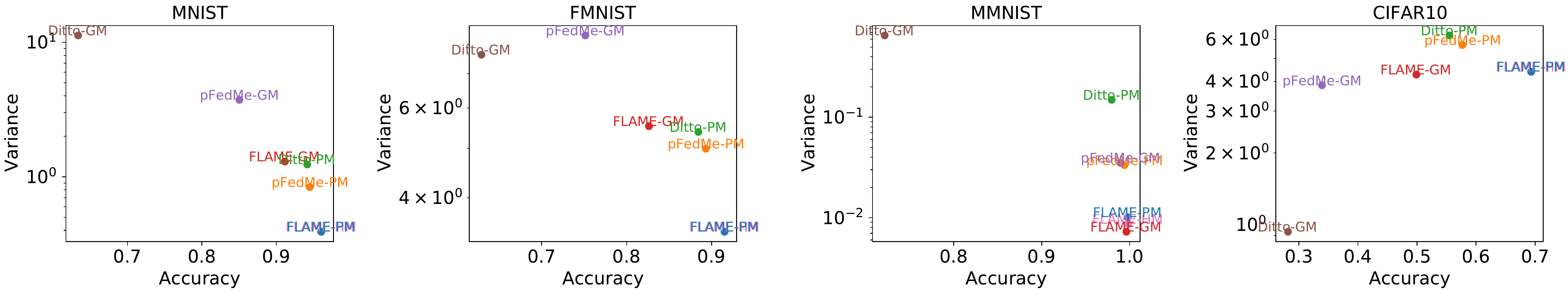}
    \vspace{-2em}
    \caption{Accuracy-fairness trade-off of competing methods with quantity-based label imbalance ($q=6$).}
    \label{fig:variance_accuracy_q-label-skew_q=6}
    \vspace{-1em}
\end{figure*}
\begin{figure*}[ht]
    \centering
    \includegraphics[width=1\linewidth]{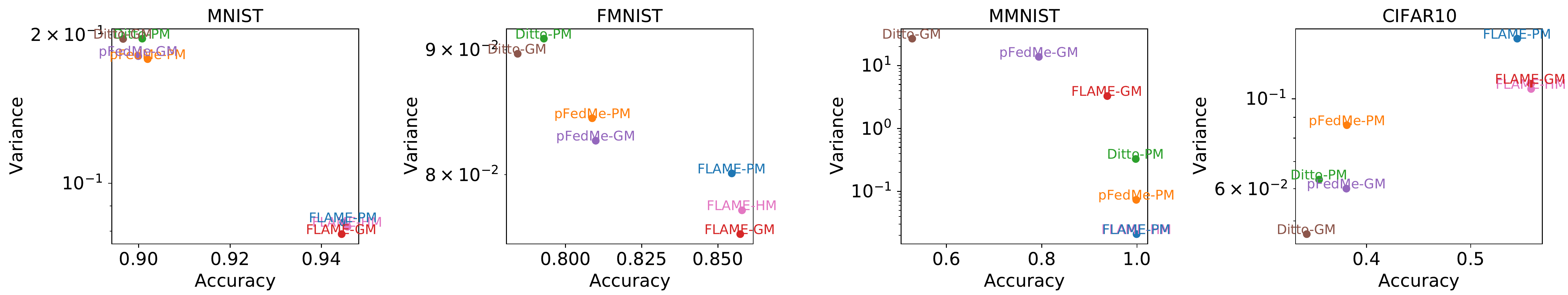}
    \vspace{-2em}
    \caption{Accuracy-fairness trade-off of competing methods with quality skew.}
    \label{fig:variance_accuracy_quality-skew_q=2}
    \vspace{-1em}
\end{figure*}
\begin{figure*}[ht]
    \centering
    \includegraphics[width=1\linewidth]{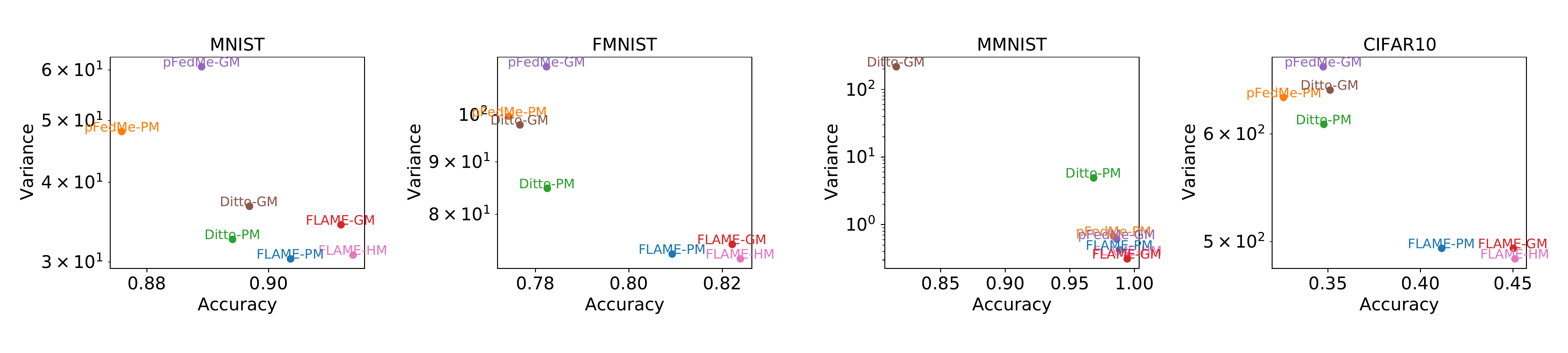}
    \vspace{-2em}
    \caption{Accuracy-fairness trade-off of competing methods with quantity skew.}
    \label{fig:variance_accuracy_dir-quantity-skew_q=2}
    \vspace{-1em}
\end{figure*}

\end{document}